%% file: iclr2026_conference.tex
\documentclass{article} 
\usepackage{iclr2026_conference,times}
\usepackage[T1]{fontenc}

\input{math_commands.tex}

\input{macros.tex}

\usepackage[pagebackref=true]{hyperref}       

\renewcommand*{\backref}[1]{}
\renewcommand*{\backrefalt}[4]{\ifcase#1\relax\or(page #2)\else(pages #2)\fi}

\usepackage{url}
\usepackage{cleveref}
\crefformat{equation}{(#2#1#3)}
\crefrangeformat{equation}{(#3#1#4) to (#5#2#6)}
\crefmultiformat{equation}{(#2#1#3)}{ and (#2#1#3)}{, (#2#1#3)}{ and (#2#1#3)}
\crefformat{section}{\S#2#1#3} 
\crefformat{subsection}{\S#2#1#3}
\crefformat{subsubsection}{\S#2#1#3}
\crefformat{appendix}{\S#2#1#3}

\usepackage{tcolorbox}
\usepackage{todonotes} 
\usepackage{wrapfig}
\usepackage{multicol}
\usepackage{multirow}
\usetikzlibrary{arrows.meta, positioning, calc, backgrounds,shapes.geometric} 
\usepackage{subcaption}

\usepackage{enumitem}
\setlist[enumerate,1]{leftmargin=2em}
\setlist[itemize,1]{leftmargin=2em}

\usepackage{booktabs}
\usepackage{amssymb}
\usepackage{adjustbox}

\title{\flock: A Knowledge Graph Foundation Model via Learning on Random Walks}

\author{
Jinwoo Kim\textnormal{\textsuperscript{1}\thanks{Equal contribution.}\quad}
Xingyue Huang\textnormal{\textsuperscript{2}\footnotemark[1]\quad}
Krzysztof Olejniczak\textnormal{\textsuperscript{2}\quad}
Kyungbin Min\textnormal{\textsuperscript{1}}\\\hspace{0.05ex}
\textbf{Michael Bronstein}\textsuperscript{2,4}\quad
\textbf{Seunghoon Hong}\textsuperscript{1}\quad
\textbf{{\.I}smail {\.I}lkan Ceylan}\textsuperscript{3,4,2} \\[0.6ex]
\textsuperscript{1}KAIST\quad
\textsuperscript{2}University of Oxford\quad
\textsuperscript{3}TU Wien\quad
\textsuperscript{4}AITHYRA
}

\iclrfinalcopy 
\begin{document}

\maketitle

\begin{abstract}
We study the problem of zero-shot link prediction on knowledge graphs (KGs), which requires models to generalize to \emph{novel entities} and \emph{novel relations}. Knowledge graph foundation models (KGFMs) address this task by enforcing equivariance over \emph{both} nodes and relations, which enables them to learn structural properties of nodes and relations that transfer to novel KGs with similar structure. However, the conventional notion of deterministic equivariance inherently limits the expressive power of KGFMs, as it prevents them from distinguishing relations that are structurally similar but semantically distinct. To overcome this limitation, we propose to leverage \emph{probabilistic} node-relation equivariance, which preserves equivariance \emph{in distribution} while using structured randomness to break symmetries at inference time. Building on this principle, we present $\flock$, a KGFM that iteratively samples random walks, encodes them into sequences, embeds them with a sequence model, and aggregates node and relation representations through learned pooling. $\flock$ respects probabilistic node-relation equivariance and, crucially, is a \emph{universal approximator} for isomorphism-invariant link-level functions over KGs. Empirically, $\flock$ perfectly solves our new diagnostic dataset \textsc{Petals} on which current KGFMs fail, and achieves state-of-the-art performance on entity and relation prediction tasks across 54 KGs from diverse domains.
\end{abstract}

\section{Introduction}

Knowledge graph foundation models (KGFMs)~\citep{galkin2023ultra,zhang2024trix,huang2025how}
aim to infer missing links over novel knowledge graphs (KGs) that are not part of the training data or domains. This task requires generalization to \emph{both} unseen nodes and unseen relation types. To achieve this, KGFMs learn \emph{node and relation invariants}: structural properties of nodes and relations that are transferable across KGs even when their relational vocabularies differ. This inductive bias is formalized as double equivariance \citep{gao2023double}---equivariance under permutations of both entities and relations---and used as a core design principle of current KGFMs.

\begin{wrapfigure}{r}{5.5cm}
    \centering
    \vspace{-2em}

    \begin{tikzpicture}[
        scale=1,
        every node/.style={circle, draw, fill=black, inner sep=2pt, line width=0.5pt}, 
        line width=1pt, 
        >=Stealth,
        shorten >=3pt, 
        shorten <=3pt,
        transform shape
    ]

    \node (s) at (0,0)   [label={[font=\scriptsize]270:$\mathsf{Luke}$}] {};
    \node (t) at (0,1.5) [label={[font=\scriptsize,yshift=-0.5em]90:$\mathsf{Yoda}$}] {};

    \node (a1) at (-1, 1)   [label={[font=\scriptsize,yshift=-1em]90:$\mathsf{HanSolo}$}] {};
    \node (a2) at (-1,-1)   [label={[font=\scriptsize,yshift=1em]-90:$\mathsf{Emperor}$}] {};
    \node (a3) at (-2, 0)   [label={[font=\scriptsize]-90:$\mathsf{Jabba}$}] {};
    
    \node (b1) at ( 1, 1)   [label={[font=\scriptsize,yshift=-1.5em]90:$\mathsf{DarthVader}$}] {};
    \node (b2) at ( 1,-1)   [label={[font=\scriptsize,yshift=1.4em]-90:$\mathsf{Chewbacca}$}] {};
    \node (b3) at ( 2, 0)   [label={[font=\scriptsize]-90:$\mathsf{Leia}$}] {};

    \draw[->, color=solidgreen] (s) -- (t);

    \draw[->, color=likeblue]   (s) -- (a1);
    \draw[->, color=dislikered] (s) -- (a2);
    \draw[->, color=dislikered] (a1) -- (a3);
    \draw[->, color=dislikered] (a2) -- (a3);

    \draw[->, color=dislikered] (s) -- (b1);
    \draw[->, color=likeblue]   (s) -- (b2);
    \draw[->, color=likeblue]   (b1) -- (b3);
    \draw[->, color=likeblue]   (b2) -- (b3);

    \draw[->, color=solidgreen, dashed] (s) to[in=-80, out=170] (a1);
    \draw[->, color=solidgreen, dashed] (s) to[in=80,  out=190] (a2);

    \end{tikzpicture}
    \vspace{-1.5em}
    \caption{A KG representing characters' relationships in Star Wars movies. Blue arrows indicate \textcolor{likeblue}{$\mathsf{like}$}, red -- \textcolor{dislikered}{$\mathsf{dislike}$}, and green arrows indicate \textcolor{solidgreen}{$\mathsf{friendWith}$}.}
    \vspace{-1em}
    \label{fig:movie-preferences}
\end{wrapfigure}

\paragraph{Problem statement.} In this work, we challenge a fundamental assumption dictated by strict equivariance in existing KGFMs: \emph{structural isomorphism of relations implies semantic equivalence}.
Consider, for example, the KG in \Cref{fig:movie-preferences}, where the relations \textcolor{likeblue}{$\mathsf{like}$} and \textcolor{dislikered}{$\mathsf{dislike}$} are structurally isomorphic yet semantically opposite. Any KGFM that computes relation invariants is forced to assign the same representation to both \textcolor{likeblue}{$\mathsf{like}$} and \textcolor{dislikered}{$\mathsf{dislike}$}---losing the ability to distinguish entities with opposite relationships. This expressiveness limitation is an architectural one and \emph{cannot} be resolved through finetuning, further limiting the downstream use of existing KGFMs. This raises a central question: how can we design KGFMs that are both \emph{expressive} and have the right \emph{inductive bias} for generalization?

\paragraph{Approach.} 
We propose a new approach for KGFMs that relies on \emph{probabilistic} node-relation equivariance as inductive bias, instead of enforcing \textit{deterministic} equivariance over nodes and relations. Probabilistic equivariance relaxes the hard constraint that ``structurally isomorphic relations \emph{must} have identical representations'', and requires only that the representations of structurally isomorphic relations need to be equivalent \emph{in distribution} over a model's stochastic predictive processes \citep{Srinivasan2020On,abboudsurprising}. This way, the model retains the inductive bias needed for generalizing across different KGs, while the stochasticity of each forward pass ensures that structurally identical but semantically distinct relations are assigned different representations.

Inspired by the success of models that learn probabilistic invariants via random walks~\citep{deepwalk, node2vec-kdd2016, rwnn_original, kim_rwnn}, we introduce $\flock$, a KGFM that inherently computes probabilistic node-relation invariants. Given a potentially unseen KG and a query, $\flock$ iteratively samples random walks over the KG based on the query, encoding both encountered nodes and relations with a recording protocol. To ensure the model can generalize to unseen entities and relation types, we \emph{anonymize} all nodes and relations, enforcing that $\flock$ only learn from their structural roles. These anonymized sequences are then fed into a sequence processor, and the representations for each node and relation are aggregated via a consensus protocol. 
Finally, we construct per-query features from the aggregated entity and relation embeddings and input them into a binary classifier for link prediction.

\paragraph{Key findings and contributions.} The design of $\flock$ offers several key advantages over existing KGFMs. First, it entirely abandons the conventional two-stage process of encoding relations and node representations via two separate networks, and does not rely on message passing at all, thereby avoiding the well-known expressivity limitations of MPNNs on KGs~\citep{barcelo2022wlgorelational, huang2023theory, huang2025how}.  Second, $\flock$ is a universal approximator (Proposition \ref{prop:universal_approximator}), capable of approximating every link-level function on KGs of any bounded size. Finally, $\flock$'s architecture inherently respects probabilistic node-relation equivariance, enabling strong generalization. Our experiments on both entity and relation prediction validate this approach, demonstrating that $\flock$ consistently outperforms state-of-the-art KGFMs on existing benchmarks.
Our contributions are:
\begin{itemize}
    \item We highlight a limitation of existing KGFMs: their reliance on deterministic node–relation equivariance prevents them from distinguishing between structurally similar but semantically different relations, limiting their expressivity.
    \item We \draft{propose to leverage} probabilistic node-relation equivariance, a property for KGFMs that ensures equivariance only in distribution, \draft{as an effective solution} balancing the model expressivity and generalization.
    \item We propose $\flock$, a KGFM that respects probabilistic node-relation equivariance. $\flock$ replaces the conventional two-stage, message-passing paradigm with a direct sequence encoding approach based on random walks, and is a universal approximator of link-level functions.
    \item We validate our approach on entity and relation prediction tasks across 54 diverse KGs, where $\flock$ consistently achieves state-of-the-art performance. We further construct a synthetic dataset \textsc{Petals} to confirm our theoretical results empirically.
\end{itemize}

All proofs are in \Cref{app:proofs}.
The code is available at \url{https://github.com/jw9730/flock}.

\section{Related work}

\paragraph{Link prediction and KGFMs.}
Early methods for inferring missing links in KGs relied on learned embeddings \citep{brodes2013transe,sun2019rotate, Balazevic_2019,abbound2020boxe,schlichtkrull2017modeling,vashishth2020compositionbased}, operating in the \emph{transductive} setting and incapable of generalizing to unseen entities or relation types.
Later MPNN-based methods based on the labeling trick \citep{grail2020teru,LabelingTrick2021} or conditional message passing \citep{zhu2022neural,zhu2023anet,zhang2022redgnn,adaprop,huang2023theory} unlocked the \emph{node-inductive} scenario, while remaining restricted to a fixed relational vocabulary.
KGFMs eliminate this restriction and enable \emph{node-relation inductive} link prediction over both unseen nodes and relations \citep{geng2022relationalmessagepassingfully}, typically by first encoding relations and then nodes. Early examples are InGram~\citep{ingram} and ULTRA~\citep{galkin2023ultra}, later extended by TRIX~\citep{zhang2024trix} into a more expressive framework.
KG-ICL~\citep{cui2024prompt} achieved full inductiveness by combining in-context learning with node-relation tokenization.
ISDEA~\citep{gao2023double} and MTDEA~\citep{zhou2023multitaskperspetivelinkprediction} highlighted the benefits of node-relation equivariance, while MOTIF~\citep{huang2025how} proposed a general KGFM framework with a theoretical analysis of their expressive power.
Our work advances the field with \draft{a stochastic KGFM} that achieves invariance in probability to ensure generalization while being provably more expressive than the existing methods.
Notably, $\flock$ achieves universality without any form of message passing, instead relying on random walks and sequence models to encode both nodes and relations anonymously.
\draft{This is distinct from prior stochastic KGFMs that rely on random initialization of message passing \citep{ingram, gao2023double}.}

\paragraph{Random walks for graph representations.}
Random walks have been widely adopted in graph learning due to their simplicity and ability to gather context from neighborhoods.
DeepWalk \citep{deepwalk} and node2vec \citep{node2vec-kdd2016} were among the first approaches, treating walks as analogues of natural language sentences and processing them with skip-gram models.
Subsequent work has developed diverse neural architectures based on random walks: \citet{rwnn_original} generates graph-level predictions using joint walks on direct products of graphs and their subgraphs, CRaWL~\citep{tonshoff_crawl} represents a graph as a collection of walks and processes them with a convolutional network, WalkLM~\citep{walklm} samples walks from graphs with textual features and embeds them using a language model, RWNN~\citep{kim_rwnn} and RUM~\citep{wang2024rum} anonymize walks and process them with sequence networks, and NeuralWalker~\citep{neuralwalker} integrates walk encodings into message passing.

\paragraph{Probabilistic invariance.} Neural architectures that enforce invariance to specific transformations often exhibit more stable training and improved performance~\citep{bronstein2021geometric}, but this inductive bias can reduce their expressivity by preventing the model from distinguishing non-equivalent inputs. 
In graph learning, this trade-off is exemplified by MPNNs, whose power is limited by the 1-WL test~\citep{xu2018howexpressive,MorrisAAAI19}.
Randomization has emerged as a solution, enhancing expressivity through techniques such as noise injection~\citep{abboudsurprising}, vertex dropping~\citep{papp2021dropgnn}, subgraph sampling~\citep{bevilacquaequivariant,subgraphgnn_hierarchy}, dynamic rewiring~\citep{finkelshtein2024cooperative},
and random walks~\citep{kim_rwnn, wang2024rum}.
Despite their stochasticity, such methods can remain probabilistically invariant, ensuring that equivalent inputs yield identical expected outputs, or even identical output distributions.
\draft{In line with a prior work \citep{gao2023double}, we extend the notion of probabilistic invariance to KGs and prove that $\flock$ satisfies invariance in distribution, and further clarify its usefulness in KG learning.}

\section{Preliminary}
\label{sec:preliminaries}

\textbf{Knowledge graphs.} 
A \emph{knowledge graph} (KG) is a tuple $G=(V,E,R)$, where $V$ denotes the set of entities (nodes), $R$ the set of relation types, and $E \subseteq V \times R \times V$ the set of labeled edges (\emph{facts}). 
A fact is written as $(h,r,t)$ (or $h \xrightarrow{r} t$ interchangeably) with $r \in R$ and $h,t \in V$. 
A (potential) \emph{link} in $G$ is any triple $(h,r,t)$ in $V \times R \times V$, regardless of whether it is present in $E$. 
We denote by $R^{-1}$ the set of inverses of relations $R$, defined as $\{r^{-1}\mid r\in R\}$, and mean $r$ when writing $(r^{-1})^{-1}$. \draft{Further, let $\mathbb{K}_{n,m}$ be the space of knowledge graphs with $n$ vertices and $m$ relation types.}

\textbf{Isomorphism.} 
An \emph{isomorphism} between two knowledge graphs $G=(V,E,R)$ and $G'=(V',E',R')$ is a pair of bijections $\mu=(\pi,\phi)$, 
where $\pi: V \to V'$ and $\phi: R \to R'$, such that a fact $(h,r,t)$ belongs to $E$ if and only if the fact $\mu((h,r,t)) = (\pi(h),\phi(r),\pi(t))$ belongs to $E'$. 
Two KGs are \emph{isomorphic} if such a mapping exists, in which case we write $G\simeq G'$.

\textbf{Link invariance.}
In this work, we focus on link-invariant functions. 
Let $\omega$ be a function assigning to each KG $G=(V,E,R) \in \mathbb{K}_{n,m}$ a map $\omega(G): V \times R \times V \rightarrow \mathbb{R}^d$.
We say that $\omega$ is \emph{link invariant} if for every pair of isomorphic KGs $G,G' \in \mathbb{K}_{n,m}$, 
every isomorphism $(\pi,\phi)$ from $G$ to $G'$, and every link $(h,r,t)$ in $G$, we have
$
\omega(G)((h,r,t)) = \omega(G')((\pi(h),\phi(r),\pi(t))).
$

\textbf{Probabilistic invariance.}
A stochastic KG model $\varphi$ can be viewed as a function that takes a KG~and returns a random variable $\varphi(G)$. Following \citet{kim_rwnn}, we call $\varphi$ \emph{invariant in probability} if
\[
\forall G, G' \in \mathbb{K}_{n,m}: \qquad G\simeq G' \implies \varphi(G) \probeq \varphi(G'),
\]
i.e., the distributions of $\varphi(G)$ and $\varphi(G')$ are equal. In particular, this implies $\mathbb{E}[\varphi(G)] = \mathbb{E}[\varphi(G')]$.

\section{Methodology}\label{sec:methods}

\begin{figure}
\vspace{-2em}
\begin{tikzpicture}[
        scale=1,
        every node/.style={circle, draw, fill=black, inner sep=2pt, line width=0.5pt}, 
        line width=0.5pt, 
        >=Stealth, 
        shorten >=3pt, 
        shorten <=3pt,
        transform shape
    ]
    \definecolor{vaguegray}{rgb}{0.5,0.5,0.5} 

\begin{scope}[scale=0.8, xshift=0cm]
    \node[circle, draw, fill=white, inner sep=2pt, line width=0.5pt] (s) at (0,0) {};
    \node[circle, draw, fill=white, inner sep=2pt, line width=0.5pt] (t) at (0,1.5) {};

    \node[circle, draw, fill=white, inner sep=2pt, line width=0.5pt] (a1) at (-1, 1) {};
    \node[circle, draw, fill=white, inner sep=2pt, line width=0.5pt] (a2) at (-1,-1){};
    \node[circle, draw, fill=white, inner sep=2pt, line width=0.5pt] (a3) at (-2, 0) {};

    \node[circle, draw, fill=white, inner sep=2pt, line width=0.5pt] (b1) at (1, 1) {};
    \node[circle, draw, fill=white, inner sep=2pt, line width=0.5pt] (b2) at (1,-1){};
    \node[circle, draw, fill=white, inner sep=2pt, line width=0.5pt] (b3) at (2, 0) {};

    \draw[->, color=black] 
        (s) to node[midway, right, draw=none, fill=none, text=black] {\scriptsize $r_3$} (t);

    \draw[->, color=black] 
        (s) to node[midway, below, draw=none, fill=none, text=black] {\scriptsize $r_1$} (a1);
    \draw[->, color=black] 
        (s) to node[midway, above, draw=none, fill=none, text=black] {\scriptsize $r_1$} (b2);
    \draw[->, color=black] 
        (b1) to node[midway, above, draw=none, fill=none, text=black] {\scriptsize $r_1$} (b3);
    \draw[->, color=black] 
        (b2) to node[midway, below, draw=none, fill=none, text=black] {\scriptsize $r_1$} (b3);

    \draw[->, color=black] 
        (s) to node[midway, above, draw=none, fill=none, text=black] {\scriptsize $r_2$} (a2);
    \draw[->, color=black] 
        (a1) to node[midway, above, draw=none, fill=none, text=black] {\scriptsize $r_2$} (a3);
    \draw[->, color=black] 
        (a2) to node[midway, below, draw=none, fill=none, text=black] {\scriptsize $r_2$} (a3);
    \draw[->, color=black] 
        (s) to node[midway, below, draw=none, fill=none, text=black] {\scriptsize $r_2$} (b1);

\begin{scope}[on background layer, transparency group, opacity=0.1]

    \draw[draw=red!25, line width=2pt, fill=red!25] (t) circle (0.3);
    \draw[draw=red!25, line width=2pt, fill=red!25] (s) circle (0.3);
    \draw[draw=red!25, line width=2pt, fill=red!25] (a1) circle (0.3);
    \draw[draw=red!25, line width=2pt, fill=red!25] (a3) circle (0.3);

  \draw[->, draw=red!25, double=red!25, double distance=1pt,
        line width=1.5pt, shorten >=0.5pt, shorten <=0.5pt] (t) to (s);

  \draw[->, draw=red!25, double=red!25, double distance=1pt,
        line width=1.5pt, shorten >=0.5pt, shorten <=0.5pt] (s) to (a1);

  \draw[->, draw=red!25, double=red!25, double distance=1pt,
        line width=1.5pt, shorten >=0.5pt, shorten <=0.5pt] (a1) to (a3);
\end{scope}

\begin{scope}[on background layer, transparency group, opacity=0.1]
\draw[draw=teal!25, line width=2pt, fill=teal!25] (s) circle (0.2);
\draw[draw=teal!25, line width=2pt, fill=teal!25] (b3) circle (0.2);
\draw[draw=teal!25, line width=2pt, fill=teal!25] (b2) circle (0.2);
\draw[draw=teal!25, line width=2pt, fill=teal!25] (b1) circle (0.2);
  \draw[->, draw=teal!25, double=teal!25, double distance=1pt,
        line width=1.5pt, shorten >=0.5pt, shorten <=0.5pt] (s)  to (b2);
  \draw[->, draw=teal!25, double=teal!25, double distance=1pt,
        line width=1.5pt, shorten >=0.5pt, shorten <=0.5pt] (b2) to (b3);
  \draw[->, draw=teal!25, double=teal!25, double distance=1pt,
        line width=1.5pt, shorten >=0.5pt, shorten <=0.5pt] (b3) to (b1);
\end{scope}

    \node[fill=none, draw=none,text=black] (title) at (0,-1.95) {\textbf{Random walks}};
\end{scope}

\begin{scope}[scale=0.8, xshift=3cm, yshift=0.5cm,shorten >=1pt, 
        shorten <=1pt,]
    \node[fill=white, label=above:{$1$,}] (t1) at (0,0) {};
    \node[fill=white, label=above:{$2$,}] (t2) at (1.0,0) {};
    \node[fill=white, label=above:{$3$,}] (t3) at (2.0,0) {};
    \node[fill=white, label=above:{$4$\phantom{,}}] (t4) at (3.0,0) {};

    \draw[->, draw=red!25, double=red!25,  double distance=3pt,
        line width=0.5pt, shorten >=0.5pt, shorten <=0.5pt] (t1) -- (t2) node[midway, above,fill=none, draw=none,text=black, yshift=0.0em] {$\alpha^{\scalebox{0.4}{$-1$}}$,};
    \draw[->, draw=red!25, double=red!25,  double distance=3pt,
        line width=0.5pt, shorten >=0.5pt, shorten <=0.5pt] (t2) -- (t3) node[midway, above,fill=none, draw=none,text=black, yshift=0.2em] {$\beta$,};
    \draw[->, draw=red!25, double=red!25,  double distance=3pt,
        line width=0.5pt, shorten >=0.5pt, shorten <=0.5pt] (t3) -- (t4) node[midway, above,fill=none, draw=none,text=black, yshift=0.2em] {$\gamma$,};

    \draw[->, color=black, line width=0.3pt] (t2) -- (t1)  ;
    \draw[->, color=black, line width=0.3pt] (t2) -- (t3) ;
    \draw[->, color=black, line width=0.3pt] (t3) -- (t4) ;

\begin{scope}[on background layer]
  \draw[draw=red!25, line width=2pt, fill=red!25] (t1) circle (0.2);
  \draw[draw=red!25, line width=2pt, fill=red!25] (t2) circle (0.2);
  \draw[draw=red!25, line width=2pt, fill=red!25] (t3) circle (0.2);
  \draw[draw=red!25, line width=2pt, fill=red!25] (t4) circle (0.2);
\end{scope}
    
    \node[fill=white, label=below:{$1$,}] (b1) at (0,-1) {};
    \node[fill=white, label=below:{$2$,}] (b2) at (1.0,-1) {};
    \node[fill=white, label=below:{$3$,}] (b3) at (2.0,-1) {};
    \node[fill=white, label=below:{$4$}\phantom{,}] (b4) at (3.0,-1) {};

    \draw[->, draw=teal!25, double=teal!25,  double distance=3pt,
        line width=0.5pt, shorten >=0.5pt, shorten <=0.5pt] (b1) -- (b2) node[midway, below,fill=none, draw=none,text=black, yshift=-0.4em] {$\alpha$,};
    \draw[->, draw=teal!25, double=teal!25,  double distance=3pt,
        line width=0.5pt, shorten >=0.5pt, shorten <=0.5pt] (b2) -- (b3) node[midway, below,fill=none, draw=none,text=black, yshift=-0.4em] {$\alpha$,};
    \draw[->, draw=teal!25, double=teal!25,  double distance=3pt,
        line width=0.5pt, shorten >=0.5pt, shorten <=0.5pt] (b3) -- (b4) node[midway, below,fill=none, draw=none,text=black, yshift=-0.05em] {$\alpha^{\scalebox{0.4}{$-1$}}$,};

    \draw[->, color=black, line width=0.3pt] (b1) -- (b2)  ;
    \draw[->, color=black, line width=0.3pt] (b2) -- (b3) ;
    \draw[->, color=black, line width=0.3pt] (b4) -- (b3) ;

    \node[fill=none, draw=none,text=black] (bracket1left) at (-0.3,0.42) {$[$};
    \node[fill=none, draw=none,text=black] (bracket1left) at (3.2,0.42) {$]$};

    \node[fill=none, draw=none,text=black] (bracket1left) at (-0.3,-1.40) {$[$};
    \node[fill=none, draw=none,text=black] (bracket1left) at (3.2,-1.40) {$]$};

\begin{scope}[on background layer]
  \draw[draw=teal!25, line width=2pt, fill=teal!25] (b1) circle (0.2);
  \draw[draw=teal!25, line width=2pt, fill=teal!25] (b2) circle (0.2);
  \draw[draw=teal!25, line width=2pt, fill=teal!25] (b3) circle (0.2);
  \draw[draw=teal!25, line width=2pt, fill=teal!25] (b4) circle (0.2);
\end{scope}

    \node[fill=none, draw=none,text=black] (title) at (1.5,-2.5) {\textbf{Recording protocol}};
\end{scope}

\begin{scope}[scale=0.8, xshift=7.2cm, yshift=0.5cm, shorten >=1pt, shorten <=1pt]
  \node[fill=vaguegray!40, draw=none]       (t1) at (0,0)   {}; 
  \node[fill=yellow]        (t2) at (1.0,0) {}; 
  \node[fill=vaguegray!40, draw=none]       (t3) at (2.0,0) {}; 
  \node[fill=vaguegray!40, draw=none]       (t4) at (3.0,0) {}; 

  \draw[->, color=green!50!black, opacity=0.4,    line width=0.7pt] (t2) -- (t1);
  \draw[->, color=blue!70!black,  line width=0.7pt] (t2) -- (t3);
  \draw[->, color=red, opacity=0.4,  line width=0.7pt] (t3) -- (t4);

  \node[fill=red]           (b1) at (0,-1)  {}; 
  \node[fill=vaguegray!40, draw=none]     (b2) at (1.0,-1){}; 
  \node[fill=vaguegray!40, draw=none]  (b3) at (2.0,-1){}; 
  \node[fill=vaguegray!40, draw=none] (b4) at (3.0,-1){}; 

  \draw[->, color=blue!30,   line width=0.7pt] (b1) -- (b2);
  \draw[->, color=blue!40, line width=0.7pt] (b2) -- (b3);
  \draw[->, color=blue!70,   line width=0.7pt] (b4) -- (b3);

\begin{scope}[on background layer]

\node[circle, draw=vaguegray!40, fill=vaguegray!40,
      minimum size=0.5cm, inner sep=0pt] at (t2) {};

\node[circle, draw=vaguegray!40, fill=vaguegray!40,
      minimum size=0.5cm, inner sep=0pt] at (b1) {};

\end{scope}

  \node[fill=none, draw=none, text=black] (bracket1left) at (1.5,-2.5) {\textbf{Sequence processor}};
\end{scope}

    \begin{scope}[scale=0.8, xshift=13cm]
         \node[circle, fill=orange, inner sep=2pt, line width=0.5pt] (s) at (0,0) {};
        \node[color=vaguegray, opacity=0.4, draw=none] (t) at (0,1.5) {};
        \node[color=vaguegray,opacity=0.4, draw=none]  (a1) at (-1, 1) {};
        \node[color=vaguegray,opacity=0.1, draw=none]  (a2) at (-1,-1){};
        \node[color=vaguegray,opacity=0.4, draw=none]  (a3) at (-2, 0) {};
        \node[color=vaguegray,opacity=0.4, draw=none]  (b1) at (1, 1) {};
        \node[color=vaguegray,opacity=0.4, draw=none]  (b2) at (1,-1){};
        \node[color=vaguegray,opacity=0.4, draw=none]  (b3) at (2, 0) {};
        \draw[->, color=green!50!black, opacity=0.4] (s) to (t);
        \draw[->, color=blue] (s) to (a1);
        \draw[->, color=vaguegray, opacity=0.1] (s) to (a2);
        \draw[->, color=red, opacity=0.4] (a1) to (a3);
        \draw[->, color=vaguegray, opacity=0.1] (a2) to (a3);
        \draw[->, color=vaguegray, opacity=0.1] (s) to (b1);
        \draw[->, color=blue] (s) to (b2);
        \draw[->, color=blue] (b1) to (b3);
        \draw[->, color=blue] (b2) to (b3);

\begin{scope}[on background layer]
  \node[circle, draw=vaguegray!40, fill=vaguegray!40,
        minimum size=0.5cm, inner sep=0pt] (s) at (0,0) {};
\end{scope}

    \node[fill=none, draw=none,text=black] (bracket1left) at (0,-2) {\textbf{Consensus protocol}};
    \end{scope}

\end{tikzpicture}
\vspace{-3em}
\caption{\textbf{An overview.} In each updating step, 
$\flock$ \textbf{(1)} samples random walks on the KG \draft{(two walks indicated by \textcolor{white!40!red}{red} and \textcolor{white!30!teal}{teal}, respectively)}, \textbf{(2)} anonymizes the encountered nodes and relations via a recording protocol \draft{(for each walk, nodes are anonymized as $1, 2, ...$ and relations as $\alpha,\beta,...$)}, and \textbf{(3)} feeds the sequences in a sequence processor to compute node and relation representations. \textbf{(4)} A consensus protocol then pools them back to the original KG’s nodes and relations.}
    \label{fig:pipeline}
\end{figure}

We present $\flock$, a KGFM respecting probabilistic node-relation invariance.
$\flock$ is a randomized function $X_\theta(\cdot)$ which takes as input a KG $G = (V, E, R)$ and a link prediction query $q$.
We consider two types of queries: {\em entity prediction} $q=(h, r, ?)$ and {\em relation prediction} $q=(h, ?, t)$.
$\flock$ outputs a random variable $\hat{{\bf y}}\sim X_\theta(G, q)$ which is suited for the task at hand.
For entity prediction, it outputs $\hat{{\bf y}}:V\to[0,1]$ such that a potential link $(h, r, t)$ can be evaluated by $\hat{{\bf y}}(t)\in[0, 1]$.
For relation prediction, it outputs $\hat{{\bf y}}:R\to[0,1]$ such that a link $(h, r, t)$ can be evaluated by $\hat{{\bf y}}(r)$.

\draft{At test time, we average multiple (\(P\)) independent stochastic predictions to produce the final output.
This improves performance and reduces variance through an ensembling effect.}

We describe the architecture of $\flock$ in \Cref{sec:architecture} focused on four main components, and then analyze its theoretical properties in \Cref{sec:theoretical_properties}, showing universality and probabilistic equivariance.
\draft{An expansion on the model details can be found in \Cref{app:methodology_details}.}

\subsection{Flock}\label{sec:architecture}
Internally, $\flock$ has two lookup tables of hidden states, ${\bf v}:V\to \mathbb{R}^d$ for entities and ${\bf r}:R\to \mathbb{R}^d$ for relations, respectively.
At each forward pass, it starts from trained initializations of these states ${\bf v}^{(0)}(\cdot)\coloneqq {\bf v}_0$ and ${\bf r}^{(0)}(\cdot)\coloneqq {\bf r}_0$, and updates them iteratively ${\bf v}^{(i)}, {\bf r}^{(i)}$ for $i\leq I$.
Each update is done residually using a randomized function $U_{\theta_i}$:
\begin{align*}
    {\bf v}^{(i+1)} \coloneqq {\bf v}^{(i)} + \Delta{\bf v},\qquad
    {\bf r}^{(i+1)} \coloneqq {\bf r}^{(i)} + \Delta{\bf r},\qquad
    (\Delta{\bf v}, \Delta{\bf r}) \sim {\rm update}_{\theta_i}({\bf v}^{(i)}, {\bf r}^{(i)}).
\end{align*}
The final hidden states ${\bf v}^{(I)}:V\to\mathbb{R}^d$ and ${\bf r}^{(I)}:R\to\mathbb{R}^d$ are then processed by a binary classifier ${\rm head}:\mathbb{R}^d\to[0,1]$ to produce the output $\hat{\bf y}$ which is $V\to[0, 1]$ or $R\to[0, 1]$ depending on task.

We now describe the randomized ${\rm update}_\theta$.
We drop $i$ for brevity.
It consists of four components:
\begin{enumerate}
    \item \textbf{Random walk algorithm} produces $n$ random walks $\eta_1, ..., \eta_n$ of length $\ell$ on the input KG.
    \item \textbf{Recording protocol} $w:\eta_j\mapsto {\bf z}_j$ transforms each walk into a graph-agnostic sequence.
    \item \textbf{Sequence processor} $f_\theta:{\bf z}_j\mapsto {\bf h}_j$ processes each sequence independently, outputting features.
    \item \textbf{Consensus protocol} $c:({\bf h}_{1:N}, \eta_{1:N})\mapsto(\Delta{\bf v}, \Delta{\bf r})$ collects features of all walks and decides state updates for each entity and relation type.
\end{enumerate}
An overview is presented in Figure~\ref{fig:pipeline}.
We note that $w$, $f_\theta$, and $c$ are all deterministic, and the random walk is the only source of stochasticity.
We now discuss the design choice for each.
For the ease of exposition, we explain for entity prediction tasks $q=(h,r,?)$, but relation prediction is similar.

\paragraph{Random walks.}
In $\flock$, random walks are central in two ways: they rewrite the connectivity of nodes and relations as sequences, and support generalization via probabilistic equivariance.

Formally, the random walk algorithm produces $n$ random walks $\eta_1,...,\eta_n$ of length $\ell$ on KG $G$.
Each random walk $\eta$ is a chain of random variables, written as:n as:
\begin{align*}
    \eta = v_0 \xrightarrow{r_1} v_1 \xrightarrow{r_2} \cdots \xrightarrow{r_\ell} v_\ell,\qquad v_s\in V, r_s\in R,(v_{s-1},r_s,v_s)\in E,
\end{align*}
where the underlying transition mechanism and $\ell$ are hyperparameters.

To support probabilistic equivariance, we ask the walk algorithm to be invariant in probability.
We say $\eta$ is invariant in probability if for any $G \simeq H$ in $\mathbb{K}_{n,m}$ with isomorphism $(\pi, \phi)$ from $G$ to $H$:
\begin{align*}
    \pi(v_0) \xrightarrow{\phi(r_1)} \pi(v_1) \xrightarrow{\phi(r_2)} \cdots \xrightarrow{\phi(r_\ell)} \pi(v_\ell) \probeq u_0 \xrightarrow{s_1} u_1 \xrightarrow{s_2} \cdots \xrightarrow{s_\ell} u_\ell,
\end{align*}
where $v_0 \xrightarrow{r_1} \cdots \xrightarrow{r_\ell} v_\ell$ and $u_0 \xrightarrow{s_1} \cdots \xrightarrow{s_\ell} u_\ell$ follow the distributions of $\eta(G, \ell)$ and $\eta(H, \ell)$, respectively.
In such case, the isomorphism $(\pi, \phi)$ yields a natural translation from walks in $G$ to $H$.

\draft{In $\flock$, we use a simple random walk algorithm which we show to be invariant in probability.
Specifically, we use uniform walks with non-backtracking, with minor modifications to handle directed multi-edges of KGs.
Despite the simplicity, we find that this choice works well in practice, consistent with findings of prior works \citep{tonshoff_crawl, kim_rwnn}.}

\draft{Under this choice, we diversify the starting locations of walks such that local context around the query $q$ and broad regions of the nodes and relations in a KG are both well-captured. Our \emph{diversification strategy} is as follows:} given a base walk count $n$, for entity prediction queries $(h,r,?)$, we use $3n$ walks with three types of start locations.
\draft{The first $n$ walks start at query node $h$, capturing local context around the query; the second $n$ walks start by traversing a random edge $(v, s, u)$ where $s$ is a uniformly chosen relation, broadly capturing the relations of the KG including $r$; the last $n$ walks start at random nodes, broadly capturing various regions of the KG.
For relation prediction queries $(h,?,t)$, we additionally start $n$ walks at the tail node $t$, sampling a total of $4n$ walks.}

We lastly discuss how to choose the base walk count $n$.
While this is a fixed hyperparameter $n_{\rm train}$ at pretraining, we find that scaling it adaptively to input KG at test-time benefits size generalization. We thus propose \emph{test-time adaptation of walk counts}, and use:
\begin{align}
\label{eq:harmonic_mean}
    n = n_{\rm train}\times {\rm harmonic}\ {\rm mean}\left(\frac{|V|}{|V|_{\rm train}}, \frac{|E|}{|E|_{\rm train}}\right),
\end{align}
where $|V|_{\rm train},|E|_{\rm train}$ are average numbers of nodes and edges in pretraining KGs, respectively.
Intuitively, this scales $n$ proportionally to the sizes of test KGs relative to pretraining.
In practice, we clamp $n$ to the nearest power of $2$ \draft{and limit its value in an interval to avoid out-of-memory errors}.

\paragraph{Recording protocol.}
While random walks provide a basis for invariant sequence representations of KGs, two issues remain: (1) They reveal nodes $v_s$ and relations $r_s$ specific to each KG which obstructs transferability to unseen KGs. (2) They do not offer a way to condition on current states of entities ${\bf v}$, relations ${\bf r}$, and the query $q=(h,r,?)$ as often done in KGFMs via the labeling trick.

The recording protocol $w:\eta_j\mapsto {\bf z}_j$ resolves this by transforming each walk into a \emph{graph-agnostic} sequence that only leaves structural information.
Motivated by prior works on node anonymization for invariance~\citep{kim_rwnn, wang2024rum}, we propose an extension called node-relation anonymization: reserve separate namespaces for nodes and relations, respectively, and assign unique names in the order of discovery.
For example, with $1, 2, 3, ...$ for nodes and $\alpha, \beta,...$ for relations:
\begin{align*}
    \eta =v_0 \xrightarrow{r_1} v_1 \xrightarrow{r_2} v_2 \xrightarrow{r_1^{-1}} v_0 \qquad \mapsto \qquad 1 \xrightarrow{\alpha} 2 \xrightarrow{\beta} 3 \xrightarrow{\alpha^{-1}} 1,
\end{align*}
where $(\cdot)^{-1}$ marks direction of a relation.
The protocol additionally employs a simple conditioning on current states $({\bf v},{\bf r})$ and query $q=(h,r,?)$, completing the record ${\bf z}$ as follows:
\begin{align}
    w:\eta \mapsto {\bf z} = (1, {\bf v}(v_0), {\bf 1}_h(v_0)) \xrightarrow{\alpha, {\bf r}(r_1), {\bf 1}_r(r_1)} (2,{\bf v}(v_1), {\bf 1}_{h}(v_1)) \xrightarrow{\beta, {\bf r}(r_2),{\bf 1}_r(r_2)}\cdots,\label{eq:record_example}
\end{align}
with indicator functions ${\bf 1}_h(\cdot), {\bf 1}_r(\cdot)$ at $h$ and $r$, respectively.
As we will show, the recording protocol keeps node-relation invariance by hiding nodes and relations while leaving their structural roles.

\paragraph{Sequence processor.}
Now that the recordings ${\bf z}$ only encode structural information of KG, we can safely process them with an arbitrary neural network $f_\theta:{\bf z}\mapsto{\bf h}$ without the risk of losing invariance.
Since ${\bf z}$ are sequences, we choose sequence networks to leverage their inductive bias.
Specifically, we use bidirectional GRU~\citep{cho2014properties} equipped with RMSNorm~\citep{zhang2019root} and SwiGLU feedforward network~\citep{shazeer2020glu}, which provided robust results.
\draft{To convert anonymizations into input feature vectors to the GRU, we use trainable embedding tables.}

Given that $f_\theta$ is a sequence network, it is convenient to interpret its output ${\bf h}$ as positionally aligned with each step of the walk $\eta$ or record ${\bf z}$.
Specifically, for the example in \Eqref{eq:record_example}, we obtain:
\begin{align*}
    f_\theta:{\bf z} \mapsto {\bf h} = (\Delta{\bf v}_0, a_0) \xrightarrow{\Delta{\bf r}_1, b_1} (\Delta{\bf v}_1, a_1) \xrightarrow{\Delta{\bf r}_2, b_2} \cdots.
\end{align*}
where $\Delta{\bf v}_s,\Delta{\bf r}_s\in\mathbb{R}^{h\times d_h}$ and $a_s,b_s\in\mathbb{R}^h$ are the decoded outputs at each position using linear projections.
Intuitively, $\Delta{\bf v}_s,\Delta{\bf r}_s$ encode proposals of state updates for entities and relations by $f_\theta$, and $a_s, b_s$ encode respective confidences of $f_\theta$ for the proposed updates.
This separation is useful due to the localized, pure-structure nature of the recordings $\mathbf{z}$.
If a random walk $\eta$ densely visited a cycle-like region and then terminated in a dangling manner, it is natural to assign more confidence to the cycle-like region of the structural encodings ${\bf h}$, and less confidence to the dangling region.

\paragraph{Consensus protocol.}
After sequence processing, we are left with a handful of state update proposals ${\bf h}_{1:N}$ from $f_\theta$, that are positionally aligned with random walks $\eta_{1:N}$ on KG.
The consensus protocol $c$ uses the information to decide final state updates $\Delta{\bf v}:V\to\mathbb{R}^d$ and $\Delta{\bf r}:R\to\mathbb{R}^d$.

Since $c$ can access how each $\Delta{\bf v}_s$ within ${\bf h}_j$ is associated to a node $v_s\in V$ (and how each $\Delta{\bf r}_s$ is associated to a relation $r_s\in R$) through the random walk $\eta_j$, a simple way to form a consensus is by finding all proposals $\{\Delta{\bf v}_s\}$ associated to each node $v$, and all $\{\Delta{\bf r}_s\}$ associated to each relation $r$, and take averages of these proposals.
The drawback is that uninformative proposals from e.g., dangling regions of walks are not directly suppressed, and can affect the state updates.

We can leverage the confidences $a_s,b_s$ from $f_\theta$ to alleviate this issue.
For each node $v\in V$ or relation $r \in R$, we first find all respective associated pairs $\{(\Delta{\bf v}_s,a_s)\}$ or $\{(\Delta{\bf r}_s,b_s)\}$ of proposals and confidences, and compute a multi-head softmax-normalized weighted average:
\begin{align*}
    \Delta{\bf v}(v)\coloneqq \left[\sum\exp(a_s) \odot\Delta{\bf v}_s\right] \oslash \sum\exp(a_s)\quad \Delta{\bf r}(r)\coloneqq \left[\sum\exp(b_s)\odot \Delta{\bf r}_s\right] \oslash \sum\exp(b_s),
\end{align*}
where $\odot$ and $\oslash$ are row-wise multiplication and division, respectively.
Intuitively, this normalization induces competition between state update proposals, naturally leading to uninformative proposals being suppressed.
Similar ideas are presented by \citet{locatello2020object}.

Again, we can show that the consensus protocol does not operate in a way specific to particular KGs, and hence retains node-relation equivariance.

\subsection{Theoretical analysis}
\label{sec:theoretical_properties}

\textbf{Expressivity. }
Following the notion of probabilistic expressivity introduced by \cite{abboudsurprising}, we say that a $\flock$ model $X_\theta$ is a universal approximator of link invariant functions over $\mathbb{K}_{n,m}$ if for any link invariant $\varphi:\mathbb{K}_{n,m} \rightarrow (V\times R\times V \rightarrow [0,1])$ and any $\epsilon, \delta > 0$, there exists a choice of the network parameters $\theta$ and the length of the sampled random walks $\ell$, such that:
\begin{align*}
    \sP(|\varphi(G)((h,r,t)) - X_\theta(G, (h,r,?))(t)| < \epsilon) > 1-\delta 
\end{align*}
for all graphs $G=(V,E,R) \in \mathbb{K}_{n,m}$ and all links $(h,r,t) \in V\times R\times V$.

\begin{proposition}
    \label{prop:universal_approximator}
    With a powerful enough sequence processor $f_\theta$, the $\flock$ framework described above is a universal approximator of link invariant functions over $\mathbb{K}_{n,m}$ for all pairs $(n,m)$.
\end{proposition}

\draft{All proofs are in \Cref{app:proofs}.
To offer an intuition behind the result, we provide a proof sketch.}

\draft{
\emph{Proof sketch.}  A sufficiently long random walk will cover all edges of the graph with high probability. Then, from its anonymized version, assigning unique positional identifiers to every node and relation, one can reconstruct the input graph, up to isomorphism. Thus, with a sufficiently expressive sequence processor, $\flock$  can approximate any link-invariant function.
}

\textbf{Invariance. }
Despite the stochastic nature of our framework, beyond randomized node embeddings \citep{abboudsurprising}, $\flock$ can be provably guaranteed to satisfy probabilistic invariance:

\begin{proposition}
\label{prop:general_invarinace}
    Suppose that the walk sampling protocol $\eta$ is invariant in probability and both the recording protocol $w$ and the consensus protocol $c$ are invariant.
    Then, regardless of the choice of the deterministic sequence processor $f_
    \theta$, the corresponding $\flock$ model is invariant in probability.
\end{proposition}

\draft{
\emph{Proof sketch.} Since each of these components is invariant (in probability), and invariance of individual component is preserved under composition, we have that $\flock$ is invariant.
}

Moreover, the designs of $\flock$'s components provided earlier in this section satisfy the conditions of Proposition \ref{prop:general_invarinace}.
Therefore, the suggested pipeline is indeed invariant in probability:

\begin{proposition}
\label{prop:specific_invariance}
    Any $\flock$ model with components as outlined in this section, and detailed in \Cref{app:methodology_details} is invariant in probability.
\end{proposition}

\section{Experiments}
\label{sec:experiment}
We test $\flock$ over a wide range of KGs for both entity and relation predictions, aiming to answer:

\begin{itemize}
    \item[\textbf{Q1.}]
    Can $\flock$ approximate functions that existing KGFMs cannot? 
    \item[\textbf{Q2.}] How does $\flock$ generalize to unseen entities and relations compared to existing KGFMs? 
    \item[\textbf{Q3.}] How does performance scale with the sizes of pretraining graph mix and test-time ensemble?
    \item[\textbf{Q4.}] What is the impact of \draft{choices of each component on the behavior and performance of $\flock$}?
\end{itemize}

We further provide detailed scalability analysis in \Cref{sec:complexity}~and~\Cref{app:scalability_analysis}, and comparisons against current KGFMs augmented with noise injection in \Cref{sec:noise_injection}.
Experiment details and hyperparameters are in \Cref{app:experimental_details}.

\subsection{Synthetic dataset}
\label{sec:synthetic_main}

\begin{wrapfigure}{r}{0.38\textwidth}
  \centering
  \vspace{-1.2\baselineskip} 
  \begin{tikzpicture}[
      scale=1,
      every node/.style={circle, draw, fill=black, inner sep=2pt, line width=0.5pt}, 
      line width=1pt, 
      >=Stealth,
      shorten >=3pt, 
      shorten <=3pt,
      transform shape
  ]
    \node[label={-90:$s$}] (s) at (0,0) {};
    \node (t) at (0,1.5) {};
    \node[label={90:$t_1$}] (a1) at (-1, 1) {};
    \node[label={-90:$t_2$}] (a2) at (-1,-1) {};
    \node (a3) at (-2, 0) {};
    \node (b1) at ( 1, 1) {};
    \node (b2) at ( 1,-1) {};
    \node (b3) at ( 2, 0) {};

    \draw[->, color=solidgreen] (s) to (t);

    \draw[->, color=likeblue]   (s) to (a1);
    \draw[->, color=dislikered] (s) to (a2);
    \draw[->, color=dislikered] (a1) to (a3);
    \draw[->, color=dislikered] (a2) to (a3);

    \draw[->, color=dislikered] (s) to (b1);
    \draw[->, color=likeblue]   (s) to (b2);
    \draw[->, color=likeblue]   (b1) to (b3);
    \draw[->, color=likeblue]   (b2) to (b3);

    \draw[->, color=solidgreen, dashed] (s) to[in=-80, out=170] (a1);
    \draw[->, color=solidgreen, dashed] (s) to[in= 80, out=190] (a2);
  \end{tikzpicture}
  \caption{A KG in \textsc{Petals}. KGFMs with relational invariants must equate \textcolor{likeblue}{$r_1$} and \textcolor{dislikered}{$r_2$}, thus predicting the same scores for dashed queries with \textcolor{solidgreen!80!black}{$r_0$}.}
  \label{fig:synthetic-example}
\end{wrapfigure}

\paragraph{Setup.}
To validate the limitations of KGFMs with node-relation equivariance (\textbf{Q1}), we construct a synthetic benchmark \textsc{Petals}.
It contains $220$ instances, each including: \textbf{(1)} a KG $G = (V,E,R)$ consisting of a `central' node~$s$, a `stem' $T\subset V$ with query relation $r_0$, and multiple cyclic `petals', each `colored' with a different pair of relations in $R \setminus \{r_0\}$,
\textbf{(2)} an entity prediction query $(h,r_0,?)$ with $h\in \{s\}\cup T$,
and \textbf{(3)} two candidate targets $t_1$ and $t_2$ from the same `petal', located at the same distance from $s$.
An example is in \Cref{fig:synthetic-example}. See \Cref{app:synthetic-datasets} for more details.

\textsc{Petals} is designed such that each instance always admits non-trivial automorphisms, meaning that swapping relations occurring in the same `petal' results in an isomorphic KG.
Consequently, any model computing relation invariants will not be able to distinguish between potential links $(s,r_0,t_1)$ and $(s,r_0,t_2)$. 
However, the samples are constructed so that these links are not isomorphic from the graph perspective, making them distinguishable for general link-invariant functions.
We say a model \emph{solves} an instance if it can classify $(s,r_0,t_1)$ as \textsc{true} and $(s,r_0,t_2)$ as \textsc{false}.
For empirical validation, we train $\ultra$~\citep{galkin2023ultra}, $\mgnn(\gF_\text{Path}^{3})$~\citep{huang2025how}, $\trix$~\citep{zhang2024trix}, and $\flock$ from scratch and validate them on the training instances.

\begin{wraptable}{r}{0.32\textwidth}
  \centering
  \footnotesize
  \vspace{-2.2em}
  \caption{\textsc{Petals} accuracies.}
  \vspace{-0.8em}
  \begin{tabular}{lc}
    \toprule
    \textbf{Model} & \textsc{Petals} \\
    \midrule
    \ultra & 50\% \\
    \mgnn($\gF_\text{Path}^{3}$) & 50\% \\
    \trix & 50\% \\
    \midrule
    $\flock$ & \textbf{100\%} \\
    \bottomrule
  \end{tabular}
  \vspace{-2em}
  \label{tab:synthetic-exp-results}
\end{wraptable}

\paragraph{Results.} 
The results are provided in \Cref{tab:synthetic-exp-results}.
As expected, all existing KGFMs relying on learning deterministic relational invariants fail to distinguish between the candidate target triplets completely, achieving $50\%$ accuracy due to random guesses.
In contrast, $\flock$ succeeds on \textit{all} considered instances, displaying that, while remaining invariant in probability, it can differentiate between non-isomorphic links, even with isomorphic relations. 

\subsection{Entity and relation prediction over knowledge graphs}
\label{sec:main_experiment}

\begin{table}[t!]
\footnotesize
\vspace{-0.5cm}
\setlength{\tabcolsep}{4pt}
\centering
\caption{Average entity prediction MRR and Hits@10 over 54 KGs from distinct domains.}
\vspace{-0.2cm}
\label{tab:main-entity-result}
\begin{tabular}{@{}lcccccc||cc||cc@{}}
\toprule
& \multicolumn{2}{c}{\textbf{Inductive $e, r$}} & \multicolumn{2}{c}{\textbf{Inductive $e$}} & \multicolumn{2}{c}{\textbf{Transductive}} & \multicolumn{2}{c}{\textbf{Total Avg}} & \multicolumn{2}{c}{\textbf{Pretrained}} \\
\textbf{Model} & \multicolumn{2}{c}{(23 graphs)} & \multicolumn{2}{c}{(18 graphs)} & \multicolumn{2}{c}{(13 graphs)} & \multicolumn{2}{c}{(54 graphs)} & \multicolumn{2}{c}{(3 graphs)} \\
 \cmidrule(lr){2-3} \cmidrule(lr){4-5} \cmidrule(lr){6-7} \cmidrule(lr){8-9} \cmidrule(lr){10-11}
& \textbf{MRR} & \textbf{H@10} & \textbf{MRR} & \textbf{H@10} & \textbf{MRR} & \textbf{H@10} & \textbf{MRR} & \textbf{H@10}  & \textbf{MRR} & \textbf{H@10} \\
\midrule
$\ultra$ (zero-shot) & 0.345 & 0.513 & 0.431 & 0.566 & 0.312 & 0.458 & 0.366 & 0.518 & - & -\\
$\trix$ (zero-shot) & 0.368 & 0.540 & 0.455 & 0.592 & 0.339 & 0.500 & 0.390 & 0.548 & - & -\\
$\flock$ (zero-shot) &  \textbf{0.369} & \textbf{0.554} & \textbf{0.456} & \textbf{0.604} & \textbf{0.340} & \textbf{0.509} & \textbf{0.391} & \textbf{0.560} & - & -\\
\midrule
$\ultra$ (finetuned) & 0.397 & 0.556 & 0.440 & 0.582 & 0.379 & 0.543 & 0.408 & 0.562 & 0.407 & \textbf{0.568}\\
$\trix$ (finetuned) & 0.401 & 0.556 & 0.459 & 0.595 & \textbf{0.390} & \textbf{0.558} & 0.418 & 0.569 & \textbf{0.415} & 0.564\\
$\flock$ (finetuned) & \textbf{0.417} & \textbf{0.576} & \textbf{0.473} & \textbf{0.619} & 0.383 & 0.544 & \textbf{0.427} & \textbf{0.582} & \textbf{0.415} & 0.561\\
\bottomrule
\end{tabular}
\vspace{-0.1cm}
\end{table}

\begin{table*}[t!]
\footnotesize
\setlength{\tabcolsep}{4pt}
\centering
\caption{Average relation prediction MRR and Hits@1 over 54 KGs from distinct domains.}
\vspace{-0.2cm}
\label{tab:main_relation_prediction}
\begin{tabular}{@{}lcccccc||cc||cc@{}}
\toprule
& \multicolumn{2}{c}{\textbf{Inductive $e, r$}} & \multicolumn{2}{c}{\textbf{Inductive $e$}} & \multicolumn{2}{c}{\textbf{Transductive}} & \multicolumn{2}{c}{\textbf{Total Avg}} & \multicolumn{2}{c}{\textbf{Pretrained}}\\
\textbf{Model} & \multicolumn{2}{c}{(23 graphs)} & \multicolumn{2}{c}{(18 graphs)} & \multicolumn{2}{c}{(13 graphs)} & \multicolumn{2}{c}{(54 graphs)}  & \multicolumn{2}{c}{(3 graphs)} \\
\cmidrule(lr){2-3} \cmidrule(lr){4-5} \cmidrule(lr){6-7} \cmidrule(lr){8-9}\cmidrule(lr){10-11}
& \textbf{MRR} & \textbf{H@1} & \textbf{MRR} & \textbf{H@1} & \textbf{MRR} & \textbf{H@1} & \textbf{MRR} & \textbf{H@1}  & \textbf{MRR} & \textbf{H@1} \\
\midrule
$\ultra$ (zero-shot) & 0.785 & 0.691 & 0.714 & 0.590 & 0.629 & 0.507 & 0.724 & 0.613 & - & - \\
$\trix$ (zero-shot) & 0.842 & 0.770 & 0.756 & 0.611 & 0.752 & 0.647 & 0.792 & 0.687 & - & - \\
$\flock$ (zero-shot) & \textbf{0.898} & \textbf{0.846} & \textbf{0.864} & \textbf{0.782} & \textbf{0.873} & \textbf{0.813} & \textbf{0.881} & \textbf{0.817} & - & - \\
\midrule
$\ultra$ (finetuned) & 0.823 & 0.741 & 0.716 & 0.591 & 0.707 & 0.608 & 0.759 & 0.659 & 0.876 & 0.817 \\
$\trix$ (finetuned) & 0.850 & 0.785 & 0.759 & 0.615 & 0.785 & 0.693 & 0.804 & 0.706 & 0.879 & 0.797 \\
$\flock$ (finetuned) & \textbf{0.929} & \textbf{0.889} & \textbf{0.887} & \textbf{0.808} & \textbf{0.897} & \textbf{0.844} & \textbf{0.907} & \textbf{0.851} & \textbf{0.977} & \textbf{0.959} \\
\bottomrule
\end{tabular}
\vspace{-0.2cm}
\end{table*}

\paragraph{Setup.} To answer \textbf{Q2}, we follow the protocol of \citet{galkin2023ultra,zhang2024trix} and pretrain $\flock$ on FB15k-237~\citep{FB15k237}, WN18RR~\citep{Dettmers2018FB}, and CoDEx Medium~\citep{safavi-koutra-2020-codex}. We then evaluate its zero-shot and finetuned inference performance with the test set of 54 KGs (see \Cref{app:experimental_details} for details). These KGs are extracted from diverse domains across three settings: inductive on nodes and relations (\textbf{Inductive} $e,r$), inductive on nodes (\textbf{Inductive} $e$), and \textbf{transductive}. Note that these settings differ only during finetuning setup; in zero-shot setup, all entities and relations are unseen.
We choose state-of-the-art KGFMs $\ultra$ \citep{galkin2023ultra} and $\trix$ \citep{zhang2024trix} as baselines, as they are pretrained on the same KGs, to ensure a fair comparison.
%
For evaluation, we use the filtered ranking protocol \citep{brodes2013transe}, reporting mean reciprocal rank (MRR) and Hits@10 for entity prediction, and Hits@1 for relation prediction, as some KGs have fewer than 10 relations. Per-dataset results are in \Cref{app:experimental_details}.

\paragraph{Entity prediction.}
\Cref{tab:main-entity-result} shows the entity prediction results.
In the zero-shot setting, $\flock$ consistently outperforms $\ultra$ and $\trix$, demonstrating strong generalization across diverse domains. Notably, on \emph{Metafam} \citep{zhou2023multitaskperspetivelinkprediction}, a dataset designed to challenge models with conflicting and compositional relational
patterns, $\flock$ roughly doubles MRR over $\ultra$ and achieves $\approx$ 40\% gain in MRR over $\trix$ in the zero-shot setting.
\draft{We find that $\flock$ distinguishes structurally similar but semantically conflicting relations while $\ultra$ fails (\Cref{app:case_study}), explaining the gain.}
These findings align with our hypothesis that probabilistic equivariance improves expressivity without sacrificing generalization. In the finetuning setting, we observe a similar pattern: $\flock$ maintains a consistent improvement over all datasets except transductive ones, where KGs are generally larger. We hypothesize that the gap stems from walk coverages. Unlike $\ultra$ and $\trix$, whose message passing guarantees a full neighborhood coverage around the queried node, $\flock$ relies on sampled walks that may not fully cover target nodes of interest.
\draft{We find that $\flock$ favors sparse KGs (\Cref{app:sparsity}), consistent with this hypothesis as random walks cover sparse graphs faster.}

\paragraph{Relation prediction.}
\Cref{tab:main_relation_prediction} presents the relation prediction results. $\flock$ substantially outperforms all existing KGFMs across all categories in the zero-shot setting, achieving an $11.2\%$ relative improvement in MRR compared to the best baseline $\trix$. $\flock$ shows a further performance boost of $12.8\%$ in the finetuned setting. We attribute this huge gain to $\flock$'s joint encoding of entities and relations during the updating step via the sequence encoder, while existing KGFMs, $\ultra$ and $\trix$, have separate update steps for entities and relations. This joint update mechanism yields more holistic representations of both entities and relations with minimal information loss.

\subsection{Scaling analysis}
\label{sec:scaling_main}

\paragraph{Size of pretraining graph mix.}
To assess whether $\flock$ benefits from more pretraining graph and data (\textbf{Q3}), we follow the setup of \citet{galkin2023ultra}, and pretrain $\flock$ on an increasing number of KGs. We then evaluate them on all $41$ inductive benchmarks for a fair comparison. We present the detailed pretraining graph mix in \Cref{tab:pretrain-mix-table}. As shown in \Cref{fig:pretrain_scaling}, $\flock$’s generalization improves consistently as the number of pretraining KGs increases, exhibiting clear scaling behavior, which is a core characteristic of being a foundation model.

\begin{figure}[t]
  \centering
  \begin{subfigure}[t]{0.45\linewidth}
    \centering
    \includegraphics[width=\linewidth]{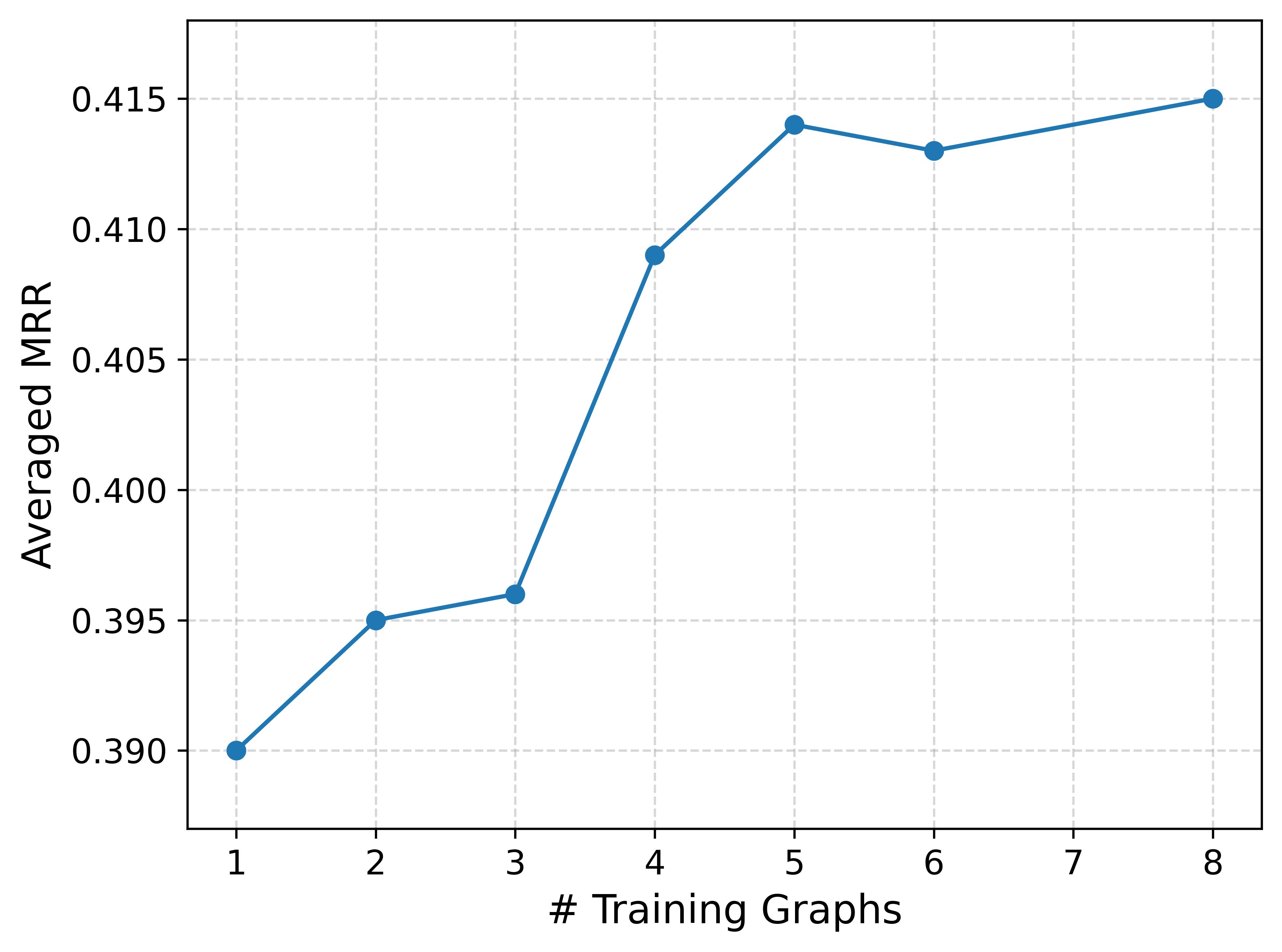}
    \caption{Zero-shot MRR vs. \#pretraining graphs.}
    \label{fig:pretrain_scaling}
  \end{subfigure}\hspace{0.5cm}
  \begin{subfigure}[t]{0.45\linewidth}
    \centering
    \includegraphics[width=\linewidth]{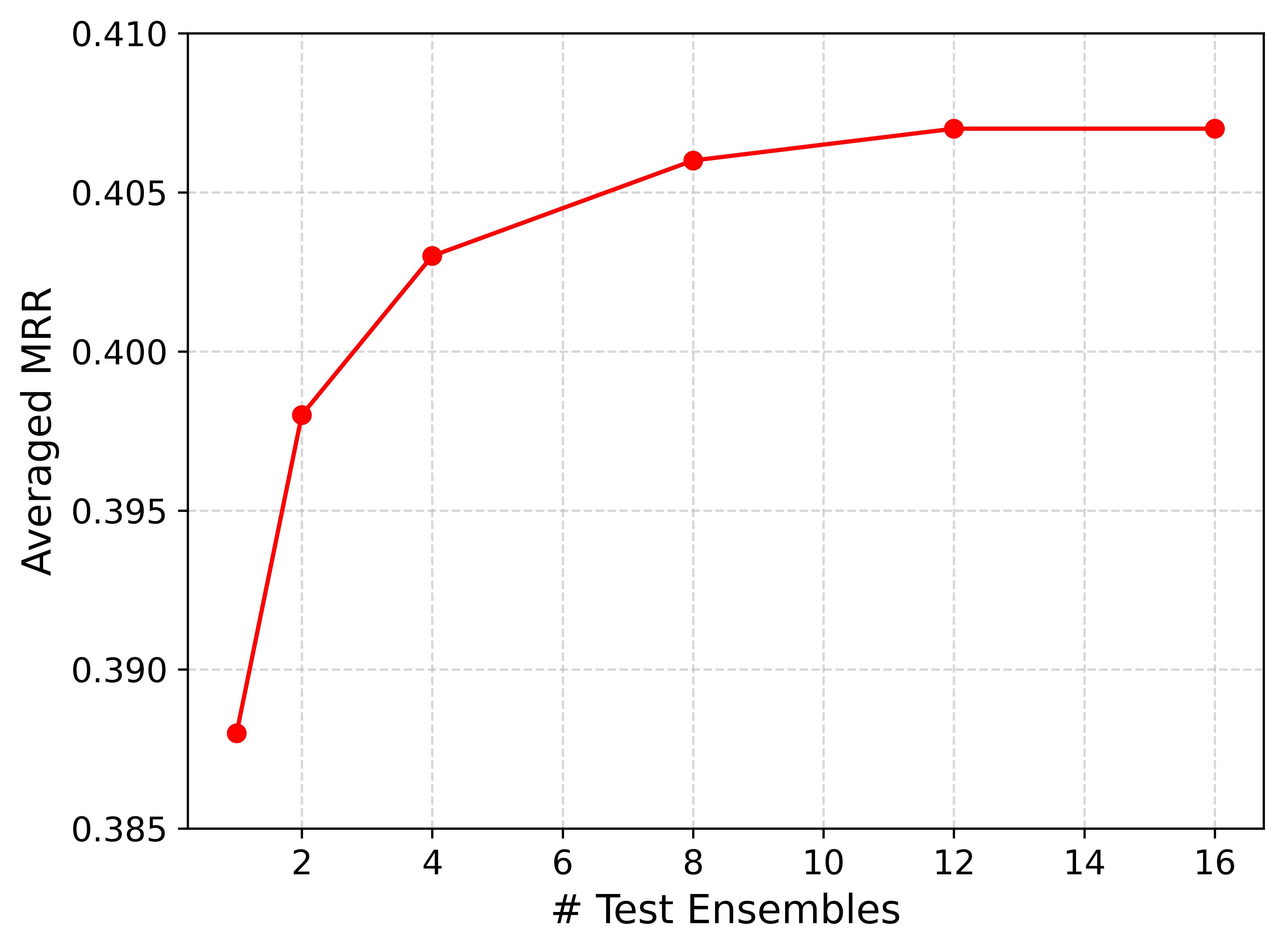}
    \caption{Zero-shot MRR vs. \#ensembled predictions.}
    \label{fig:ensemble_scaling}
  \end{subfigure}
  \caption{Pretraining and test-time scaling of $\flock$ on 41 inductive KG datasets.}
  \label{fig:pretrain_testing_scaling}
\end{figure}

\textbf{Number of ensembled predictions.}
To assess how test-time ensemble size $P$ affects performance (\textbf{Q3}), we take the pretrained $\flock$ and run zero-shot entity prediction on 41 inductive KGs by increasing the number of ensembled passes. As shown in \Cref{fig:ensemble_scaling}, performance improves from 1 to 8 passes and then begins to plateau beyond 12. This indicates a clear scaling behavior: larger ensembles provide a more accurate estimate of the underlying node and relation distributions.

\subsection{\draft{Ablation studies}}

\begin{table}[t]
\centering
\small
\caption{Ablation study of adaptive test-time walks with zero-shot entity prediction task. We show the average number of entities $|V|$, triples $|E|$, base walks $n$, MRR, and Hits@10.}
\label{tab:flock-adaptive}
\begin{tabular}{l ccccc ccccc}
\toprule
\multirow{2}{*}{\textbf{Dataset split}} & \multicolumn{2}{c}{\textbf{Statistics}} & \multicolumn{3}{c}{\flock} & \multicolumn{3}{c}{$\flock$ w/o Adap. \Cref{eq:harmonic_mean} } \\
\cmidrule(lr){2-3} \cmidrule(lr){4-6}  \cmidrule(lr){7-9}
& $|V|$ & $|E|$ & \(n\) & MRR & Hits@10  & \(n\) & MRR & Hits@10 \\
\midrule
\textbf{Inductive} \(e,r\) & 5{,}303  & 10{,}838 & 28.40 & \textbf{0.369} & \textbf{0.554} & 128 & 0.357 & 0.551 \\
\textbf{Inductive} \(e\)  & 7{,}578 & 29{,}090 & 18.08   & \textbf{0.456} & \textbf{0.604} & 128 & 0.441 & 0.596\\
\textbf{Transductive}   & 47{,}810 & 387{,}491  & 214.15 & \textbf{0.340} & \textbf{0.509} & 128 & 0.334 & 0.493\\
\bottomrule
\end{tabular}
\end{table}

\begin{table}[t]
\vspace{-0.2cm}
\centering
\small
\caption{\draft{Detailed ablation study with zero-shot entity prediction task.
For the transductive split, considering resource limits, we test NELL995, NELL23k, WDsinger, ConceptNet100k, and YAGO310.}}
\vspace{-0.2cm}
\label{tab:ablation}
\begin{adjustbox}{max width=\textwidth}
\begin{tabular}{@{}lcccccc||cc@{}}
\toprule
& \multicolumn{2}{c}{{\textbf{Inductive} $e, r$}} & \multicolumn{2}{c}{{\textbf{Inductive} $e$}} & \multicolumn{2}{c}{{\textbf{Transductive}}} & \multicolumn{2}{c}{\textbf{Total Avg}} \\
\textbf{Model} & \multicolumn{2}{c}{(23 graphs)} & \multicolumn{2}{c}{(18 graphs)} & \multicolumn{2}{c}{(5 graphs)} & \multicolumn{2}{c}{(46 graphs)}  \\
\cmidrule(lr){2-3} \cmidrule(lr){4-5} \cmidrule(lr){6-7} \cmidrule(lr){8-9} 
& {\textbf{MRR}} & \textbf{Hits@10} & {\textbf{MRR}} & \textbf{Hits@10} & {\textbf{MRR}} & {\textbf{Hits@10}} & \textbf{MRR} & \textbf{Hits@10} \\
\midrule
$\flock$ ($\ell = 128$) & {0.369} & {0.554} & {0.456} & {0.604} & {0.360} & {0.542} & \textbf{0.395} & \textbf{0.567} \\
\midrule
w/o non-backtracking & 0.370 & 0.549 & 0.456 & 0.605 & 0.334 & 0.499 & 0.386 & 0.551 \\
$\ell = 64$ & 0.372 & 0.556 & 0.459 & 0.606 & 0.351 & 0.534 & 0.394 & 0.565 \\
$\ell = 256$ & 0.360 & 0.548 & 0.458 & 0.605 & 0.338 & 0.508 & 0.385 & 0.553 \\
w/o diverse starts & 0.360 & 0.539 & 0.448 & 0.596 & 0.319 & 0.488 & 0.385 & 0.553 \\
transformer $f_\theta$ & 0.356 & 0.542 & 0.410 & 0.591 & 0.312 & 0.477 & 0.359 & 0.537 \\
w/o weighted consensus & 0.351 & 0.526 & 0.448 & 0.593 & 0.361 & 0.515 & 0.387 & 0.545 \\
\bottomrule
\end{tabular}
\end{adjustbox}
\vspace{-0.2cm}
\end{table}

\paragraph{Setup.} To understand the impact of each design choice on the performance and behavior of $\flock$~(\textbf{Q4}), we conduct a series of ablation studies spanning random walks, sequence processor and the consensus protocol, in the entity prediction task in the zero-shot setting.

\textbf{\draft{Adaptive test-time walks.}}
Recall that we employ \emph{test-time adaptation of walk counts}, which adaptively selects the base walk count $n$ based on the graph size, computed via the harmonic-mean rule shown in \Cref{eq:harmonic_mean} during inference.
\Cref{tab:flock-adaptive} compares this adaptive setting with a fixed setting that uses $128$ base walks per sample for all datasets, matching the pretraining setup ($n_{\text{train}} = 128$).
As expected, the average selected base count $n$ is smaller on both inductive splits and larger on the transductive split, yet the adaptive mechanism improves performance across all settings. 
This is consistent with the intuition that adaptive $n$ scales up walks on larger KGs to improve coverage while allocating fewer walks on smaller KGs to reduce redundant visits;
$\flock$ maintains comparable \emph{visiting rates} and \emph{coverage} to those seen during pretraining, thereby producing representations closer to the pretraining distribution and resulting in consistent performance gains.

\textbf{Non-backtracking walks.}
$\flock$ employs \emph{non-backtracking} uniform random walk, which has an effect of faster exploration and coverage of distant regions \citep{alon2007non}.
In \Cref{tab:ablation}, we compare this with uniform walk that may backtrack and hence is slower in global exploration.
While non-backtracking does not alter results much on inductive splits, it significantly improves performance on the transductive split.
This is consistent with the idea that improving coverage especially benefits performances on large KGs, which $\flock$ achieves via non-backtracking.

\textbf{Walk lengths.}
$\flock$ uses random walks of length $\ell=128$, a choice made by finding the lowest~$\ell$ reliably visiting target node and relation on various KGs.
\Cref{tab:ablation} compares this with shorter and longer walks by a factor of two.
As expected from coverage, shorter walks show degraded results on the transductive split with large KGs.
Longer walks are overall worse, which is explained by higher learning complexity of the sequence processor that has a small hidden dimension (64) for scalability.
$\flock$ finds a balance of coverage and learnability, achieving robust results on diverse KGs.

\textbf{Diverse starting locations of walks.}
We recall that $\flock$ uses a \emph{diversification} strategy of starting locations of walks, with $n$ walks from the query node, $n$ walks from random relation, and $n$ walks from random node, adding up to $3n$ walks capturing both local context near query and global information of KG.
\Cref{tab:ablation} compares this against all $3n$ walks starting at the query node.
As expected, this causes degradations on all splits, showing the benefit of using both local and global information.

\textbf{Sequence processor.}
$\flock$ uses a sequence processor $f_\theta$ with bidirectional GRU.
In \Cref{tab:ablation}, we compare this against a transformer $f_\theta$ with a similar SwiGLU-RMSNorm architecture \citep{meta2024llama3} and parameter count.
This alternative does not deliver good results, which is explained by the restrictions on model scales that are enforced to scale to large KGs.
$\flock$ benefits from reasoning efficiency of GRU in limited parameter regime, gaining good performance and scalability together.

\textbf{Consensus protocol.}
$\flock$ uses softmax-weighted averaging to pool sequence processor outputs into state updates for nodes and relations, under the intuition that this can suppress uninformative proposals from the sequence processor better than simple, unweighted averaging.
\Cref{tab:ablation} provides a comparison, showing that weighted consensus outperforms the unweighted counterpart.
This verifies our intuition on how the design of the consensus protocol strengthens $\flock$.

\section{Conclusions}
We introduced $\flock$, a knowledge graph foundation model that respects probabilistic node-relation equivariance. $\flock$ iteratively samples query-conditioned random walks, records encountered nodes and relations via a recording protocol, and relies on a sequence processor and consensus protocol to obtain node and relation representations. We evaluate $\flock$ on 54 KGs across different domains for both entity and relation prediction, demonstrating superior zero-shot and finetuned performance. We further construct a synthetic dataset $\textsc{petals}$ to validate our theoretical findings. 
One limitation is scalability (\Cref{app:scalability_analysis}): ensuring coverage of the sampled random walk in large KGs requires many long walks, which can quickly become computationally infeasible. A future direction is to develop approximation strategies~\draft{\citep{chamberlaingraph, lacki2020walking}} that reduce the cost of random walk sampling while retaining $\flock$'s downstream performance.
\draft{Another avenue for future work is studying the families of approximable functions when the walk lengths are restricted, for example based on connections to subgraph-based reconstructions \citep{cotta2021reconstruction}.}

\section*{Acknowledgment}
This work was in part supported by the National Research Foundation of Korea (RS-2024-00351212 and RS-2024-00436165) and the Institute of Information \& Communications Technology Planning \& Evaluation (IITP) (RS-2024-00509279, RS-2022-II220926, RS-2022-II220959, and RS-2019-II190075) funded by the Korean government (MSIT). Computational resources were in part provided by the “HPC support” project funded by MSIT and NIPA. MB is supported by EPSRC Turing AI World-Leading Research Fellowship No. EP/X040062/1 and EPSRC AI Hub on Mathematical Foundations of Intelligence: An ``Erlangen Programme'' for AI No. EP/Y028872/1.

\section*{Ethics statement}
\label{app:ethics_statement}

This work introduces a probabilistic framework for knowledge graph foundation models, aiming to improve the generalization in zero-shot link prediction. Our contributions are methodological and theoretical, with evaluations performed on publicly available benchmarks and a synthetic dataset designed to validate our theoretical results. We do not anticipate any direct ethical risks associated with this approach. We acknowledge and adhere to the \href{https://iclr.cc/public/CodeOfEthics}{ICLR Code of Ethics}.

\section*{Reproducibility statement}
\label{app:reproducibility_statement}

We make every effort to ensure the reproducibility of the experiments in our paper.  We release a codebase at \url{https://github.com/jw9730/flock} with training and evaluation scripts for \flock, including pretraining scripts and checkpoints on FB15k-237, WN18RR, and CoDEx Medium, evaluation over 54 KGs, and the synthetic dataset \textsc{Petals} generator. All architectural details needed to reimplement the method, including the random-walk sampler, recording protocol, sequence processor, and consensus protocol, are specified in \Cref{app:methodology_details}, and our theoretical claims are supported with complete proofs in \Cref{app:proofs}. We additionally include further experimental details in \Cref{app:experimental_details}.

\bibliography{iclr2026_conference}
\bibliographystyle{iclr2026_conference}

\appendix

\Crefname{appendix}{Appendix}{Appendices}
\crefalias{section}{appendix}
\crefalias{subsection}{appendix} 

\clearpage

\section{Methodology - details}
\label{app:methodology_details}

In this section, we expand on the descriptions of individual components of $\flock$ summarized in \Cref{sec:methods}: the random walk algorithm, the recording protocol, the sequence processor, and the consensus protocol.

\subsection{Uniform random walk}
\label{sec:random_walk_details}

Let $G=(V,E,R)$ be a knowledge graph, and let $\ell$ be the length of random walks.
For each node $v\in V$, we will denote by $\mathcal{N}(v)$ the set of neighbors of $v$:
\[
\mathcal{N}(v) = \{w \in V : \exists r\in R . (v,r,w)\in E \lor (w,r,v) \in E\}
\]
and by $E{(v,w)}$, the set of relational edges from $v$ to $w$ (allowing for the inverse direction):
\small
\[\begin{aligned}
E{(v,w)} =& \{(v,r,w) \in R\times\{v\}\times\{w\} : (v,r,w)\in E\} \\&\cup \{(v,r^{-1},w) \in R^{-1}\times\{v\}\times\{w\} : (w,r,v)\in E\}
\end{aligned}
\]
\normalsize
where $R^{-1}$ is the set symbolizing the inverses of relation types in $R$.
The uniform random walk with no backtracking $\eta(G,\ell)$ of length $\ell$ over $G$, represented as:
\[
V_0 \xrightarrow{R_1} V_1 \xrightarrow{R_2} \cdots \xrightarrow{R_\ell} V_\ell
\]
is a second-order Markov process that follows the rules:
\begin{equation}
\label{eq:uniform_random_walk_markov}
\begin{aligned}
    \sP(V_{i+2} = v \mid V_{i+1}=w, V_{i}=u) &= \begin{cases}
        0 &\text{if } v=u \text{ and } |\mathcal{N}_w| > 1 \\
        1 &\text{if } v=u \text{ and } \mathcal{N}_w = \{u\} \\
        \frac{1}{|\mathcal{N}_w|-1} &\text{if } v\neq u \text{ and } v\in\mathcal{N}_w \\
        0 &\text{if } v \notin\mathcal{N}_w
    \end{cases}\\
    \sP(R_{j+1} = r \mid V_{j+1}=w, V_j = u) &= \begin{cases}
        \frac{1}{|E_{(w,u)}|} &\text{if } r(w,u) \in E_{(w,u)} \\
        0 &\text{otherwise}
    \end{cases}
\end{aligned}
\end{equation}

for all $i\geq 0, j\geq 1$. Intuitively, at each step of the walk, we first select a neighbor (except for the vertex chosen one step ago) of the current node uniformly at random (disregarding multi-edges and edge directions), and then sample an edge between these two nodes uniformly at random. If the current node has only one neighbor, we are forced to return to it.

The initial conditions depend on the selected scenario.
Given a query $q=(h,r,?)$ over $G$, we can describe the process of selecting the first step $V_0\xrightarrow{R_1}V_1$ as setting either (each with probability $\frac{1}{3}$):

\begin{itemize}
    \item $V_0=h$ and selecting the first step uniformly at random as described above, meaning:
    \[
    \begin{aligned}
            \sP(V_{1} = v \mid V_{0}=h) &= \begin{cases}
        \frac{1}{|\mathcal{N}_h|} &\text{if } v\in\mathcal{N}_h \\
        0 &\text{if } v \notin\mathcal{N}_h
    \end{cases}\\
    \sP(R_{1} = r \mid V_{1}=w) &= \begin{cases}
        \frac{1}{|E_{(w,h)}|} &\text{if } r(w,h) \in E_{(w,h)} \\
        0 &\text{otherwise}
        \end{cases}
    \end{aligned}
    \]
    \item setting $R_1=r$ and selecting $V_0\xrightarrow{R_1}V_1$ uniformly at random from edges with type $r$.
    \item choosing $V_0$ uniformly at random, and then sampling the first step at random as well:
    \[
        \begin{aligned}
        \sP(V_0 = w)&= \frac{1}{|V|}\\
            \sP(V_{1} = w \mid V_{0}=v) &= \begin{cases}
        \frac{1}{|\mathcal{N}_w|} &\text{if } v\in\mathcal{N}_w \\
        0 &\text{if } v \notin\mathcal{N}_w
    \end{cases}\\
    \sP(R_{1} = r \mid V_{1}=v, V_0=w) &= \begin{cases}
        \frac{1}{|E_{(w,v)}|} &\text{if } r(w,v) \in E_{(w,v)} \\
        0 &\text{otherwise}
        \end{cases}
    \end{aligned}
    \]
\end{itemize}

For the relation prediction objective, we add one more scenario, similar to the first one described above, but substituting $V_0=t$ instead.
For that problem, each scenario is chosen with probability $\frac{1}{4}$.

\subsection{Recording function}
\label{sec:recording_protocol_details}
Given a KG $G=(V,E,R)$, a query $q=(h_q,r_q,?)$, a walk $\bar\eta = v_0 \xrightarrow{r_1} v_1 \xrightarrow{r_2} \dots \xrightarrow{r_\ell} v_\ell$ of length $\ell$ over $G$, and a set of embeddings $\mathbf{v}$ of nodes $V$ and $\mathbf{r}$ of relations $R$, our recording function $w$ first splits the walk into a sequence of $\ell+1$ steps:
\[(r_0, v_0), (r_1, v_1), \dots (r_\ell, v_\ell)\]
with $r_0 = r_\varnothing$ being a special marker for no relation. Each step $(r_i, v_i)$ is transformed into a 7-tuple:
\[
S_i = \left(\operatorname{id}_V(v_i; \bar\eta), \operatorname{id}_R(r_i; \bar\eta), \operatorname{dir}_i, \delta_{v_i=h_q}, \delta_{r_i=r_q}, \mathbf{v}(v_i), \mathbf{r}(r_i)\right)
\]
where:
\begin{itemize}
    \item $\operatorname{id}_V(v_i; \bar\eta)$ and $\operatorname{id}_R(r_i; \bar\eta)$ are the anonymized id's of the node $v_i$ and relation $r_i$, evaluated as:
    \[
    \begin{aligned}
        \operatorname{id}_V(v_i; \bar\eta) &= \argmin_t [v_t = v_i]\\
        \operatorname{id}_R(r_i; \bar\eta) &= \argmin_t \left[r_t = r_i \lor r_t = r_i^{-1}\right]
    \end{aligned}
    \]
    \item $\operatorname{dir}_i$ denotes the direction in which we follow the edge. We set $\operatorname{dir}_i\!=\!0$ if $r_i\!\in\!R$ (the edge is traversed from head to tail) and $\operatorname{dir}_i = 1$ if $r_i\in R^{-1}$ (the edge is taken in the reverse direction).
    \item $\delta_{v_i=h_q}$ and $\delta_{r_i\sim r_q}$ are binary flags representing whether the current node $v_i$ is the query head $v_q$ and if the relation $r_i$ is either the queried relation $r_q$ or its inverse $r_q^{-1}$.
    \item $\mathbf{v}(v_i),  \mathbf{r}(r_i)$ are the embeddings of $v_i$ and $r_i$, respectively.
\end{itemize}
The output of $w$ for $\bar\eta$ given $G, q, \mathbf{v}, \mathbf{r}$ is then:
\[
w(\bar\eta ; G,q,\mathbf{v},\mathbf{r}) = (S_0, S_1, \dots, S_\ell)
\]

\subsection{Sequence processor}

Once the sampled walks are anonymized by the recording protocol $w$, the output for each walk $\bar\eta_i$:
\[
w(\bar\eta ; G,q,\mathbf{v},\mathbf{r}) = (S_0, S_1, \dots, S_\ell)
\]
is passed through the sequence processor $f_\theta$, parametrized by the following modules:
\begin{itemize}
    \item $\mathbf{A}_v, \mathbf{A}_r \in \mathbb{R}^{(\ell+1)\times d}$: embedding tables for anonymized vertices and relations, respectively,
    \item $\mathbf{D} \in \mathbb{R}^{2\times d}$: look-up table for the direction embedding,
    \item $\mathbf{Q}_h, \mathbf{Q}_r \in \mathbb{R}^{2\times d}$: embedding tables for the binary query labels,
    \item $\mathbf{V}, \mathbf{R} : \mathbb{R}^{d}\rightarrow \mathbb{R}^{d}$: linear maps applied to the passed embeddings of vertices and relations,
    \item $\mathbf{\Omega}$: a bi-directional GRU~\citep{cho2014properties} cell equipped with RMSNorm~\citep{zhang2019root} and SwiGLU~\citep{shazeer2020glu} activation function.
\end{itemize}
For each step, encoding $S_i$ of the form:
\[
S_i = \left(\operatorname{id}_V(v_i; \bar\eta_i), \operatorname{id}_R(r_i; \bar\eta_i), \operatorname{dir}_i, \delta_{v_i=h_q}, \delta_{r_i=r_q}, \mathbf{v}(v_i), \mathbf{r}(r_i)\right)
\]
we evaluate the processed embedding $\mathbf{c_i}$ of $S_i$ as a sum of the corresponding encoded components:
\[
\begin{aligned}
\mathbf{c}_i =& \mathbf{A}_v(\operatorname{id}_V(v_i;\bar\eta_i)) + \mathbf{A}_r(\operatorname{id}_R(r_i;\bar\eta_i)) + \mathbf{D}(\operatorname{dir}_i)\\
&+\mathbf{Q_h}(\delta_{v_i=h_q}) + \mathbf{Q}_r(\delta_{r_i=r_q}) + \mathbf{V}(\mathbf{v}(v_i)) + \mathbf{R}(\mathbf{r}(r_i))
\end{aligned}
\]

These are then passed to the GRU cell $\mathbf\Omega$, which fuses the features across the whole walk and produces multi-head embeddings of vertices and relations, as well as the associated weights:
\[
\left(\mathbf{s}^{(V)}_i, \mathbf{s}^{(R)}_i, \mathbf{a}^{(V)}_i, \mathbf{a}^{(R)}_i\right) 
= \mathbf{\Omega} ([\mathbf{c}_0, \mathbf{c}_1, \dots, \mathbf{c}_\ell])
\]
where $\mathbf{s}^{(V)}_i,\mathbf{s}^{(R)}_i \in \mathbb{R}^{(\ell+1) \times h \times d_h}$ and $\mathbf{a}^{(V)}_i, \mathbf{a}^{(R)}_i \in \mathbb{R}^{(\ell+1)\times h}$.
Stacking all $N$ of them gives us the final output of the sequence processor.

\subsection{Consensus protocol}
\label{sec:consensus_protocol_details}
Given walks $\bar\eta_{1:N}$ over $G=(V,E,R)$ and the outputs $\mathbf{s}^{(V)},\mathbf{s}^{(R)}, \mathbf{a}^{(V)},\mathbf{a}^{(R)}$ of the sequence processor, the consensus protocol $c$ aggregates the signal for each node by evaluating a weighted sum over the appearances of this node across the walks. More precisely, for each node $v\in V$, we find all pairs of indices $(i,j)$, such that the $j^\text{th}$ node visited in $\bar\eta_i$ was $v$, and concatenate the weighted sums of embeddings produced by each head, with weights exponentially proportional to the scores $\mathbf{a}^{(V)}$:
\[
\Delta \mathbf{v}(v) = \bigoplus_{k=1}^{h} \frac{\sum\limits_{\substack{i,j\\\bar\eta_i(v_j) = v}} \exp\left(\mathbf{a}^{(V)}_{i,j,k}\right)\cdot \mathbf{s}^{(V)}_{i,j,k}}{\sum\limits_{\substack{i,j\\\bar\eta_i(v_j) = v}} \exp\left(\mathbf{a}^{(V)}_{i,j,k}\right)}
\]
Similarly, we aggregate the encodings for relations, considering their occurrences in both directions:
\[
\Delta \mathbf{r}(r) = \bigoplus_{k=1}^{h} \frac{\sum\limits_{\substack{i,j\\\bar\eta_i(r_j) \in \{r, r^{-1}\}}} \exp\left(\mathbf{a}^{(R)}_{i,j,k}\right)\cdot \mathbf{s}^{(R)}_{i,j,k}}{\sum\limits_{\substack{i,j\\\bar\eta_i(r_j) \in \{r, r^{-1}\}}} \exp\left(\mathbf{a}^{(R)}_{i,j,k}\right)}
\]
In both formulas above, $\bigoplus$ denotes concatenation. 

Additionally, we say that a consensus protocol $c$ is \emph{invariant} if for any pair of isomorphic KGs $G=(V,E,R)$ and $H = (V',E',R')$, any isomorphism $\mu=(\pi,\phi)$  from $G$ to $H$, any list of embeddings $\mathbf{h}_{1:N}$ with $\mathbf{h}_i\in\mathbb{R}^d$, and any sequence of sampled walks $\bar{\eta}_{1:N}$ over $G$, the outputs
\[
\begin{aligned}
(\Delta \mathbf{v}, \Delta \mathbf{r}) &= c(\mathbf{h}_{1:N}, \bar{\eta}_{1:N}) \\  
(\Delta \mathbf{v}', \Delta \mathbf{r}') &= c(\mathbf{h}_{1:N}, \mu(\bar{\eta}_{1:N}))   
\end{aligned}
\]
satisfy:
\[
\begin{aligned}
    \Delta\mathbf{v}(v) &= \Delta\mathbf{v}'(\pi(v)) \quad &\forall v\in V\\
    \Delta\mathbf{r}(r) &= \Delta\mathbf{r}'(\phi(r)) \quad &\forall r\in R\\
\end{aligned}
\]

\newpage

\section{Proofs}
\label{app:proofs}

\subsection{Expressivity}

The main proposition of this section formalizes the fact that $\flock$ can approximate any link-invariant function over fixed-size knowledge graphs in probability.
Intuitively, when the length of the sampled walks $\ell$ becomes higher, the probability of a single walk witnessing all the edges grows to 1.
Once a walk visits all the edges, a sufficiently powerful sequence processor can derive the whole graph structure from its anonymized representation, recreating the graph in its entirety, up to isomorphism.
Then, the processor can return the value of the approximated function for that graph.

We start by showing that the edge cover time $C_E(\cdot)$ of graphs in $K_{n,m}$ is bounded:

\begin{lemma}
    \label{lem:cover_time_is_finite}
    Let $G\in\mathbb{K}_{n,m}$ for some $n,m$. The edge cover time $C_E(G)$ of $G$, using the algorithm from \Cref{sec:random_walk_details}, is finite.
\end{lemma}
\begin{proof}
Let $G=(V,E,R)\in \mathbb{K}_{n,m}$ be a graph. For any edge $e\in E$ and any vertex $v\in V$, let $H_v(e)$ denote the expected number of steps of the random walk algorithm $\eta$ described in \Cref{sec:random_walk_details}.
Then, the edge cover time $C_E(G)$ of $G$ with $\eta$, i.e. the expected number of steps that $\eta$ needs to take before visiting every edge in $G$, is bounded above by:
\[
C_E(G) \leq \sum\limits_{e\in E} \max_{v\in V} H_v(e) \leq m \cdot \max_{\substack{e\in E\\v\in V}} H_v(e)
\]
Indeed, consider the event of visiting all these edges in order $e_1, \dots, e_m$:
\[
\begin{aligned}
    C_E(G) &= \mathbb{E}[\text{\#steps to visit all }e_1, \dots, e_m]\\
    &\leq \mathbb{E}[\text{\#steps to visit }e_1,\text{ then }e_2, \dots, \text{ then } e_m] \\
    &\leq \mathbb{E}[\text{\#steps to visit }e_1] + \sum_{i=1}^{m-1} \mathbb{E}[\text{\#steps to visit }e_{i+1} \text{ starting from }h_i \text{ or } t_i] \\
    &\leq \max_{v\in V} H_v(e_1) + \sum_{i=1}^{m-1} \max(H_{h_i}(e_{i+1}), H_{t_i}(e_{i+1})) \\
    &\leq \max_{v\in V} H_v(e_1) + \sum_{i=1}^{m-1} \max_{v\in V} H_v(e_{i+1}) \\
    &= \sum_{i=1}^{m} \max_{v\in V} H_v(e_{i})
\end{aligned}
\]
where $h_i$ and $t_i$ are the head and tail of the edge $e_i$, respectively.
Therefore, to show that $C_E(G)$ is finite, it suffices to prove that $H_v(e)$ is bounded for all $v\in V, e\in E$.

Fix $v\in V$ and $e\in E$. Consider an infinite random walk generated with $\eta$ over $G$, starting at $v$:
\[
v = v_0 \xrightarrow{r_1} v_1 \xrightarrow{r_2} v_2 \xrightarrow{r_3} \dots
\]
We want to bound the expected first index $t$, such that $e$ is the edge traversed in step $v_{t-1}\xrightarrow{r_t} v_t$.
Denote by $\Delta$ a maximum degree of a vertex in $G$ (counted as the number of connected vertices $\mathcal{N}(v)$), by $\rho$ the maximum number of edges between any single pair of nodes and by $d$ -- the diameter of the graph, i.e. the length of the longest shortest path between two vertices (in the undirected version of $G$).
Consider the series of events $A_0, A_1, \dots$ where $A_i$ is characterized as:
\[
\begin{aligned}
A_i := &\text{ the event that starting from }v_{i(d+2)}\text{ the walk will follow a shortest path}\\
&\text{ to one of the endpoints of }e\text{ and then go through }e
\end{aligned}
\]
Let $e=(h_e,r_e,t_e)$.
For all values of $i$, by definition, the length of the shortest path from $v_{i(d+2)}$ to $h_e$ or $t_e$ is at most $d$.
Therefore, the whole part of the walk described in $A_i$ is at most $d+1$ steps long.
By the definition of the used random walk algorithm, which only looks at the previously taken edge, we can deduce that the events $A_i$ are all mutually independent.

Moreover, let $v_{i(d+2)} = u_0 \xrightarrow{s_1} u_1 \xrightarrow{s_2} \dots \xrightarrow{s_\ell}u_\ell\in\{h_e, t_e\}$ be a shortest path from $v_{i(d+2)}$ to one of $h_e, t_e$.
Note that by minimality, there cannot be any backtracking while following this path.
Therefore, the probability of the next visited node is dependent only on the value of the previous one, and we can bound the probability $P(A_i)$ of $A_i$ from below by:
\[
\sP(A_i) \geq \sP(\text{pass through }e\text{ after reaching } h_e \text{ or } t_e) \cdot\prod_{j=0}^{\ell-1} \sP(v_{i(d+2)+j+1} = u_{j+1} \mid v_{i(d+2)+j} = u_j) 
\]
The first term on the right hand side is the probability of selecting $e$ while being at $h_e$ or $t_e$, which is the probability of first selecting the other endpoint (out of at most $\Delta$ neighbors) and then picking $e$ over other edges between $h_e$ and $t_e$ (of which there is at most $\rho$). Hence:
\[
\sP(\text{pass through }e\text{ after reaching } h_e \text{ or } t_e) \geq \frac{1}{\Delta} \cdot \frac{1}{\rho} = \frac{1}{\Delta \cdot \rho}
\]
As we never reach a backtracking situation by minimality of the shortest path, we can also write:
\[
\sP(v_{i(d+2)+j+1} = u_{j+1} \mid v_{i(d+2)+j} = u_j) = \frac{1}{|\mathcal{N}(v_{i(d+2)+j})|} \geq \frac{1}{\Delta}
\]
Combining these observations, we can derive a bound for $\sP(A_i)$ in terms of $\Delta, \rho$ and $d$:
\[
\begin{aligned}
\sP(A_i) &\geq \sP(\text{pass through }e\text{ after reaching } h_e \text{ or } t_e) \cdot\prod_{j=0}^{\ell-1} \sP(v_{i(d+2)+j+1} = u_{j+1} \mid v_{i(d+2)+j} = u_j)\\
&\geq \frac{1}{\Delta\cdot\rho} \cdot \prod_{j=0}^{\ell-1} \frac{1}{\Delta}
\quad\geq\quad \frac{1}{\Delta\cdot\rho} \left(\frac{1}{\Delta}\right)^\ell
\quad\geq\quad \frac{1}{\Delta\cdot\rho} \left(\frac{1}{\Delta}\right)^d
\quad=\quad \frac{1}{\rho \Delta^{d+1}}
\end{aligned}
\]
Finally, note that if $A_i$ is true, then the first index $t$ such that $v_{t-1}\xrightarrow{r_t} v_t$ traverses $e$ is at most $(i+1)(d+2)$.
We can therefore bound the expectation of such $t$, being $H_v(e)=H_{v_0}(e)$ by:
\small
\[
\begin{aligned}
H_v(e) &\leq (d+2)\cdot \sP(A_0) + 2(d+2)\cdot \sP(\neg A_0 \land A_1) + 3(d+2)\cdot \sP(\neg A_0 \land \neg A_1 \land A_2) + \dots\\
&= (d+2)\cdot \sP(A_0) + 2(d+2)\cdot \sP(\neg A_0)\cdot P(A_1) + 3(d+2)\cdot \sP(\neg A_0)\cdot \sP(\neg A_1)\cdot \sP( \land A_2) + \dots\\
&= (d+2) + \sP(\neg A_0) \cdot (d+2 + \sP(\neg A_1) \cdot (d+2 + \sP(\neg A_2)\cdot (\dots))) \\
&\leq (d+2) + \left(1-\frac{1}{\rho \Delta^{d+1}}\right) \cdot \left(d+2 + \left(1-\frac{1}{\rho \Delta^{d+1}}\right) \cdot \left(d+2 + \left(1-\frac{1}{\rho \Delta^{d+1}}\right)\cdot (\dots)\right)\right)\\
&=  \rho (d+2)\Delta^{d+1}
\end{aligned}
\]
\normalsize
Since $\rho \leq m, d+2\leq n$ and $\Delta \leq n$, we have $H_v(e) \leq m(n+2)n^n$, which completes the proof.
\end{proof}

\begin{remark}
The bound obtained in the proof of Lemma \ref{lem:cover_time_is_finite} is very crude.
In fact, we could transform the given knowledge graph into a simple graph (undirected, with no multi-edges) by substituting each edge $u\xrightarrow{r}v$ with two undirected edges $u\leftrightarrow v_{(u,r,v)} \leftrightarrow v$.
The augmented graph will then have $n+m$ vertices, and our random walk algorithm naturally translates to a weighted random walk on the transformed graph.
This hints at an assumption that in practice, the edge cover time of the used random walk algorithm is of the magnitude $O((n+m)^3)=O(n^3+m^3)$.
\end{remark}

Let us now prove a fact about the number of distinct, up to isomorphism, graphs in $\mathbb{K}_{n,m}$.

\begin{lemma}
    \label{lem:finite_isomorphism_classes}
    For any $n,m$, the number of isomorphism classes in $\mathbb{K}_{n,m}$ is finite.
\end{lemma}
\begin{proof}
    Since the number of distinct relation types that a graph in $\mathbb{K}_{n,m}$ is at most $m$, it suffices to show that the number of isomorphism classes of graphs in $\mathbb{K}_{n,m}$ with exactly $k$ relation types is bounded, for all $k\in\{1,2,\dots,m\}$.

    Fix the number $k\in\{1,2,\dots,m\}$ and consider $G=(V,E,R)\in \mathbb{K}_{n,m}$ with $|R|=k$.
    We will show that, up to isomorphism, there are finitely many such choices of $G$.
    Firstly, as renaming does not change the graph structure, without loss of generality, we can assume that:
    \[
    V = \{v_1, v_2, \dots, v_n\} \qquad \text{and} \qquad R = \{r_1, r_2, \dots, r_k\}
    \]
    Then, there are exactly $n^2k$ possible relational edges $e\in(V\times R\times V)$, and $E\subseteq V\times R\times V$ is a set of $m$ elements.
    Hence, there are \(n^2k\choose{m}\) possible choices of $E$, and hence, at most \(n^2k\choose{m}\) non-isomorphic choices of $G$. Since $k$ was chosen arbitrarily, this completes the proof.
\end{proof}

\begin{lemma}
    \label{lem:bounded_cover_time_in_knm}
    For each pair $(n,m)$, there exists a number $C_{n,m}$ such that the edge cover time, using the algorithm from \Cref{sec:random_walk_details}, of any knowledge graph in $\mathbb{K}_{n,m}$ is at most $C_{n,m}$.
\end{lemma}
\begin{proof}
    The result follows from Lemmas \ref{lem:cover_time_is_finite} and \ref{lem:finite_isomorphism_classes}.
    As two isomorphic graphs have identical cover time, we can set $C_{n,m}$ to be the maximum of cover times of representatives of all isomorphic classes, which, by finiteness of both, is well-defined.
\end{proof}

\begin{lemma}
\label{lem:walk_to_iso_graph}
    Let $G\in \mathbb{K}_{n,m}$ be a graph, $q=(h_q,r_q,?)$ be a link query over $G$, and $\bar\eta$ be a walk over~$G$. If $\bar\eta$ traverses all edges of $G$, then using only the output $w(\bar\eta; G, q, \cdot, \cdot)$ of the recording~function $w$ detailed in \Cref{sec:recording_protocol_details}, we can construct a graph-query pair $(H,q')$ isomorphic to $(G,q)$.
\end{lemma}

\begin{proof}
    Suppose that $\bar\eta = v_0 \xrightarrow{r_1} v_1 \xrightarrow{r_2}\dots \xrightarrow{r_\ell} v_\ell$ visits all edges of $G=(V,E,R)$ and let $\ell$ be its length.
    Recall the anonymization functions $\operatorname{id}_V(\cdot;\bar\eta)$ and $\operatorname{id}_R(\cdot;\bar\eta)$ as defined in \Cref{sec:recording_protocol_details}.
    The~output $w(\bar\eta; G, q, \cdot, \cdot)$ (the embedding functions provided as the last two arguments are irrelevant) is a sequence of tuples $S_0, S_1, \dots, S_\ell$ with each $S_i$ equal to:
    \[
    S_i = \left(\operatorname{id}_V(v_i; \bar\eta), \operatorname{id}_R(r_i; \bar\eta), \operatorname{dir}_i, \delta_{v_i=h_q}, \delta_{r_i=r_q}, \cdot, \cdot\right)
    \]
    Consider a graph $H=(V', E', R')$ constructed as follows:
    \begin{itemize}
        \item the vertices $V'$ correspond to the anonymized node ids $\operatorname{id}_V(v_i; \bar\eta)$:
        \[
        V' = \{\operatorname{id}_V(v;\bar\eta) \mid v\in V\}
        \]
        Since each vertex must have been visited by $\bar\eta$, this is well-defined.
        \item the relation types $R'$ are the anonymized relation ids $\operatorname{id}_R(r_i; \bar\eta)$:
        \[
        R' = \{\operatorname{id}_V(r;\bar\eta) \mid r\in R\}
        \]
        Again, this is well-defined, as each relation must have been noticed by $\bar\eta$.
        \item the edges $E'$ are reconstructed from the consecutive step encodings using the anonymized vertex and relation indices and the direction $\operatorname{dir}_i$:
        \[
        \begin{aligned}
        E' =& \left\{ (\operatorname{id}_V(v_{i-1}), \operatorname{id}_R(r_i), \operatorname{id}_V(v_i)) \mid \operatorname{dir}_i = 0, 1\leq i \leq l\right\}\\ 
            &\:\cup \left\{ (\operatorname{id}_V(v_{i}), \operatorname{id}_R(r_i), \operatorname{id}_V(v_{i-1})) \mid \operatorname{dir}_i = 1, 1\leq i \leq l\right\}
        \end{aligned}
        \]
    \end{itemize}
    and a query $q' = (\operatorname{id}_V(v_i; \bar\eta), \operatorname{id}_R(r_j, \bar\eta), ?)$ for $i,j$ such that $\delta_{v_i=h_q}=1$ and $\delta_{r_j=r_q}=1$.

    Then by the definition of $w$ (\Cref{sec:recording_protocol_details}), it is straightforward to check that the pair $(\operatorname{id}_V(\cdot;\bar\eta), \operatorname{id}_R(\cdot;\bar\eta))$ defines an isomorphism from $(G,q)$ to $(H, q')$.
    Indeed, both these functions are injective by construction, and as $\bar\eta$ witnesses all nodes and relations, they are well-defined bijections.
    For each unique edge traversed by $\bar\eta$, there exists a unique edge in $E'$ translated to the anonymized space, which implies an isomorphism between $E$ and $E'$.
    Finally, by utilizing the flags $\delta_{v_i=h_q}$ and $\delta_{r_j=r_q}$, we can identify the query head node and relation in the new graph.
    All things considered, we can reconstruct the pair $(G,q)$, up to isomorphism, from the output of $w(\bar\eta;G,q,\cdot,\cdot)$.
\end{proof}

We are now ready to prove the main result regarding the universality of $\flock$ as an approximation of link invariant functions.
The outline of the proof is as follows:
1) Using the upper-bound on the edge cover time of graphs in $\mathbb{K}_{n,m}$ derived in Lemma \ref{lem:bounded_cover_time_in_knm}, we can bound the probability of sampling a walk that visits all edges,
2) Once such a walk is sampled, we can recover the graph and query, up to isomorphism, from its anonymized form (Lemma \ref{lem:walk_to_iso_graph}),
3) Lastly, we can return the value of the approximated function for the derived isomorphic instance.
Since the approximated function is link invariant, if the reconstructed graph matches the original one, we return precisely the correct value.

\begin{propositioncopy}{\ref{prop:universal_approximator}}
    With a powerful enough sequence processor $f_\theta$, the $\flock$ framework described in \Cref{sec:methods} is a universal approximator of link invariant functions over $\mathbb{K}_{n,m}$ for all pairs $(n,m)$.
\end{propositioncopy}

\begin{proof}
Let $\varphi:\mathbb{K}_{n,m} \rightarrow (V\times R\times V \rightarrow [0,1])$ be a link invariant function over $\mathbb{K}_{n,m}$ returning values from the interval $[0,1]$.
Let $G=(V,E,R)\in\mathbb{K}_{n,m}$, $q=(h,r,?)$ be a link prediction query over $G$ and $t\in V$ be a target node.
Pick some $\epsilon, \delta > 0$. Our goal is to show that:
\begin{equation}
    \label{eq:universality_result}
\sP(|\varphi(G)((h,r,t)) - X_\theta(G, (h,r,?))(t)| < \epsilon) > 1-\delta
\end{equation}

For simplicity, let us consider a situation where only a single walk $\bar\eta$ of length $\ell$ is sampled by the~model (otherwise, omit additional walks).
We will also restrict the argument to a single refinement case -- the result can be extended to multiple refinement steps by returning $\Delta\mathbf{v},\Delta\mathbf{r}=0$ during all additional iterations.
Consider a sequence processor $f_\theta$ that given the output $w(\bar\eta;G,q,\cdot,\cdot)$ of the recording protocol, creates a graph-query pair $(H, q')$ with $q'=(h_{q'}, r_{q'}, ?)$ using the~strategy described in the proof of Lemma \ref{lem:walk_to_iso_graph}, and returns a vector 
$\mathbf{h}\in \mathbb{R}^{l+1}$ whose $i^\text{th}$ entry is equal $\mathbf{h}_i = \varphi(H)((h_{q'}, r_{q'}, \operatorname{id}_V(v_i;\bar\eta))$ where $v_i$ is the $i^\text{th}$ node visited by $\bar\eta$.
The consensus protocol $c$, provided $t$ was visited by $\bar\eta$, can then identify $t$ as one of the $v_j$ and pull the corresponding embedding
$\mathbf{h}_j = \varphi(H)((h_{q'}, r_{q'}, \operatorname{id}_V(t;\bar\eta))$, returning it as the output $\mathbf{v}(t) = \mathbf{h}_j$ (note that no matter which specific value of $j$ is chosen, this value will be the same).
Finally, the classification head can work as an identity operation, returning $X_\theta(G,q)(t) = \mathbf{v}(t) =\varphi(H)((h_{q'}, r_{q'}, \operatorname{id}_V(t;\bar\eta))$.

We claim that if the sampled walk $\bar\eta$ traverses all edges of $G$, then the output of the $\flock$ model described above satisfies:
\[
\varphi(G)((h,r,t)) = X_\theta(G, (h,r,?))(t)
\]
By Lemma \ref{lem:walk_to_iso_graph}, in such case, the reconstructed pair $(H,q')$ is isomorphic to $(G,q)$ by the isomorphism $id = (\operatorname{id}_V(\cdot;\bar\eta), \operatorname{id}_R(\cdot; \bar\eta))$. 
Since $\varphi$ is link invariant, we can write:
\[
\begin{aligned}
\varphi(G)((h, r, t)) &= \varphi(\operatorname{id}(G))((\operatorname{id}_V(h;\bar\eta), \operatorname{id}_R(r;\bar\eta), \operatorname{id}_V(t;\bar\eta)))\\
&= \varphi(H)((h_{q'}, r_{q'}, \operatorname{id}_V(t;\bar\eta)))\\
&= X_\theta(G,(h,r,?))(t)
\end{aligned}
\]

Therefore, whenever the walk $\bar\eta$ witnesses all edges of $G$, the output of the $\flock$ model satisfies:
\[
\varphi(G)((h,r,t)) = X_\theta(G, (h,r,?))(t)
\]

Hence, to show (\ref{eq:universality_result}), it suffices to prove that we can uniformly choose the length $\ell$ of the random walk so that the probability of $\bar\eta$ covering all the edges is greater than $1-\delta$.
By Markov's inequality:
\[
\begin{aligned}
\sP(\bar\eta \text{ does not cover all edges}) &= \sP(\text{it takes }>\ell \text{ steps for }\eta \text{ to cover edges of }G) \\
& \leq \frac{\mathbb{E}[\text{\#steps such that }\eta\text{ covers all edges of } G]}{\ell} \\
& = \frac{C_E(G)}{\ell}
\end{aligned}
\]
But by Lemma \ref{lem:bounded_cover_time_in_knm}, $C_E(G) \leq C_{n,m}$ for some constant $C_{n,m}$.
Hence, taking $\ell > \frac{C_{n,m}}{\delta}$, we get:
\[
\begin{aligned}
\sP(\bar\eta \text{ does not cover all edges}) \leq \frac{C_E(G)}{\ell} \leq \frac{C_{n,m}}{\ell} < \delta 
\end{aligned}
\]
This means that for such a choice of $\ell$:
\[
\sP(\bar\eta \text{ witnesses all edges of }G) > 1-\delta
\]
which leads to the conclusion that for $\ell > \frac{C_{n,m}}{\delta}$, the proposed $\flock$ framework satisfies:
\[
\sP(|\varphi(G)((h,r,t)) - X_\theta(G, (h,r,?))(t)| < \epsilon) > 1-\delta
\]
for any choice of $G=(V,E,R)\in\mathbb{K}_{n,m}$ and $(h,r,t)\in V\times R\times V$.
\end{proof}

\subsection{Invariance}

First, let us recall the definition of invariance for the context of knowledge graphs and the associated notion of invariance in probability.
We say that a function $\varphi$ taking KGs as input is invariant if for any pair of isomorphic KGs $G\simeq H$ it produces the same input, i.e. \(G\simeq H \implies \varphi(G) = \varphi(H)\).

We extend the notion of invariance for further types of inputs, not limited to full knowledge graphs, particularly to random walks and link prediction queries.
Let $G = (V,E,R) \in \mathbb{K}_{n,m}$ and let $H = (V',E',R')\simeq G$ be a KG isomorphic to $G$ via the isomorphism $\mu = (\pi, \phi)$. For any $h\in V$ and $r\in R$, we identify the link prediction query $q = (h, r, ?)$ in $H$ using the isomorphism $\mu$ as:
\[
\mu(q) = \mu((h, r, ?)) = (\pi(h), \phi(r), ?)
\]
Similarly, let $\eta = v_0 \xrightarrow{r_1} \dots \xrightarrow{r_\ell} v_\ell $ be a walk of length $\ell$ in $G$.
The view of $\eta$ with $\mu$ is defined as:
\[
\mu\left(
v_0 \xrightarrow{r_1} v_1 \xrightarrow{r_2} \dots \xrightarrow{r_\ell} v_\ell
\right)
= \pi(v_0) \xrightarrow{\phi(r_1)} \pi(v_1) \xrightarrow{\phi(r_2)} \dots \xrightarrow{\phi(r_\ell)} \pi(v_\ell) 
\]

Let $f$ be a function taking inputs drawn from KGs. We call $f$ invariant if for any pair of isomorphic graphs $G\musim H$ and an associated isomorphism $\mu = (\pi,\phi)$, $f$ satisfies
\[
f(x) = f(\mu(x))
\]
where $x$ can be, e.g., a walk or link prediction query.
In words,  invariance means that the function preserves output under the re-identifications of the input graph and the induced transformations of queries and walks.

This notion extends to functions with multiple inputs, where we enforce the transformation on each graph-related input.
For example, a function $\varphi$ taking a KG, query and a $d$-dimensional vector is invariant if it satisfies:
\[
\forall G\musim H, q, \mathbf{v}\in \mathbb{R}^d:\quad \varphi(G,q,\mathbf{v}) = \varphi(\mu(G), \mu(q), \mathbf{v}) 
\]

Following the definition of \textit{invariance in probability}, provided in \Cref{sec:preliminaries}, we extend all the definitions above to the stochastic case, replacing equality ($=$) with equality in distribution ($\probeq$).

We can now prove the main propositions stated in \Cref{sec:theoretical_properties}.
Let's begin with the more general:

\begin{propositioncopy}{\ref{prop:general_invarinace}}
    Suppose that the walk sampling protocol $\eta$ is invariant in probability and both the recording protocol $w$ and the consensus protocol $c$ are invariant.
    Then, regardless of the choice of the deterministic sequence processor $f_
    \theta$, the corresponding $\flock$ model is invariant in probability.
\end{propositioncopy}
\begin{proof}
    Let $(V,E,R) = G\simeq H = (V',E',R')$ be isomorphic knowledge graphs with isomorphism $\mu = (\pi, \phi)$ transforming $G$ into $H$.
    Our goal is to show that when the statement conditions are met for a $\flock$ model $X_\theta$ with $I$ refinement steps, then for any link prediction query $q = (h,r,?)$ and any target node $t\in V$, the prediction of $\flock$ for $t$ over $(G,q)$ is an identical random variable to the prediction for $\pi(t)$ over $(H,\mu(q))$, i.e.
    \[
    X_\theta(G,q)(t) \probeq X_\theta(H,\mu(q))(\pi(t))
    \]
    where $\mu(q) = (\pi(h),\phi(r), ?)$. Recall that these predictions are defined as:
    \begin{equation*}
        \begin{aligned}
            X_\theta(G,q)(t) &\coloneqq {\rm head}({\bf v}^{(I)}(t) + {\bf r}^{(I)}(r)) \\
            X_\theta(H,\mu(q))(\pi(t)) &\coloneqq {\rm head}({\bf v'}^{(I)}(\pi(t)) + {\bf r'}^{(I)}(\phi(r))) \\
        \end{aligned}
    \end{equation*}
    As ${\rm head}$ is a deterministic map, it suffices to show that the final embeddings $\mathbf{v}^{(I)}, \mathbf{r}^{(I)}$ for $(G,q)$ and $\mathbf{v'}^{(I)}, \mathbf{r'}^{(I)}$ for $(H,\mu(q))$ satisfy:
    \[
    \:\mathbf{v}^{(I)}(v) \probeq \mathbf{v'}^{(I)}(\pi(v)) \quad\text{and}\quad 
     \mathbf{r}^{(I)}(r) \probeq \mathbf{r'}^{(I)}(\phi(r))
    \qquad \forall v\in V, r\in R
    \]
    We will prove this result by induction on the number of layers $i$. The base case $i=0$ is trivial, as we initialize the embeddings of all nodes with a pretrained vector $\mathbf{v}_0$, and all relations with $\mathbf{r}_0$.

    For the induction step, suppose the claim holds for $i$.
    We drop the superscript $(i)$ for readability.
    The result for $i+1$ becomes apparent by unfolding the definitions of invariance of the considered components.
    Since $\eta$ is invariant in probability, we have
    \begin{equation}
    \label{eq:walk_invariance}
    \mu(\eta(G)) \probeq \eta(H)
    \end{equation}
    Let $\eta_1, \dots, \eta_n$ be the random walks over $G$ using $\eta$ and $\eta_1', \dots, \eta_n'$ be random walks over $H$.
    Now, $\eta_1, \dots, \eta_n$ are independent and identically distributed random variables, each following the distribution $\eta_j \sim \eta(G)$. Similarly, using (\ref{eq:walk_invariance}):
    \begin{equation}
    \label{eq:walks_in_H}
    \eta_j' \sim \eta(H) \probeq \mu(\eta(G))
    \implies \eta_j'\probeq \mu(\eta_j)
    \end{equation}

    As the recording protocol $w$ is invariant,
    \(
    w(\eta_j)\! \!= \!w(\mu(\eta_j))
    \)
    for all $j$,
    which with (\ref{eq:walks_in_H}) yields:
    \begin{equation}
        \label{eq:after_recording_inv}
        \mathbf{z}_j :=w(\eta_j) = w(\mu(\eta_j)) \probeq w(\eta_j') := \mathbf{z'_j}
    \end{equation}

    Then, $f_\theta$ is a deterministic map, so (\ref{eq:after_recording_inv}) implies:
    \begin{equation*}
        \label{eq:after_processor_inv}
        \mathbf{h}_j := f_\theta(\mathbf{z}_j) \probeq f_\theta(\mathbf{z}'_j) := \mathbf{h}_j'
    \end{equation*}

Let $(\Delta \mathbf{v}, \Delta\mathbf{r}) = c(\mathbf{h}_{1:N}, \eta_{1:N})$, $(\Delta \mathbf{v}', \Delta\mathbf{r}') = c(\mathbf{h}'_{1:N}, \eta'_{1:N})$ be the outputs of the consensus protocol.
We will denote by $c_\mathbf{v}$ and $ c_\mathbf{r}$, the restrictions to the first and second output, e.g. $\Delta \mathbf{v} = c_\mathbf{v}(\mathbf{h}_{1:N}, \eta_{1:N})$.
Let $\mathbf{x}\in\mathbb{R}^d$ be a vector, and denote by $\mathcal{W}(G)$ the space of walks over $G$. For any vertex $v\in V$, the probability that $\Delta\mathbf{v}(v) = \mathbf{x}$ equals:

\begin{equation*}
\begin{aligned}
    \sP(\Delta\mathbf{v}(v) = \mathbf{x}) &= \sum_{\bar{\eta}\in \mathcal{W}(G)^n} \sP(\Delta\mathbf{v}(v) =\mathbf{x} \vert \eta_{1:N} =\bar{\eta})\cdot \sP(\eta_{1:N} = \bar{\eta})\\
    &= \sum_{\bar{\eta}\in \mathcal{W}(G)^n} \sP(c_\mathbf{v}(\mathbf{h}_{1:N}, \eta_{1:N}) =\mathbf{x} \vert \eta_{1:N} =\bar{\eta})\cdot \sP(\eta_{1:N} = \bar{\eta})\\
    &= \sum_{\bar{\eta}\in \mathcal{W}(G)^n} \sP(c_\mathbf{v}(f_\theta(w(\eta_{1:N})), \eta_{1:N}) =\mathbf{x} \vert \eta_{1:N} =\bar{\eta})\cdot \sP(\eta_{1:N} = \bar{\eta})\\
    &= \sum_{\bar{\eta}\in \mathcal{W}(G)^n} \sP(c_\mathbf{v}(f_\theta(w(\bar\eta)), \bar\eta) =\mathbf{x} \vert \eta_{1:N} =\bar{\eta})\cdot \sP(\eta_{1:N} = \bar{\eta})\\
    &= \sum_{\substack{\bar{\eta}\in \mathcal{W}(G)^n\\c_\mathbf{v}(f_\theta(w(\bar\eta)), \bar\eta)(v) =\mathbf{x}}} \sP(\eta_{1:N} = \bar{\eta})
\end{aligned}
\end{equation*}

Similarly, we can derive:
\begin{equation*}
    \sP(\Delta\mathbf{v}'(\pi(v)) = \mathbf{x}) = \sum_{\substack{\bar{\eta}'\in \mathcal{W}(H)^n\\c_\mathbf{v}(f_\theta(w(\bar\eta')), \bar\eta')(\pi(v)) =\mathbf{x}}} \sP(\eta'_{1:N} = \bar{\eta}')
\end{equation*}
Using the invariance of the consensus protocol and the invariance of $f_\theta \circ w$, we can write:
\[
\begin{aligned}
c_\mathbf{v}(f_\theta(w(\bar\eta')), \bar\eta')(\pi(v))
&= c_\mathbf{v}(f_\theta(w(\mu(\bar\eta))), \mu(\bar\eta))(\pi(v)) \\
&= c_\mathbf{v}(f_\theta(w(\bar\eta)), \mu(\bar\eta))(\pi(v)) \\
&= c_\mathbf{v}(f_\theta(w(\bar\eta)), \bar\eta)(v) 
\end{aligned}
\]
The graph isomorphism $\mu$ defines a bijection between walks $\mathcal{W}(G)$ in $G$ and walks $\mathcal{W}(H)$ in $H$, so we can use this correspondence to deduce:
\begin{equation}
    \label{eq:cond_sum_on_h2}
    \begin{aligned}
    \sP(\Delta\mathbf{v}'(\pi(v)) = \mathbf{x})
    \;\;\;=& \sum_{\substack{\bar{\eta}'\in \mathcal{W}(H)^n\\c_\mathbf{v}(f_\theta(w(\bar\eta')), \bar\eta')(\pi(v)) =\mathbf{x}}} \!\!\!\!\!\sP(\eta'_{1:N} = \bar{\eta}') \\
    \;\;\;=& \sum_{\substack{\mu(\bar{\eta})\in \mathcal{W}(H)^n\\c_\mathbf{v}(f_\theta(w(\mu(\bar\eta))), \mu(\bar\eta))(\pi(v)) =\mathbf{x}}} \!\!\!\!\!\sP(\eta'_{1:N} = \mu(\bar{\eta})) \\
    \;\;\;=& \sum_{\substack{\bar{\eta}\in \mathcal{W}(G)^n\\c_\mathbf{v}(f_\theta(w(\bar\eta)), \bar\eta)(v) =\mathbf{x}}} \!\!\!\!\!\sP(\eta'_{1:N} = \mu(\bar{\eta}))
    \end{aligned}
\end{equation}
Since $\eta$ is invariant in probability, $\sP(\eta_{1:N} \!=\! \bar\eta) = \sP(\eta'_{1:N} \!=\! \mu(\bar{\eta}))$. Applying this to (\ref{eq:cond_sum_on_h2}) yields:
\begin{align*}
    \sP(\Delta\mathbf{v}'(\pi(v)) = \mathbf{x}) &= \sum_{\substack{\bar{\eta}\in \mathcal{W}(G)^n\\c_\mathbf{v}(f_\theta(w(\bar\eta)), \bar\eta)(v) =\mathbf{x}}} \!\!\!\!\!\sP(\eta'_{1:N} = \mu(\bar{\eta}))
    \\&=\!\!\!\! \sum_{\substack{\bar{\eta}\in \mathcal{W}(G)^n\\c_\mathbf{v}(f_\theta(w(\bar\eta)), \bar\eta)(v) =\mathbf{x}}} \!\!\!\!\!\!\!\sP(\eta_{1:N} = \bar\eta) = \sP(\Delta\mathbf{v}(v)= \mathbf{x})
\end{align*}
As $\mathbf{x}$ was chosen arbitrarily, we can conclude that $\Delta\mathbf{v}(v) \probeq \Delta\mathbf{v}'(\pi(v))$. 
The proof for relations follows analogously, considering $c_\mathbf{r}$ instead of $c_\mathbf{v}$. This allows us to write:
    \begin{equation}
        \label{eq:embedding_update_eq}
        \begin{aligned}
            \Delta\mathbf{v}(v) &\probeq \Delta\mathbf{v}'(\pi(v)) &\quad\forall v\in V\\
            \Delta\mathbf{r}(r) &\probeq \Delta\mathbf{r}'(\phi(r)) &\quad\forall r\in R
        \end{aligned}
    \end{equation}
    By the induction hypothesis, $\mathbf{v}^{(i)}(v)\probeq \mathbf{v'}^{(i)}(\pi(v))$ for all $v\in V$ and $\mathbf{r}^{(i)}(r)\probeq \mathbf{r'}^{(i)}(r)$ for all $r\in R$. Therefore, by (\ref{eq:embedding_update_eq}), combined with properties of sums of random variables:
    \[
    \begin{aligned}
    \mathbf{v}^{(i+1)}(v) := \mathbf{v}^{(i)}(v)+\Delta \mathbf{v}(v) &\probeq \mathbf{v'}^{(i)}(\pi(v))+\Delta \mathbf{v}'(\pi(v)) := \mathbf{v'}^{(i+1)}(\pi(v)) &\quad\forall v\in V\\
     \mathbf{r}^{(i+1)}(r) := \mathbf{r}^{(i)}(r)+\Delta \mathbf{r}(r) &\probeq \mathbf{r'}^{(i)}(\phi(r))+\Delta \mathbf{r}'(\phi(r)) := \mathbf{r'}^{(i+1)}(\phi(r)) &\quad\forall r\in R
    \end{aligned}
    \]
    which completes the induction step, and hence the proof.
\end{proof}

We can use the conclusion from Proposition \ref{prop:general_invarinace} to prove the probabilistic invariance of the architecture proposed in \Cref{sec:methods}.
To be able to apply it, we first need to verify the invariance of all used components, which we formalize in the following lemmas.

\begin{lemma}
    \label{lem:first_walk_step_invariance}
    The choice of the first step $v_0 \xrightarrow{r_1}v_1$ of the uniform random walk algorithm described in \Cref{sec:random_walk_details} is invariant.
\end{lemma}
\begin{proof}
    Let $G=(V,E,R)$ be a graph and let $H\simeq G$ be an isomorphic graph, with the isomorphism $\mu=(\pi, \phi)$ taking $G$ to $H$.
    Consider a link prediction query $q=(h,r,?)$ over $G$, and its identification $q'=\mu(q)=(\pi(h), \phi(r),?)$.
    The goal is to show that when using $\eta$ described in \Cref{sec:random_walk_details} for $(G,q)$ and $(H,q')$, the first steps:
    \[
    V_0 \xrightarrow{R_1} V_1 \qquad \text{and} \qquad U_0 \xrightarrow{S_1} U_1
    \]
    of the execution of $\eta$ over $G$ and $H$, respectively, satisfy the following property:
    \[
    \pi(V_0) \xrightarrow{\phi(R_1)} \pi(V_1) \probeq U_0 \xrightarrow{S_1} U_1
    \]
    By definition of $\eta$, there are three scenarios of choosing the first step, each with probability $\frac{1}{3}$.
    Hence, it suffices to show that within each scenario, the selection process is invariant in probability:
    \begin{itemize}
        \item \textbf{Scenario 1:} selecting the query head as the first node, then proceeding by random. First, $\pi$ takes the head node of $q$ to the head node of $q'$.
        Secondly, as isomorphisms preserve the number of neighboring nodes and number of edges between a pair of nodes, we have:
        \small
        \[
            \begin{aligned}
                    \sP(V_{1} = v \mid V_{0}=h) &= \begin{cases}
                \frac{1}{|\mathcal{N}_h|} &\text{\!\!if } v\in\mathcal{N}_h \\
                0 &\text{\!\!if } v \notin\mathcal{N}_h
            \end{cases} \\&= \begin{cases}
                \frac{1}{|\mathcal{N}_{\pi(h)}|} &\text{\!\!if } \pi(v)\in\mathcal{N}_{\pi(h)} \\
                0 &\text{\!\!if } \pi(v) \notin\mathcal{N}_{\pi(h)}
            \end{cases} &=  \sP(U_{1} = \pi(v) \mid U_{0}=\pi(h))
            \end{aligned}
            \]
        and
            \[
            \begin{aligned}
            \sP(R_{1} = r \mid V_{1}=w) &= \begin{cases}
                \frac{1}{|E_{(w,h)}|} &\text{if } r(w,h) \in E_{(w,h)} \\
                0 &\text{otherwise}
                \end{cases} \\
                &= \begin{cases}
                \frac{1}{|E_{(\pi(w),\pi(h))}|} &\text{if } \phi(r)(\pi(w),\pi(h)) \in E_{(\pi(w),\pi(h))} \\
                0 &\text{otherwise}
                \end{cases}\\ &= \sP(S_{1} = \phi(r) \mid U_{1}=\pi(w))
            \end{aligned}
        \]
        \normalsize
        \item \textbf{Scenario 2:} selecting an edge with query relation type at random.
        Here, we use the fact that isomorphisms preserve the number of edges of a given type.
        Hence, $\mu$ defines a bijection between the sets of edges with type $r$ in $G$ and type $\phi(r)$ in $H$,
        which allows us to conclude that this scenario is also invariant in probability.
        \item \textbf{Scenario 3:} selecting the first step completely at random. This case is similar to Scenario 1 -- using the invariance of the number of neighboring nodes under isomorphism, we can repeat similar calculations in a straightforward manner to show probabilistic invariance. 
    \end{itemize}
    Either way, we find that the selection process of the first step of $\eta$ over $G$ translates naturally via $\mu$ to the choice of the first step over $H$, proving the desired statement.
\end{proof}

\begin{lemma}
    \label{lem:uniform_walk_invariance}
    Suppose that the first step $v_0 \xrightarrow{r_1}v_1$ is chosen in an invariant manner.
    Then, the uniform random walk with no backtracking algorithm $\eta$ is invariant in probability.
\end{lemma}
\begin{proof}
Let $G=(V,E,R)$ be a knowledge graph, and let $\ell$ be the length of random walks.
Let $H$ be a KG isomorphic to $G$ via the isomorphism $\mu=(\pi, \phi)$. We aim to show that:
\[
\mu(\eta(G,\ell)) = \pi(V_0) \xrightarrow{\phi(R_1)} \pi(V_1) \xrightarrow{\phi(R_2)} \dots \xrightarrow{\phi(R_\ell)} \pi(V_\ell) \probeq U_0 \xrightarrow{S_1} U_1 \xrightarrow{S_2} \dots \xrightarrow{S_\ell} U_\ell = \eta(H,\ell)
\]
Let
\(
\bar\eta = v_0 \xrightarrow{r_1} v_1 \xrightarrow{r_2} \dots \xrightarrow{r_\ell} v_\ell \in \mathcal{W}(G)
\)
be a walk of length $\ell$ over $G$. It suffices to show that the probability of sampling $\bar\eta$ from $G$ is identical to the probability of sampling $\mu(\bar\eta)$ from $H$:
\[
\sP(\eta(G,\ell) = \bar\eta) = \sP(\eta(H,\ell) = \mu(\bar\eta))
\]
To see this, let us expand the definitions of $\sP(\eta(G,\ell) = \bar\eta)$:
\begin{equation}
\label{eq:rw_on_G_expanded}
\begin{aligned}
\sP(\eta(G,\ell) = \bar\eta) =& \sP(V_0 = v_0) \\
&\cdot \sP(V_1 = v_1 \mid V_0 = v_0)\\
&\cdot \prod^{\ell-2}_{i=0} \sP(V_{i+2}=v_{i+2} \mid V_{i+1} = v_{i+1}, V_i = v_i)\\
&\cdot \prod^{\ell-1}_{i=0} \sP(R_{i+1}=r_{i+1} \mid V_{i+1} = v_{i+1}, V_i = v_i)  
\end{aligned}
\end{equation}
and $P(\eta(H,\ell) = \mu(\bar\eta))$:
\begin{equation}
\label{eq:rw_on_H_expanded}
\begin{aligned}
\sP(\eta(H,\ell) = \mu(\bar\eta)) =& \sP(U_0 = \pi(v_0)) \\
&\cdot \sP(U_1 = \pi(v_1) \mid U_0 = \pi(v_0))\\
&\cdot \prod^{\ell-2}_{i=0} \sP(U_{i+2}=\pi(v_{i+2}) \mid U_{i+1} = \pi(v_{i+1}), U_i = \pi(v_i))\\
&\cdot \prod^{\ell-1}_{i=0} \sP(S_{i+1}=\phi(r_{i+1}) \mid U_{i+1} = \pi(v_{i+1}), U_i = \pi(v_i))  
\end{aligned}
\end{equation}
Given that the graph isomorphism preserves the number of neighbors for each node and is a bijection, we can easily verify using the definitions from (\ref{eq:uniform_random_walk_markov}) that the following indeed hold:
\small
\begin{equation}
\label{eq:rw_twohop_invariance}
\begin{aligned}
    \sP(V_{i+2} = v_{i+2} \mid V_{i+1} = v_{i+1}, V_i = v_i) &= \sP(U_{i+2} = \pi(v_{i+2}) \mid U_{i+1} = \pi(v_{i+1}), U_i = \pi(v_i)) \\ 
    \sP(R_{j+1} = r_{j+1} \mid V_{j+1} = v_{j+1}, V_j = v_j) &= \sP(S_{j+1} = \phi(r_{j+1}) \mid U_{j+1} = \pi(v_{j+1}), U_j = \pi(v_j)) 
\end{aligned}
\end{equation}
\normalsize
for all $i \in \{0, 1, \dots, \ell-2\}, j\in\{1, \dots, \ell-1\}$.
Moreover, by the assumption that the first step $V_0 \xrightarrow{R_1} V_1$ is invariant, we have:
\begin{equation}
\label{eq:initial_step_invariance}
\begin{aligned}
\sP((V_0,R_1,V_1) = (v_0, r_1, v_1)) = \sP((U_0, S_1, U_1) = (\pi(v_0), \phi(r_1), \pi(v_1)))
\end{aligned}
\end{equation}
But by the laws of conditional probability:
\small
\[
\begin{aligned}
\sP((V_0,R_1,V_1) = (v_0, r_1, v_1)) &= \sP(R_1=r_1 \mid V_0 = v_0, V_1 = v_1) \cdot \sP(V_0 = v_0, V_1 = v_1)\\
&= \sP(R_1=r_1 \mid V_0 = v_0, V_1 = v_1) \cdot \sP(V_1 = v_0 \mid V_0 = v_0) \cdot \sP(V_0 = v_0)
\end{aligned}
\]
\normalsize
and analogously:
\small
\[
\begin{aligned}
\sP((U_0,&S_1,U_1) = (\pi(v_0), \phi(r_1), \pi(v_1))) \\
&= \sP(S_1=\phi(r_1) \mid U_0 = \pi(v_0), U_1 = \pi(v_1)) \cdot \sP(U_1 = \pi(v_0) \mid U_0 = \pi(v_0)) \cdot \sP(U_0 = \pi(v_0))
\end{aligned}
\]
\normalsize
Substituting these equalities into (\ref{eq:initial_step_invariance}) and multiplying both sides by the equalities from (\ref{eq:rw_twohop_invariance}) for all choices of $i \in \{0, 1, \dots, \ell-2\}, j\in\{1, \dots, \ell-1\}$, we get precisely the equality of the right sides of \Cref{eq:rw_on_G_expanded} and \Cref{eq:rw_on_H_expanded}. Hence,
\[
\sP(\eta(G,\ell)=\bar\eta) = \sP(\eta(H,\ell) = \mu(\bar\eta))
\]
and we can conclude that $\mu(\eta(G,\ell)) \probeq \eta(H,\ell)$, and the algorithm $\eta$ is invariant in probability.
\end{proof}

\begin{corollary}
    \label{cor:invariance_of_random_walk}
    The random walk algorithm presented in \Cref{sec:random_walk_details} is invariant in probability.
\end{corollary}

\begin{lemma}
    \label{lem:recording_function_invariance}
    The recording protocol $w$, as described in \Cref{sec:recording_protocol_details}, is invariant, provided that the embedding functions $\mathbf{v}$ and $\mathbf{r}$ are invariant.
\end{lemma}
\begin{proof}
    Let $G=(V,E,R)$ and $H=(V',E',R')$ be isomorphic knowledge graphs with the isomorphism $\mu=(\pi,\phi)$ taking $G$ to $H$.
    Let $q = (h_q, r_q, ?)$ be a link prediction query over $G$, and $\mu(q) = (\pi(h_q), \phi(r_q), ?)$ be its identification in $H$.
    Let $\bar\eta = v_0 \xrightarrow{r_1} v_1 \xrightarrow{r_2} \dots \xrightarrow{r_\ell} v_\ell \in \mathcal{W}(G)$ be a~walk over $G$, and
    $\bar\eta' = \mu(\bar\eta) = \pi(v_0) \xrightarrow{\phi(r_1)} \pi(v_1) \xrightarrow{\phi(r_2)} \dots \xrightarrow{\phi(r_\ell)} \pi(v_\ell)$ be the analogous walk over $H$. To prove that the recording protocol $w$ outlined in \Cref{sec:recording_protocol_details} is invariant, it suffices to show that the encoding of each step:
    \[
S_i = \left(\operatorname{id}_V(v_i; \bar\eta), \operatorname{id}_R(r_i; \bar\eta), \operatorname{dir}_i, \delta_{v_i=h_q}, \delta_{r_i=r_q}, \mathbf{v}(v_i), \mathbf{r}(r_i)\right)
\]
    is identical for $\bar\eta$ and $\bar\eta'$. We will show this for each component:
    \begin{itemize}
        \item since $\pi$ defines a bijection between nodes in $G$ and $H$, for any $i$, we have:
        \[
        \operatorname{id}_V(v_i; \bar\eta) = \argmin_t [v_t = v_i] = \argmin_t [\pi(v_t) = \pi(v_i)] = \operatorname{id}_V(\pi(v_i);\bar\eta')
        \]
        \item similarly to the point above, $\phi$ is a bijection between relations of $G$ and $H$, so we can write:
        \[
        \begin{aligned}
        \operatorname{id}_R(r_i; \bar\eta) &= \argmin_t \left[r_t = r_i \lor r_t = r_i^{-1}\right] \\
        &= \argmin_t \left[\phi(r_t) = \phi(r_i) \lor \phi(r_t) = \phi(r_i)^{-1}\right] \\
        &= \operatorname{id}_R(\phi(r_i);\bar\eta')
        \end{aligned}
        \]
        \item $\operatorname{dir}_i$ is clearly preserved, as the isomorphism $\mu$ preserves the directions of edges,
        \item as $\pi, \phi$ are bijections the masks $\delta_{v_i=h_q}, \delta_{r_i=r_q}$, representing whether the $i$'th node and relation match the types in the query, satisfy:
        \[
        \begin{aligned}
            v_i = h_q \iff \pi(v_i) = \pi(h_q) \quad &\implies \quad \delta_{v_i=h_q} = \delta_{\pi(v_i) = \pi(h_q)} \\
            r_i = r_q \iff \phi(r_i) = \phi(r_q) \quad &\implies \quad \delta_{r_i=r_q} = \delta_{\phi(r_i) = \phi(r_q)} \\
        \end{aligned}
        \]
        \item $\mathbf{v}$ and $\mathbf{r}$ are invariant by assumption, so:
        \[
        \mathbf{v}(v_i) = \mathbf{v}(\pi(v_i)) \qquad \text{and} \qquad \mathbf{r}(r_i) = \mathbf{r}(\phi(r_i))
        \]
    \end{itemize}
    Combining all these observations, we can conclude that $w(\bar\eta; G, q, \mathbf{v}, \mathbf{r})=w(\mu(\bar\eta); H, \mu(q), \mathbf{v}, \mathbf{r})$ and $w$ is indeed invariant.
\end{proof}

\begin{lemma}
    \label{lem:consensus_protocol_invariance}
    The consensus protocol $c$, as described in \Cref{sec:consensus_protocol_details}, is invariant.
\end{lemma}
\begin{proof}
    Let $G=(V,E,R)$ be a knowledge graph and $H$ be isomorphic to $G$ via an isomorphism $\mu=(\pi, \phi)$. 
    Let $\bar\eta_{1:N}\in \mathcal{W}(G)$ be a sequence of walks in $G$.
    To show that the output of the consensus protocol is invariant, we need to prove that for each $v\in V$ and $r\in R$, the following holds:
    \begin{equation}
        \label{eq:consensus_invariance}
        \Delta\mathbf{v}(v) = \mathbf{v}'(\pi(v)) \qquad \text{and} \qquad \Delta\mathbf{r}(r) = \Delta\mathbf{r}'(\phi(r))
    \end{equation}
    where $(\Delta\mathbf{v},\Delta\mathbf{r}) = c(\mathbf{h}, \bar\eta_{1:N})$ and $(\Delta\mathbf{v}',\Delta\mathbf{r}') = c(\mathbf{h}, \mu(\bar\eta_{1:N}))$ for $\mathbf{h} = (\mathbf{s}^{(V)}, \mathbf{s}^{(R)}, \mathbf{a}^{(V)}, \mathbf{a}^{(R)})$. 

    The result follows from the fact that $\pi$ and $\phi$ are bijections -- whenever $v$ is the $j^\text{th}$ vertex visited in the walk $\bar\eta_i$, the  $j^\text{th}$ node of $\mu(\bar\eta_i)$ must be $\pi(v)$ (and vice versa).
    An analogous result holds for the relations.
    Hence, the aggregation performed by $c$ for $v$ (resp. $r$) over $\bar\eta_{1:N}$ is equivalent to the aggregation for $\pi(v)$ (resp. $\phi(r)$) over $\mu(\bar\eta_{1:N})$, and (\ref{eq:consensus_invariance}) is indeed satisfied. 
\end{proof}
    
\begin{propositioncopy}{\ref{prop:specific_invariance}}
$\flock$ with components as described in \Cref{sec:methods} is invariant in probability.
\end{propositioncopy}
\begin{proof}
    The result follows naturally from aggregating the results of Corollary \ref{cor:invariance_of_random_walk} and Lemmas \ref{lem:recording_function_invariance} and \ref{lem:consensus_protocol_invariance}, followed by applying Proposition \ref{prop:general_invarinace}.
\end{proof}

\newpage
\section{Details of the \textsc{Petals} benchmark}
\label{app:synthetic-datasets}

\begin{figure}
    \centering
    
    \begin{tikzpicture}[
        scale=0.7,
        every node/.style={circle, draw, fill=black, inner sep=2pt, line width=0.7pt}, 
        line width=0.75pt, 
        >=Stealth, 
        shorten >=2pt, 
        shorten <=2pt,
        transform shape
    ]

    \node (b0) [label={[yshift = -2.5em, xshift = 0em]90:$b_0$}] at (0, 0) {};

    \node (a11) [label={[yshift = -0.5em, xshift = 1em]90:$a^{(1)}_2$}] at (-1, 2) {};
    \node (a12) [label={[yshift = -0.5em, xshift = 1em]90:$a^{(1)}_4$}] at (-2, 3) {};
    \node (a13) [label={[yshift = -2.8em, xshift = -1em]90:$a^{(1)}_1$}] at (-2, 1) {};
    \node (a14) [label={[yshift = -2.8em, xshift = -1em]90:$a^{(1)}_3$}] at (-3, 2) {};
    \node (a15) [label={[yshift = -1em, xshift = -1em]90:$a^{(1)}_5$}] at (-3.5, 3.5) {};

    \node (a21) [label={[yshift = -0.5em, xshift = -0.8em]90:$a^{(2)}_1$}] at (1, 2) {};
    \node (a22) [label={[yshift = -0.7em, xshift = -0.7em]90:$a^{(2)}_3$}] at (2, 3) {};
    \node (a23) [label={[yshift = -2.2em, xshift = 1.3em]90:$a^{(2)}_2$}] at (2, 1) {};
    \node (a24) [label={[yshift = -2.5em, xshift = 1em]90:$a^{(2)}_4$}] at (3, 2) {};
    \node (a25) [label={[yshift = -1em, xshift = 1.4em]90:$a^{(2)}_5$}] at (3.5, 3.5) {};

    \node (a31) [label={[yshift = -2.8em, xshift = -1.1em]90:$a^{(3)}_2$}] at (1, -2) {};
    \node (a32) [label={[yshift = -2.8em, xshift = -1.1em]90:$a^{(3)}_4$}] at (2, -3) {};
    \node (a33) [label={[yshift = -0.5em, xshift = 1em]90:$a^{(3)}_1$}] at (2, -1) {};
    \node (a34) [label={[yshift = -0.5em, xshift = 1em]90:$a^{(3)}_3$}] at (3, -2) {};
    \node (a35) [label={[yshift = -2.5em, xshift = 1em]90:$a^{(3)}_5$}] at (3.5, -3.5) {};
    
    \node (a41) [label={[yshift = -2.5em, xshift = 1.1em]90:$a^{(4)}_1$}] at (-1, -2) {};
    \node (a42) [label={[yshift = -2.5em, xshift = 1.1em]90:$a^{(4)}_3$}] at (-2, -3) {};
    \node (a43) [label={[yshift = -1.1em, xshift = -1em]90:$a^{(4)}_2$}] at (-2, -1) {};
    \node (a44) [label={[yshift = -1.1em, xshift = -1.1em]90:$a^{(4)}_4$}] at (-3, -2) {};
    \node (a45) [label={[yshift = -2.5em, xshift = -1.4em]90:$a^{(4)}_5$}] at (-3.5, -3.5) {};

    \node (b1) [label={[yshift = 0em, xshift = 0em]90:$b_1$}] at (4, 0) {};
    \node (b2) [label={[yshift = 0em, xshift = 0em]90:$b_2$}] at (8, 0) {};
    \node (b3) [label={[yshift = 0em, xshift = 0em]90:$b_3$}] at (12, 0) {};

    \draw[->, color = likeblue] (b0) to (a11);
    \draw[->, color = dislikered] (b0) to (a13);
    \draw[->, color = dislikered] (a11) to (a12);
    \draw[->, color = dislikered] (a13) to (a14);
    \draw[->, color = dislikered] (a12) to (a15);
    \draw[->, color = dislikered] (a14) to (a15);

    \draw[->, color = likeblue] (b0) to (a21);
    \draw[->, color = unnecessarypink] (b0) to (a23);
    \draw[->, color = likeblue] (a21) to (a22);
    \draw[->, color = likeblue] (a23) to (a24);
    \draw[->, color = likeblue] (a22) to (a25);
    \draw[->, color = likeblue] (a24) to (a25);

    \draw[->, color = unapologeticyellow] (b0) to (a31);
    \draw[->, color = unnecessarypink] (b0) to (a33);
    \draw[->, color = unnecessarypink] (a31) to (a32);
    \draw[->, color = unnecessarypink] (a33) to (a34);
    \draw[->, color = unnecessarypink] (a32) to (a35);
    \draw[->, color = unnecessarypink] (a34) to (a35);

    \draw[->, color = unapologeticyellow] (b0) to (a41);
    \draw[->, color = dislikered] (b0) to (a43);
    \draw[->, color = unapologeticyellow] (a41) to (a42);
    \draw[->, color = unapologeticyellow] (a43) to (a44);
    \draw[->, color = unapologeticyellow] (a42) to (a45);
    \draw[->, color = unapologeticyellow] (a44) to (a45);

    \draw[->, color=solidgreen] (b0) to (b1);
    \draw[->, color=solidgreen] (b1) to (b2);
    \draw[->, color=solidgreen] (b2) to (b3);

    \draw[->, color=solidgreen, dashed] (b0) to[out=90, in=-20] (a11);
    \draw[->, color=solidgreen, dashed] (b0) to[out=180, in=-70] (a13);

    \end{tikzpicture}
    \caption{An example of a graph from \textsc{Petals} with $c=4$, $l=2$ and $t=3$, and the associated link prediction instances (dashed). The relation types `\textcolor{dislikered}{red}', `\textcolor{likeblue}{blue}', `\textcolor{unnecessarypink}{pink}' and `\textcolor{unapologeticyellow}{yellow}' are structurally isomorphic, hence become equated in the eyes of the existing KGFMs. }
    \vspace{-1em}
    \label{fig:relish-example}
\end{figure}

State-of-the-art knowledge graph foundation models (KGFMs) typically impose relational invariance.
Formally, given two knowledge graphs $G=(V, E, R)$ and $H=(V', E', R')$, if there exists an isomorphism $(\pi, \phi)$ from $G$ to $H$, then for any $r\in R$, the model enforces identical representations for $r$ and its image $\phi(r)\in R'$.
This design promotes generalization across different graphs, as it aligns analogous relations, but reduces expressivity within a single graph ($G=H$), where relations related by automorphisms are forced to be indistinguishable.
Concretely, if an automorphism $(\pi, \phi)$ of $G$ maps $r_1$ to $r_2$, then the model must treat $r_1$ and $r_2$ as identical during inference. 
While some approaches mitigate this limitation via the labeling trick, assigning distinct embeddings to query-specific nodes and relations, this only isolates the queried relation type and does not resolve the underlying issue in general.

Motivated by this limitation, we introduce the \textsc{Petals} benchmark.
\textsc{Petals} comprises 220 graphs, each paired with a link prediction query $(h,r,?)$ and a target set $\{t_1, t_2\}$.
While $t_1$ and $t_2$ are non-isomorphic, KGFMs enforcing relational invariance are unable to distinguish them, producing identical predictions.
We empirically validate this property by evaluating the classification accuracy of marking $t_1$ as \textsc{true} and $t_2$ as \textsc{false}, reported in \Cref{tab:synthetic-exp-results}.

\subsection{Structure of the studied KGs}
Knowledge graphs in \textsc{Petals} follow a flower-like structure, parametrized by the number $c$ of `petals', their length $l$ and the length $t$ of the `stem' (see \Cref{fig:relish-example} for visualization).

\textbf{Vertices. }
Each `petal' is a set $A^{(i)}$ of $2l+1$ vertices $A^{(i)} = \left\{a_1^{(i)}, a_2^{(i)}, \dots, a_{2l+1}^{(i)}\right\}$, while the stem $B$ consists of $t+1$ nodes $B = \{b_0, b_1, \dots, b_t\}$.
The full set of entities is then:
\[
V = B \cup \bigcup_{i=1}^{c} A(i) = \{b_0, b_1, \dots, b_t\} \cup \left\{a_j^{(i)} \mid 1 \leq i \leq c, 1 \leq j \leq 2l+1 \right\}
\]
We call $b_0$ the `central' node, as it is connected to every petal, as described below.

\textbf{Edges. }
The nodes of the stem are connected in a consecutive manner by the same relation type $r_0$.
Precisely, for each $i \in {1, \cdots, t}$, there exists an edge $(b_{i-1},r_0, b_i)$.
Each petal $A^{(i)}$ is associated with two edge types $r^{(i)}_1, r^{(i)}_2$, and is connected to the central node $b_0$ with links $\!\left(b_0, r^{(i)}_1, a_1^{(i)}\right)$ and $\!\left(b_0, r^{(i)}_2,a_2^{(i)}\right)$. The rest of the petal is connected with edges of type $r_1^{(i)}$ only, going from $a^{(i)}_{2j-1}$ to $a^{(i)}_{2j+1}$, and from $a^{(i)}_{2j}$ to $a^{(i)}_{2j+2}$. Finally, there are also edges linking $a^{(i)}_{2\ell-1}$ and $a^{(i)}_{2l}$ to $a^{(i)}_{2l+1}$.
Therefore, the full set of edges can be characterized as:
\[
\begin{aligned}
E = \left(\{(b_{i-1}, r_0, b_i) \mid 1 \leq i \leq t\}\right) &\cup \left( \bigcup_{i=1}^c \left\{(b_0, r^{(i)}_1, a^{(i)}_1), (b_0, r^{(i)}_2, a^{(i)}_2)\right\} \right)\\
&\cup \left( \bigcup_{i=1}^c\bigcup_{j=1}^{\ell-1} \left\{\!\left(a^{(i)}_{2j-1}, r^{(i)}_1, a^{(i)}_{2j+1}\right), \!\left(a^{(i)}_{2j}, r^{(i)}_{1}, a^{(i)}_{2j+2}\right)\right\} \right)\\
&\cup \left(\bigcup_{i=1}^c \left\{\!\left(a^{(i)}_{2\ell-1}, r^{(i)}_1, a^{(i)}_{2l+1}\right), \!\left(a^{(i)}_{2l}, r^{(i)}_1, a^{(i)}_{2l+1}\right)\right\}\right)
\end{aligned}
\]

We select each of the types $r^{(i)}_1$ and $r^{(i)}_2$ from the set of considered relations $R = \{r_1, \dots, r_{|R|}\}$ so that any relation-invariant model will equate all petals (i.e. so that for each pair of petals, there is an automorphism taking one to another).
For instance, \Cref{fig:relish-example} displays a cyclic pattern, in which $r^{(i)}_2 = r^{(i+1)}_1$.
Such symmetry causes all petals to be isomorphic, and leads to the inability of KGFMs to distinguish between the relations inside them. 

\textbf{Link prediction instances. }
Although the petals are isomorphic to each other, given the asymmetry of edge types from $b_0$ to $a^{(i)}_1$ and $a^{(i)}_2$, the nodes within a single petal generally can be distinguished.
Therefore, for each graph with the structure as described above, we randomly sample one of the stem nodes $b_s$, and ask the link prediction query $(b_s, r_0, ?)$.
For the target nodes, we randomly select petal index $i$ and distance $j$ from the central node $b_0$, and consider the predictions for $a^{(i)}_{2j-1}$ and $a^{(i)}_{2j}$.
For example, \Cref{fig:relish-example} shows the case when $b_s = b_0$, $i=1$ and $j=1$, where the query is $(b_0, r_0, ?)$ and we are interested in the scores for $a^{(1)}_{1}$ and $a^{(1)}_{2}$.  

\subsection{Parameters and generation}
We construct \textsc{Petals} by manually designing 11 relation-assignment schemes that guarantee isomorphism across all petals.
For each such selection, which already determines the number $c$ of petals, we generate 20 graphs corresponding to all combinations of $t\in \{1,2,3,4\}$ and $l \in \{1,2,3,4,5\}$.
Each graph is paired with a link prediction query and two target nodes, sampled as described above.
This yields $11\cdot20 = 220$ instances that constitute the \textsc{Petals} benchmark.

\newpage

\section{Computational Complexity} 
\label{sec:complexity}

Recall that $I$ is the iterations in each forward pass of $\flock$; $n$ is the base walk count; $\ell$ is the walk length; $L$ is the number of linear sequence-model layers (such as GRU); and $d$ is the hidden dimension for the sequence processor. Note that in practice, we perform $P$ forward passes and ensemble their outputs to reduce variance. For a single pass ($P{=}1$), walk sampling and recording cost $O(n\ell)$, while the sequence processor with $L$ layers of hidden dimension $d$ costs $O(n\ell L d^{2})$. The consensus protocol costs $O(n\ell d)$. In total, the time complexity is $O\big(PIn\ell Ld^{2}\big)$, which scales linearly with the number of (base) walks $n$, the length of walks $\ell$, and the number of ensembled predictions $P$. We empirically verified this in \Cref{app:scalability_analysis}. 

Compared with message-passing KGFMs like $\ultra$ and $\trix$, $\flock$'s complexity is \emph{independent} of the graph size and average degree; empirically, however, using more walks (increasing $n$) and longer walks (increasing $\ell$) improves coverage and yields more fine-grained representation.  

The space complexity of $\flock$ per forward pass is $O(n\ell d)$ plus model parameters $O(Ld^{2})$. Note that running ensembles sequentially keeps peak memory near this bound, while parallel ensembling increases it by a factor of $P$.

\section{Scalability analysis}
\label{app:scalability_analysis}

\begin{table}[t!]
\centering
\caption{Training scalability analysis on a single NVIDIA RTX A6000 (48\,GB) with batch size = 8. $\flock$ using $16$ number of base walks and $1$ ensemble.}
\footnotesize
\label{tab:train-scalability}
\begin{tabular}{lccc}
\toprule
\textbf{Model} & \textbf{Parameters} & \textbf{Time / batch (s)} & \textbf{GPU memory (GB)} \\
\midrule
\ultra  & 168{,}705 & 0.117 & 2.110 \\
\trix   & 87{,}138  & 0.690 & 3.442 \\
\flock  & 801{,}969 & 1.30  & 27.89 \\
\bottomrule
\end{tabular}
\end{table}

\begin{table}[t!]
\centering
\caption{Inference scalability on a single NVIDIA RTX A6000 (48\,GB) with batch size $=8$.
Left columns specify base walks $n$ and ensembled passes $P$. Dashes indicate not applicable.}
\footnotesize
\label{tab:test-scalability}
\begin{tabular}{lcccc}
\toprule
\textbf{Model} & \textbf{\# Base Walks $n$} & \textbf{Ensemble $P$} & \textbf{Time /batch (s)} & \textbf{GPU memory (GB)} \\
\midrule
\ultra & ---  & 1  & 0.073 & 0.848 \\
\trix  & ---  & 1  & 0.500 & 1.382 \\
\midrule
\multirow{9}{*}{\flock}
& 16  & 1   & 1.26 & 2.868 \\
&  16 & 2   & 1.99 & 2.864 \\
&  16 & 4   & 3.24 & 3.938 \\
&  16 & 8   & 5.45 & 5.172 \\
&  16 & 16  & 9.45 & 8.892 \\
& 128 & 1   & 1.77 & 5.000 \\
& 128 & 2   & 2.80 & 7.880 \\
& 128 & 4   & 5.00 & 14.42 \\
& 128 & 8   & 10.05 & 43.68 \\
\bottomrule
\end{tabular}
\end{table}

To investigate the scalability of the proposed method $\flock$, we report the training and inference time per batch and peak GPU memory for $\ultra$, $\trix$, and variants of $\flock$ on a single RTX A6000 (48 GB) in \Cref{tab:train-scalability,tab:test-scalability}. 

\paragraph{Training.} During training, we fix $\flock$ to $n=16$ base walks and with an ensemble size of $P=1$, which yields higher cost than $\ultra$/$\trix$ but remains feasible on a single GPU. In addition, unlike $\ultra$/$\trix$, $\flock$ does not rely on GNN message passing where highly optimized fused sparse kernels (e.g., RSPMM kernel developed in \citet{zhu2022neural}) accelerate computation; instead, $\flock$’s runtime is dominated by walk sampling and sequence encoding, making time per batch the main bottleneck. As a result, pretraining typically takes about three days. One avenue for future work is to develop similarly highly optimized kernels for random-walk sampling.

\paragraph{Inference.} Additionally, we report the inference results in \Cref{tab:test-scalability}, where we vary the number of walks $n$ and ensembled passes $P$. We observe near-linear growth of latency and VRAM with $n$, reflecting the dominant costs of walk sampling and sequence processing. Note that during inference, ensembled predictions are parallelizable, meaning that with sufficient GPU memory, these $P$ stochastic passes can be executed concurrently, so the latency grows sublinearly in $P$, while peak VRAM scales roughly linearly with $P$. In practice, reducing $n$ (walks) or $P$ (ensembled passes) lowers both memory and latency, while larger $n/P$ settings trade extra cost for better coverage and stability on harder KGs.

\section{Noise injection over existing KGFMs}
\label{sec:noise_injection}

\begin{table}[t!]
\caption{Noise injection over the best performing KGFM baseline $\trix$.}
  \centering
  \setlength{\tabcolsep}{4pt}
  \footnotesize
  \begin{subtable}[t]{0.33\linewidth}
    \centering
    \caption{Zero-shot entity prediction.}
    \label{tab:noise-injection-entity}
    \begin{tabular}{@{}lcc@{}}
      \toprule
      & \textbf{MRR} & \textbf{Hits@10} \\
      \midrule
      $\trix$      & 0.366 & 0.518 \\
      \,+ noise    & 0.385 & 0.545 \\
      $\flock$     & \textbf{0.391} & \textbf{0.560} \\
      \bottomrule
    \end{tabular}
  \end{subtable}\hfill
  \begin{subtable}[t]{0.33\linewidth}
    \centering
    \caption{Zero-shot relation prediction.}
    \label{tab:noise-injection-relation}
    \begin{tabular}{@{}lcc@{}}
      \toprule
      & \textbf{MRR} & \textbf{Hits@1} \\
      \midrule
      $\trix$      & 0.792 & 0.687 \\
      \,+ noise    & 0.739 & 0.643 \\
      $\flock$     & \textbf{0.881} & \textbf{0.817} \\
      \bottomrule
    \end{tabular}
  \end{subtable}\hfill
  \begin{subtable}[t]{0.33\linewidth}
    \centering
\caption{Accuracy on \textsc{Petals}.}
\label{tab:noise-injection-petals}
\begin{tabular}{@{}lc@{}}
  \toprule
  & {\textbf{Accuracy}} \\
  \midrule
  {$\trix$}      & {50\%} \\
  {\,+ noise}    & {52\%} \\
  {$\flock$}     & {\textbf{100\%}} \\
  \bottomrule
\end{tabular}
\end{subtable}
\end{table}

\paragraph{Setup.} \draft{Since noise injection is a possible way to build a probabilistic equivariant KGFM in a different way from our approach \citep{gao2023double}, it is natural to ask how such KGFMs would perform compared to $\flock$. To answer this question,} we apply noise injection over the best performing KGFMs baselines $\trix$. Specifically, in each forward pass, we add element-wise noise sampled from a uniform distribution $\epsilon \sim \gU[-0.5,0]$ to all relation and entity embeddings after the initialization stage. Note that the addition of noise technically breaks deterministic node-relation equivariance, but the resulting model ($\trix$ + noise) still respects probabilistic node-relation equivariance. We then pretrain $\trix$ using the same experimental setup shown in \Cref{sec:main_experiment}, and compare with $\trix$ without noise injection and $\flock$. To minimize the variance induced by injected noise and to ensure a fair comparison, we report ensembled prediction results with 16 samples for both $\trix$ + noise and $\flock$.
\draft{This is a strong baseline implementing the ideas of prior work on noise injection and test-time ensembling for message passing networks on KGs \citep{ingram, gao2023double}.}

\paragraph{Results.} We report the average zero-shot performance for entity prediction and relation prediction over 54 KGs in \Cref{tab:noise-injection-entity,tab:noise-injection-relation}, respectively, {as well as trained performance for \textsc{Petals} in} \Cref{tab:noise-injection-petals}. Across all tasks, $\trix$ with naive noise injection fails to close the gap between $\flock$. In particular, $\trix$ + noise degrades compared with vanilla $\trix$ without noise injection in relation prediction, while boosting the performance in the entity prediction task. We hypothesize that such a difference lies in the added randomization breaks symmetry among entity embeddings more than among relation embeddings, and entity prediction depends more on having distinguishable entity representations than relation prediction does.
Additionally, we attribute this performance gap between $\flock$ and $\trix$ + noise to the source of randomization. $\flock$ introduces stochasticity through random walks, which induces \emph{structure-informed} perturbations that respect the underlying topology. In contrast, $\trix$ with naive noise injection attempts to break deterministic node-relation equivariance by introducing structure-agonistic noise, which might, in turn, hurt the model's generalization. 
Together, these findings suggest that simply adding structure-agonistic noise is insufficient; performance gains only arise when stochasticity is topology-aware and is induced from the graph structure in a principled way.

\section{Case study: relation embedding on \textsc{Metafam}}
\label{app:case_study}

\begin{figure}[t]
    \centering
    \begin{subfigure}[b]{0.49\textwidth}
        \includegraphics[width=\textwidth]{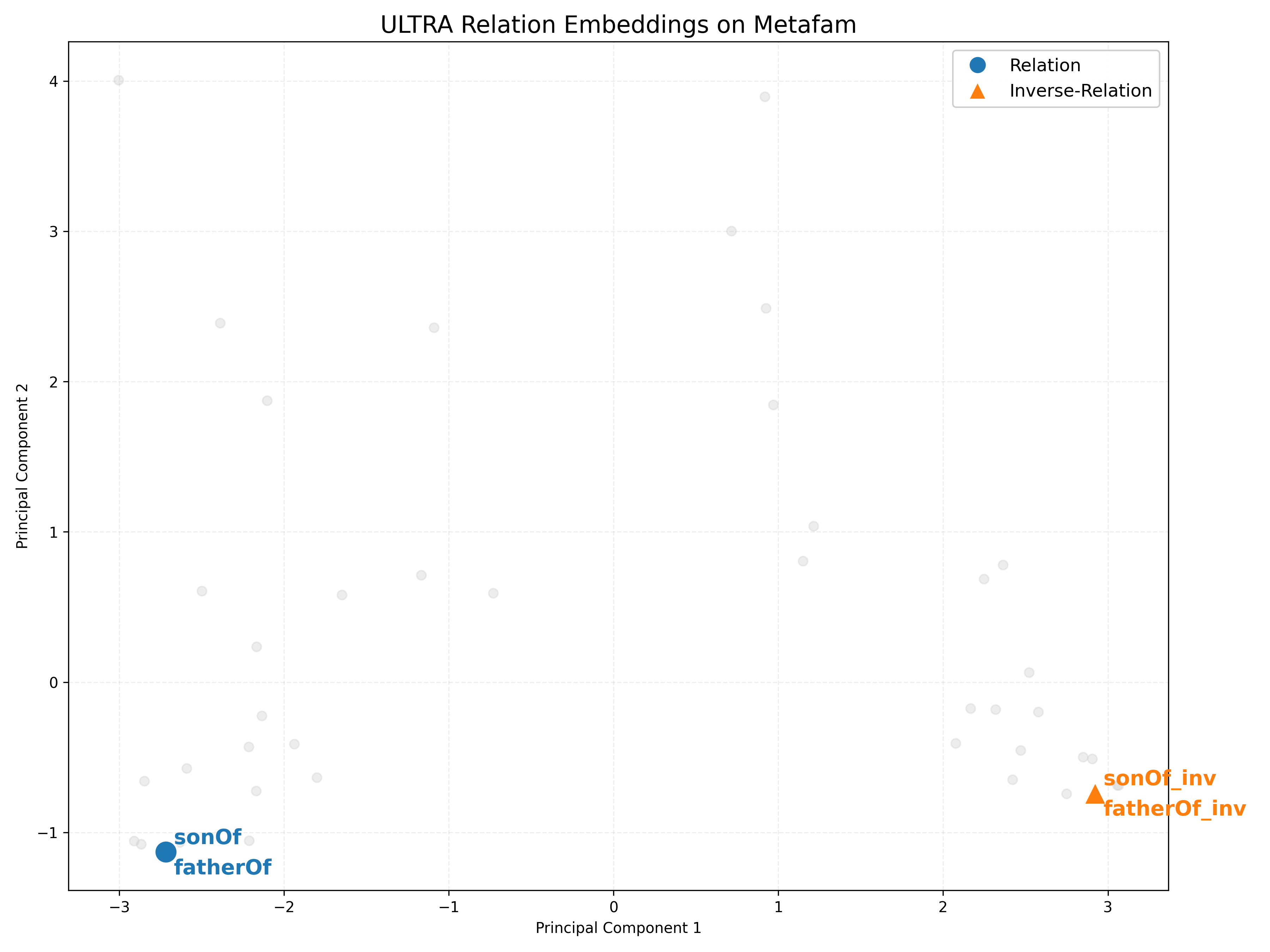}
    \end{subfigure}
    \hfill
    \begin{subfigure}[b]{0.49\textwidth}
        \includegraphics[width=\textwidth]{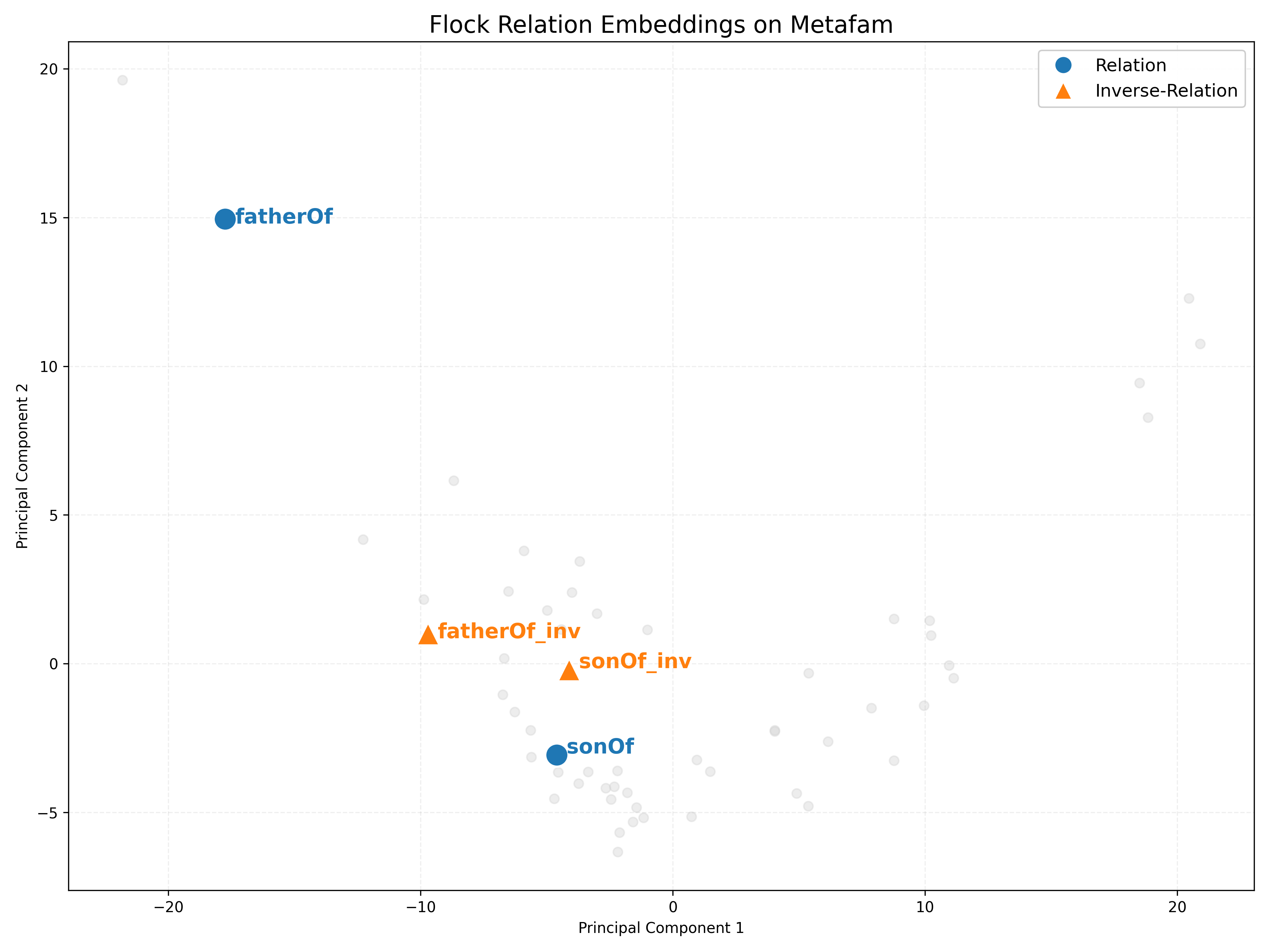}
    \end{subfigure}
    \caption{{PCA of relation embeddings on \textsc{Metafam}. $\ultra$ (Left) maps several inverse pairs (e.g., \texttt{fatherOf} vs.\ \texttt{sonOf})
    to almost similar embeddings, where $\flock$ (Right) yields clearly separated embeddings, indicating that its probabilistic equivariance allow $\flock$ to distinguish between these semantically different relations, explaining its strong zero-shot performances.}}
    \label{fig:case_study}
\end{figure}

\paragraph{Setup.} To further showcase why expressivity matters for zero-shot generalization, we present a case study on the \textsc{Metafam} dataset~\citep{zhou2023multitaskperspetivelinkprediction}. \textsc{Metafam} is built from a fixed family-relations ontology:
during training, models observe edges with relations
\texttt{motherOf}, \texttt{fatherOf}, \texttt{daughterOf}, \texttt{sonOf},
while the test queries only involve \texttt{motherOf} and \texttt{fatherOf}.
Up to gender symmetries, this reduces to two effective predictive patterns
(\emph{parent\_of} vs.\ \emph{child\_of}), and the test set focuses on a single one
(\emph{parent\_of}). 

In the zero-shot setting, KGFMs cannot adapt to this ontology and must rely on their pretrained relation representations.
Notice that here, \textsc{Metafam} is challenging: many relations are structurally similar (e.g., \texttt{fatherOf}, \texttt{sonOf}, \texttt{sisterOf}, \texttt{nieceOf}) yet encode opposite predictive patterns.

\paragraph{Result.}

\Cref{fig:case_study} shows that $\ultra$’s relation embeddings largely collapse
these families, placing \texttt{fatherOf} and \texttt{sonOf} in almost identical
positions in the PCA plane.
This collapse makes it difficult to distinguish who is the parent and who is the child, leading to poor zero-shot performance. 
$\flock$, in contrast, can distinguish between these relations even if they are structurally similar, thanks to its random-walk sampling which introduce probabilistic equivariance on nodes and relations. As a results, $\flock$ can produce distinct embeddings to \texttt{fatherOf} and \texttt{sonOf} and achieves much stronger zero-shot performance on \textsc{Metafam}.

\section{Analysis of KG sparsity and performance}\label{app:sparsity}

\begin{figure}[t]
    \centering
    \includegraphics[width=0.6\textwidth]{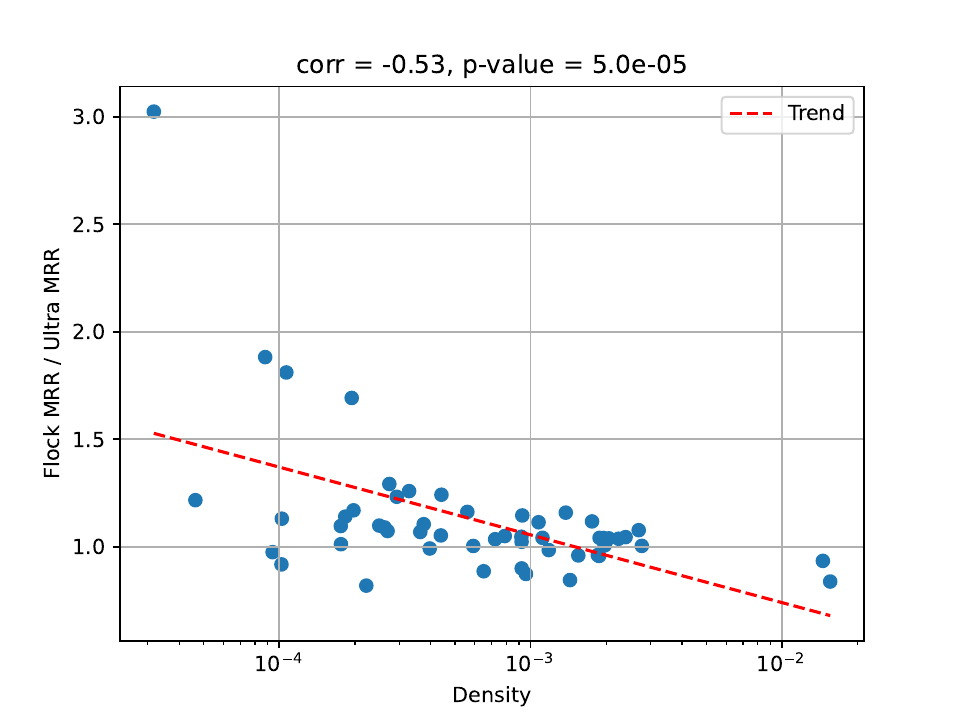}
    \caption{{The zero-shot entity prediction performance of $\flock$ relative to $\ultra$, plotted against the densities of the 53 KGs. Performance of $\flock$ and log-density of KGs have a Pearson correlation coefficient of -0.53 with p-value 5.0e-5, showing a statistically significant negative correlation.}}
    \label{fig:density}
\end{figure}

\paragraph{Setup.}

While \Cref{app:case_study} explains $\flock$'s high performance for \textsc{Metafam}, understanding its performances for other general KGs would be beneficial.
Thus, we present an additional analysis on the 53 remaining KGs by identifying a structural property that is indicative of the performance of $\flock$.
For the performance measure, we use the relative gain of $\flock$'s zero-shot entity prediction MRR compared to $\ultra$.
For the structural property, we focus on density of KGs defined by $\frac{|E|}{|V|(|V|-1)}$ which affects the speed of random walks traversing all edges of a KG (intuitively because more edges means more time needed to traverse all of them), and hence is relevant in the context of our theory in \Cref{sec:theoretical_properties}.
We make the argument more grounded below.

\paragraph{Result.}

\Cref{fig:density} shows that $\flock$ tends to perform well on sparse KGs, while less so on dense KGs, and the tendency is statistically significant. Interestingly, this agrees with our theoretical analysis in \Cref{sec:theoretical_properties}, in which a necessary condition for universality is that the random walk covers all edges of a KG with a high probability (Proposition~\ref{prop:universal_approximator}). The time taken until covering all edges is called the edge cover time, and it is known to be e.g., $O(|V| |E|)$ for uniform walks \citep{zuckerman1991on}, which is proportional to the density of a graph. This suggests the performance of $\flock$ is associated with the easiness to visit as many edges as possible rapidly, which is more challenging for dense KGs. This analysis is consistent with the observations of recent works on graph learning based on random walks, e.g., \citet[Section 6]{wang2024rum} and \citet{kim_rwnn}.

\section{Further experimental details}
\label{app:experimental_details}

\paragraph{Datasets.}
This section provides the full details for all experiments described in the main text. For pretraining, we fit the $\flock$ model on three standard transductive knowledge graph completion benchmarks, following \citet{galkin2023ultra}: FB15k-237~\citep{FB15k237}, WN18RR~\citep{Dettmers2018FB}, and CoDEx Medium~\citep{safavi-koutra-2020-codex}.
Then, we evaluate zero-shot transfer of both entity prediction and relation prediction, as well as the finetuning performance on multiple datasets grouped as follows:

\begin{itemize}
    
\item \textbf{Inductive $e,r$.} Link prediction tasks involving previously unseen nodes and relation types. This includes the 13 datasets from \textsc{Ingram}~\citep{ingram}: FB-25, FB-50, FB-75, FB-100, WK-25, WK-50, WK-75, WK-100, NL-0, NL-25, NL-50, NL-75, NL-100, as well as 10 datasets from MTDEA~\citep{zhou2023multitaskperspetivelinkprediction}: MT1 tax, MT1 health, MT2 org, MT2 sci, MT3 art, MT3 infra, MT4 sci, MT4 health, Metafam, and FBNELL.

\item  \textbf{Inductive $e$.} Link prediction tasks involving novel nodes but fixed relation types. This category comprises 12 GraIL datasets~\citep{grail2020teru} (WN-v1 through WN-v4, FB-v1 through FB-v4, NL-v1 through NL-v4), 4 INDIGO benchmarks~\citep{INDIGO} (HM 1k, HM 3k, HM 5k, HM Indigo), and 2 NodePiece datasets~\citep{galkin2022nodepiece}: ILPC Small and ILPC Large.

\item  \textbf{Transductive.} Link prediction tasks where both entities and relations are observed during training. These include CoDEx Small, CoDEx Large~\citep{safavi-koutra-2020-codex}, NELL-995~\citep{nell995WenhanXiongDeepPath}, YAGO 310~\citep{Mahdisoltani2015YAGO3AK}, WDsinger, NELL23k, FB15k-237(10), FB15k-237(20), FB15k-237(50)~\citep{FB15k237-10-20-50}, AristoV4~\citep{chenrelation}, DBpedia100k~\citep{ding2018improving}, ConceptNet100k~\citep{malaviya2020commonsense}, and Hetionet~\citep{hetionet}.

\end{itemize}

\paragraph{Full results of \Cref{sec:main_experiment}.}
Full tables of zero-shot entity prediction results are presented in \Cref{tab:zeroshot-entity}, and full tables of finetuned performance are given in \Cref{tab:finetuned-entity}.
We further provide the complete zero-shot and finetuned relation prediction results in \Cref{tab:zeroshot-relation} and \Cref{tab:finetuned-relation}.
The dataset statistics are in \Cref{tab:inductive-e-r-statistics}, \Cref{tab:inductive-e-statistics} and \Cref{tab:transductive-statistics}.
\Cref{tab:pretrain-mix-table} presents the pretraining graph mix shown in \Cref{sec:scaling_main}.
Finally, detailed hyperparameter settings can be found in \Cref{tab:hyperparameter}, \Cref{tab:hyperparameter-finetune-1} and \Cref{tab:hyperparameter-finetune-2}.

\paragraph{Training.}  Following conventions in the literature~\citep{zhu2022neural,huang2023theory}, for each triple $(h,r,t)$, we add the corresponding inverse triple $(h,r^{-1},t)$, where $r^{-1}$ is a fresh relation symbol.
All $\flock$ instances and its variants are optimized to minimize the negative log-likelihood over positive and negative facts under the \emph{partial completeness assumption}~\citep{partial_completeness_assumption}, where
negatives are generated by randomly corrupting either the head or the tail entity (for entity prediction)
or by corrupting the relation (for relation prediction).
To reduce overfitting, we remove edges that directly connect the queried endpoints.
The best checkpoint is selected by validation performance.
For entity prediction, we take the embedding for potential target $t$ and relations $r$, and obtain the score $p(h,r,t)$ by passing into a 2-layer MLP. For relation prediction, we concatenate the embedding for source $h$, target $t$, and potential relation $r$ to obtain the score $p(h,r,t)$. 

Let \((h,r,t)\) be a positive triple and let \(k\) denote the number of negatives sampled per positive, where $(h_i,r,t_i)$ is the $i$-th negative samples for entity prediction, and $h,r_i,t_i$ is the $i$-th negative samples for relation prediction. 
Following \citet{sun2019rotate}, we also consider a self-adversarial variant where negatives are reweighted
according to their current difficulty. With adversarial temperature \(\alpha>0\),
the weights for entity and relation prediction, respectively, are
\[
w^{\mathrm{ent}}_{i,\alpha}
=\operatorname{Softmax}\!\left(\frac{\log\!\big(1-p(h_i',r,t_i')\big)}{\alpha}\right),\qquad
w^{\mathrm{rel}}_{i,\alpha}
=\operatorname{Softmax}\!\left(\frac{\log\!\big(1-p(h,r_i',t)\big)}{\alpha}\right).
\]
The corresponding losses become
\begin{align*}
\mathcal{L}_{\mathrm{ent}}^{\mathrm{adv}}
&= -\log p(h,r,t)\;-\;\sum_{i=1}^{k} w^{\mathrm{ent}}_{i,\alpha}\,\log\!\big(1-p(h_i',r,t_i')\big), \\
\mathcal{L}_{\mathrm{rel}}^{\mathrm{adv}}
&= -\log p(h,r,t)\;-\;\sum_{i=1}^{k} w^{\mathrm{rel}}_{i,\alpha}\,\log\!\big(1-p(h,r_i',t)\big).
\end{align*}

\begin{table}[p]
\caption{Zero-shot entity prediction results.
Bold indicates the best score per row.}
\label{tab:zeroshot-entity}
\centering
\footnotesize
\begin{adjustbox}{max height=0.48\textheight}
\begin{tabular}{lcccc|cc}
\toprule
\multirow{2}{*}{Dataset} 
& \multicolumn{2}{c}{$\ultra$} 
& \multicolumn{2}{c}{$\trix$} 
& \multicolumn{2}{c}{$\flock$} \\
\cmidrule(lr){2-3}\cmidrule(lr){4-5}\cmidrule(lr){6-7}
& MRR & Hits@10 & MRR & Hits@10 & MRR & Hits@10 \\
\midrule 
\multicolumn{7}{c}{\textbf{Inductive} $e,r$} \\ 
\midrule
FB-25   & 0.388 & 0.640 & 0.393 & 0.650 & \textbf{0.404} & \textbf{0.664} \\
FB-50   & 0.338 & 0.543 & 0.334 & 0.547 & \textbf{0.352} & \textbf{0.566} \\
FB-75   & 0.403 & 0.604 & 0.401 & 0.611 & \textbf{0.418} & \textbf{0.622} \\
FB-100  & 0.449 & 0.642 & 0.436 & 0.635 & \textbf{0.452} & \textbf{0.663} \\
WK-25   & \textbf{0.316} & \textbf{0.532} & 0.305 & 0.496 & 0.280 & 0.491 \\
WK-50   & \textbf{0.166} & \textbf{0.324} & \textbf{0.166} & 0.313 & 0.136 & 0.278 \\
WK-75   & 0.365 & 0.537 & 0.368 & 0.513 & 0.382 & \textbf{0.538} \\
WK-100  & 0.164 & 0.286 & \textbf{0.188} & 0.299 & 0.187 & \textbf{0.304} \\
NL-0    & 0.342 & 0.523 & \textbf{0.385} & 0.549 & 0.381 & \textbf{0.606} \\
NL-25   & \textbf{0.395} & 0.569 & 0.377 & 0.589 & 0.345 & \textbf{0.590} \\
NL-50   & \textbf{0.407} & \textbf{0.570} & 0.404 & 0.548 & 0.366 & 0.565 \\
NL-75   & \textbf{0.368} & \textbf{0.547} & 0.351 & 0.525 & 0.311 & 0.524 \\
NL-100  & 0.471 & 0.651 & \textbf{0.486} & 0.676 & 0.452 & \textbf{0.692}\\
MT1 tax    & 0.224 & 0.305 & \textbf{0.358} & 0.452 & 0.282 & \textbf{0.383} \\
MT1 health & 0.298 & 0.374 & 0.376 & 0.457 & \textbf{0.385} & \textbf{0.481} \\
MT2 org    & 0.095 & 0.159 & 0.091 & 0.156 & \textbf{0.100} & \textbf{0.163} \\
MT2 sci    & 0.258 & 0.354 & \textbf{0.323} & \textbf{0.465} & 0.318 & 0.458 \\
MT3 art    & 0.259 & 0.402 & 0.284 & 0.441 & \textbf{0.301} & \textbf{0.466} \\
MT3 infra  & 0.619 & 0.755 & 0.655 & 0.797 & \textbf{0.684} & \textbf{0.821} \\
MT4 sci    & 0.274 & 0.449 & 0.290 & 0.460 & \textbf{0.301} & \textbf{0.463} \\
MT4 health & 0.624 & 0.737 & 0.677 & 0.775 & \textbf{0.680} & \textbf{0.780} \\
Metafam    & 0.238 & 0.644 & 0.341 & 0.815 & \textbf{0.476} & \textbf{0.935} \\
FBNELL     & 0.485 & 0.652 & 0.473 & 0.660 & \textbf{0.502} & \textbf{0.700} \\
\midrule 
\multicolumn{7}{c}{\textbf{Inductive} $e$} \\ 
\midrule
FB-v1     & 0.498 & 0.656 & \textbf{0.515} & 0.682 & 0.500 & \textbf{0.697} \\
FB-v2     & 0.512 & 0.700 & 0.525 & 0.730 & \textbf{0.535} & \textbf{0.737} \\
FB-v3     & 0.491 & 0.654 & 0.501 & 0.669 & \textbf{0.511} & \textbf{0.685} \\
FB-v4     & 0.486 & 0.677 & 0.493 & 0.687 & \textbf{0.505} & \textbf{0.702} \\
WN-v1     & 0.648 & 0.768 & \textbf{0.699} & 0.791 & 0.698 & \textbf{0.803} \\
WN-v2     & 0.663 & 0.765 & 0.678 & 0.781 & \textbf{0.696} & \textbf{0.790} \\
WN-v3     & 0.376 & 0.476 & 0.418 & 0.541 & \textbf{0.467} & \textbf{0.608} \\
WN-v4     & 0.611 & 0.705 & 0.648 & 0.723 & \textbf{0.653} & \textbf{0.729} \\
NL-v1     & 0.785 & \textbf{0.913} & \textbf{0.806} & 0.898 & 0.658 & 0.863 \\
NL-v2     & 0.526 & 0.707 & 0.569 & 0.768 & \textbf{0.588} & \textbf{0.797} \\
NL-v3     & 0.515 & 0.702 & 0.558 & 0.743 & \textbf{0.590} & \textbf{0.783} \\
NL-v4     & 0.479 & 0.712 & 0.538 & 0.765 & \textbf{0.555} & \textbf{0.786} \\
HM 1k     & 0.059 & 0.092 & \textbf{0.072} & \textbf{0.128} & 0.069 & 0.119 \\
HM 3k     & 0.037 & 0.077 & \textbf{0.069} & \textbf{0.119} & 0.067 & 0.118 \\
HM 5k     & 0.034 & 0.071 & 0.062 & 0.110 & \textbf{0.064} & \textbf{0.116} \\
HM Indigo & \textbf{0.440} & \textbf{0.648} & 0.436 & 0.645 & 0.423 & 0.638 \\
ILPC Small& 0.302 & 0.443 & 0.303 & 0.455 & \textbf{0.309} & \textbf{0.459} \\
ILPC Large& 0.290 & 0.424 & 0.307 & 0.428 & \textbf{0.318} & \textbf{0.438} \\
\midrule
\multicolumn{7}{c}{\textbf{Transductive}} \\ 
\midrule
NELL995         & 0.406 & 0.543 & 0.472 & 0.629 & \textbf{0.494} & \textbf{0.655} \\
NELL23k         & 0.239 & 0.408 & \textbf{0.290} & \textbf{0.497} & 0.233 & 0.398 \\
WDsinger        & 0.382 & 0.498 & \textbf{0.511} & \textbf{0.609} & 0.410 & 0.528 \\
ConceptNet100k  & 0.082 & 0.162 & 0.193 & 0.345 & \textbf{0.248} & \textbf{0.453} \\
CoDEx Small     & \textbf{0.472} & 0.667 & \textbf{0.472} & \textbf{0.670} & 0.441 & 0.644 \\
CoDEx Large     & 0.338 & \textbf{0.469} & 0.335 & \textbf{0.469} & \textbf{0.342} & 0.464 \\
YAGO310         & \textbf{0.451} & 0.615 & 0.409 & 0.627 & 0.414 & \textbf{0.674} \\
AristoV4        & 0.182 & 0.282 & 0.181 & 0.286 & \textbf{0.308} & \textbf{0.443} \\
DBpedia100k     & 0.398 & 0.576 & 0.426 & 0.603 & \textbf{0.450} & \textbf{0.627} \\
Hetionet        & 0.257 & 0.379 & \textbf{0.279} & \textbf{0.420} & 0.246 & 0.371 \\
FB15k-237(10) & \textbf{0.248} & 0.398 & 0.246 & 0.393 & 0.246 & \textbf{0.402} \\
FB15k-237(20) & 0.272 & 0.436 & 0.269 & 0.430 & \textbf{0.273} & \textbf{0.444} \\
FB15k-237(50) & \textbf{0.324} & \textbf{0.526} & 0.321 & 0.521 & 0.319 & 0.518 \\
\bottomrule
\end{tabular}
\end{adjustbox}
\end{table}

\begin{table}[p]
  \caption{Finetuned entity prediction results. Bold indicates the best score per row.}
  \label{tab:finetuned-entity}
  \centering
  \footnotesize
  \begin{adjustbox}{max height=0.48\textheight}
  \begin{tabular}{lcccc|cc}
    \toprule
    \multirow{2}{*}{Dataset} & \multicolumn{2}{c}{$\ultra$} & \multicolumn{2}{c}{$\trix$} & \multicolumn{2}{c}{$\flock$} \\
    \cmidrule(lr){2-3}\cmidrule(lr){4-5}\cmidrule(lr){6-7}
     & MRR & Hits@10 & MRR & Hits@10 & MRR & Hits@10 \\
    \midrule
    \multicolumn{7}{c}{\textbf{Inductive} $e,r$} \\
    \midrule
    FB-25      & 0.383 & 0.635 & 0.393 & 0.650 & \textbf{0.405} & \textbf{0.666} \\
    FB-50      & 0.334 & 0.538 & 0.334 & 0.547 & \textbf{0.357} & \textbf{0.570} \\
    FB-75      & 0.400 & 0.598 & 0.401 & 0.611 & \textbf{0.425} & \textbf{0.630} \\
    FB-100     & 0.444 & 0.643 & 0.436 & 0.633 & \textbf{0.460} & \textbf{0.668} \\
    WK-25      & \textbf{0.321} & \textbf{0.535} & 0.300 & 0.493 & 0.298 & 0.506 \\
    WK-50      & 0.140 & 0.280 & \textbf{0.166} & \textbf{0.313} & 0.127 & 0.260 \\
    WK-75      & 0.380 & 0.530 & 0.368 & 0.513 & \textbf{0.405} & \textbf{0.556} \\
    WK-100     & 0.168 & 0.286 & \textbf{0.188} & 0.299 & 0.187 & \textbf{0.306} \\
    NL-0       & 0.329 & 0.551 & 0.385 & 0.549 & \textbf{0.418} & \textbf{0.619} \\
    NL-25      & \textbf{0.407} & 0.596 & 0.377 & 0.589 & 0.405 & \textbf{0.626} \\
    NL-50      & \textbf{0.418} & \textbf{0.595} & 0.405 & 0.555 & 0.391 & 0.562 \\
    NL-75      & \textbf{0.374} & \textbf{0.570} & 0.351 & 0.525 & 0.344 & 0.544 \\
    NL-100     & 0.458 & 0.684 & 0.482 & 0.691 & \textbf{0.486} & \textbf{0.714} \\
    MT1 tax    & 0.330 & 0.459 & 0.397 & \textbf{0.508} & \textbf{0.413} & 0.497 \\
    MT1 health & 0.380 & 0.467 & 0.376 & 0.457 & \textbf{0.394} & \textbf{0.493} \\
    MT2 org    & 0.104 & 0.170 & 0.098 & 0.162 & \textbf{0.107} & \textbf{0.174} \\
    MT2 sci    & 0.311 & 0.451 & 0.331 & \textbf{0.526} & \textbf{0.366} & 0.525 \\
    MT3 art    & 0.306 & 0.473 & 0.289 & 0.441 & \textbf{0.330} & \textbf{0.483} \\
    MT3 infra  & 0.657 & 0.807 & 0.672 & 0.810 & \textbf{0.709} & \textbf{0.838} \\
    MT4 sci    & 0.303 & 0.478 & 0.305 & 0.482 & \textbf{0.324} & \textbf{0.509} \\
    MT4 health & 0.704 & 0.785 & 0.702 & 0.785 & \textbf{0.711} & \textbf{0.790} \\
    Metafam    & \textbf{0.997} & \textbf{1.000} & \textbf{0.997} & \textbf{1.000} & 0.992 & \textbf{1.000} \\
    FBNELL     & 0.481 & 0.661 & 0.478 & 0.655 & \textbf{0.531} & \textbf{0.714} \\
    \midrule
    \multicolumn{7}{c}{\textbf{Inductive} $e$} \\
    \midrule
    FB-v1      & 0.509 & 0.670 & 0.515 & 0.682 & \textbf{0.549} & \textbf{0.721} \\
    FB-v2      & 0.524 & 0.710 & 0.525 & 0.730 & \textbf{0.553} & \textbf{0.754} \\
    FB-v3      & 0.504 & 0.663 & 0.501 & 0.669 & \textbf{0.528} & \textbf{0.696} \\
    FB-v4      & 0.496 & 0.684 & 0.493 & 0.687 & \textbf{0.510} & \textbf{0.702} \\
    WN-v1      & 0.685 & 0.793 & 0.705 & 0.798 & \textbf{0.715} & \textbf{0.811} \\
    WN-v2      & 0.679 & 0.779 & 0.682 & 0.780 & \textbf{0.702} & \textbf{0.795} \\
    WN-v3      & 0.411 & 0.546 & 0.425 & 0.543 & \textbf{0.494} & \textbf{0.627} \\
    WN-v4      & 0.614 & 0.720 & 0.650 & 0.722 & \textbf{0.665} & \textbf{0.741} \\
    NL-v1      & 0.757 & 0.878 & \textbf{0.804} & 0.899 & 0.762 & \textbf{0.928} \\
    NL-v2      & 0.575 & 0.761 & 0.571 & 0.764 & \textbf{0.612} & \textbf{0.806} \\
    NL-v3      & 0.563 & 0.755 & 0.571 & 0.759 & \textbf{0.606} & \textbf{0.803} \\
    NL-v4      & 0.469 & 0.733 & 0.551 & 0.772 & \textbf{0.572} & \textbf{0.801} \\
    HM 1k      & 0.042 & 0.100 & \textbf{0.072} & 0.128 & 0.071 & \textbf{0.153} \\
    HM 3k      & 0.030 & 0.090 & \textbf{0.069} & 0.119 & 0.067 & \textbf{0.153} \\
    HM 5k      & 0.025 & 0.068 & \textbf{0.074} & 0.118 & 0.061 & \textbf{0.130} \\
    HM Indigo  & 0.432 & 0.639 & \textbf{0.436} & \textbf{0.645} & 0.418 & 0.633 \\
    ILPC Small & 0.303 & 0.453 & 0.303 & \textbf{0.455} & \textbf{0.305} & 0.454 \\
    ILPC Large & 0.308 & \textbf{0.431} & 0.310 & \textbf{0.431} & \textbf{0.320} & 0.441 \\
    \midrule
    \multicolumn{7}{c}{\textbf{Transductive}} \\ 
    \midrule
    NELL995        & 0.509 & 0.660 & 0.506 & 0.648 & \textbf{0.531} & \textbf{0.665} \\
    NELL23k        & 0.268 & 0.450 & \textbf{0.306} & \textbf{0.536} & 0.280 & 0.465 \\
    WDsinger       & 0.417 & 0.526 & \textbf{0.502} & \textbf{0.620} & 0.435 & 0.543 \\
    ConceptNet100k & 0.310 & 0.529 & 0.340 & 0.564 & \textbf{0.352} & \textbf{0.580} \\
    CoDEx Small    & \textbf{0.490} & \textbf{0.686} & 0.484 & 0.676 & 0.463 & 0.648 \\
    CoDEx Large    & 0.343 & 0.478 & \textbf{0.348} & \textbf{0.481} & 0.342 & 0.467 \\
    YAGO310        & \textbf{0.557} & \textbf{0.710} & 0.541 & 0.702 & 0.552 & 0.700 \\
    AristoV4       & 0.343 & 0.496 & 0.345 & 0.499 & \textbf{0.383} & \textbf{0.523} \\
    DBpedia100k    & 0.436 & 0.603 & 0.457 & 0.619 & \textbf{0.470} & \textbf{0.623} \\
    Hetionet       & \textbf{0.399} & \textbf{0.538} & 0.394 & 0.534 & 0.314 & 0.465 \\
    FB15k-237(10)   & 0.254 & 0.411 & 0.253 & 0.408 & \textbf{0.260} & \textbf{0.420} \\
    FB15k-237(20)   & 0.274 & 0.445 & 0.273 & 0.441 & \textbf{0.284} & \textbf{0.459} \\
    FB15k-237(50)   & \textbf{0.325} & \textbf{0.528} & 0.322 & 0.522 & 0.317 & 0.517 \\
    \midrule
    \multicolumn{7}{c}{\textbf{Pretrained}} \\ 
    \midrule
    FB15k-237 & \textbf{0.368}	 & \textbf{0.564} &	0.366	& 0.559 & 0.343	&	0.532\\
    WN18RR & 0.480	& 0.614	& 0.514	& 0.611  & \textbf{0.550} &	\textbf{0.656}\\ 
    CoDEx Medium & \textbf{0.372}  & \textbf{0.525} &	0.365 &	0.521  & 0.351 &	0.496\\
    \bottomrule
  \end{tabular}
  \end{adjustbox}
\end{table}

\begin{table}[p]
\caption{Zero-shot relation prediction results. Bold indicates the best score per row.}
\label{tab:zeroshot-relation}
\centering
\footnotesize
\begin{adjustbox}{max height=0.48\textheight}
\begin{tabular}{lcccc|cc}
\toprule
\multirow{2}{*}{Dataset} & \multicolumn{2}{c}{$\ultra$} & \multicolumn{2}{c}{$\trix$} & \multicolumn{2}{c}{$\flock$} \\
\cmidrule(lr){2-3}\cmidrule(lr){4-5}\cmidrule(lr){6-7}
& MRR & Hits@1 & MRR & Hits@1 & MRR & Hits@1 \\
\midrule
\multicolumn{7}{c}{\textbf{Inductive} $e,r$} \\ 
\midrule
FB-25        & 0.687 & 0.565 & 0.805 & 0.724 & \textbf{0.895} & \textbf{0.839} \\
FB-50        & 0.696 & 0.575 & 0.780 & 0.699 & \textbf{0.880} & \textbf{0.820} \\
FB-75        & 0.698 & 0.555 & 0.822 & 0.747 & \textbf{0.903} & \textbf{0.844} \\
FB-100       & 0.830 & 0.728 & 0.921 & 0.880 & \textbf{0.962} & \textbf{0.938} \\
WK-25        & 0.857 & 0.760 & 0.881 & 0.823 & \textbf{0.952} & \textbf{0.929} \\
WK-50        & 0.865 & 0.793 & 0.868 & 0.818 & \textbf{0.921} & \textbf{0.882} \\
WK-75        & 0.911 & 0.875 & 0.916 & 0.883 & \textbf{0.962} & \textbf{0.944} \\
WK-100       & 0.887 & 0.812 & 0.907 & 0.869 & \textbf{0.963} & \textbf{0.937} \\
NL-0         & 0.632 & 0.502 & 0.658 & 0.519 & \textbf{0.714} & \textbf{0.574} \\
NL-25        & 0.688 & 0.562 & \textbf{0.742} & 0.614 & 0.729 & \textbf{0.632} \\
NL-50        & 0.680 & 0.569 & 0.755 & 0.636 & \textbf{0.813} & \textbf{0.728} \\
NL-75        & 0.795 & 0.692 & 0.788 & 0.699 & \textbf{0.833} & \textbf{0.756} \\
NL-100       & 0.743 & 0.564 & 0.884 & 0.796 & \textbf{0.939} & \textbf{0.889} \\
MT1 tax      & 0.985 & 0.976 & 0.975 & 0.958 & \textbf{0.998} & \textbf{0.997} \\
MT1 health   & 0.721 & 0.561 & 0.973 & 0.949 & \textbf{0.991} & \textbf{0.983} \\
MT2 org      & 0.974 & 0.951 & 0.986 & 0.973 & \textbf{0.991} & \textbf{0.984} \\
MT2 sci      & 0.976 & 0.961 & 0.964 & 0.941 & \textbf{0.995} & \textbf{0.992} \\
MT3 art      & 0.881 & 0.798 & 0.885 & 0.825 & \textbf{0.944} & \textbf{0.907} \\
MT3 infra    & 0.962 & 0.935 & 0.940 & 0.905 & \textbf{0.989} & \textbf{0.980} \\
MT4 sci      & 0.933 & 0.891 & 0.966 & 0.944 & \textbf{0.974} & \textbf{0.957} \\
MT4 health   & 0.826 & 0.719 & 0.937 & 0.898 & \textbf{0.990} & \textbf{0.983} \\
Metafam      & 0.124 & 0.000 & 0.291 & 0.011 & \textbf{0.490} & \textbf{0.223} \\
FBNELL       & 0.700 & 0.564 & 0.726 & 0.605 & \textbf{0.833} & \textbf{0.737} \\
\midrule 
\multicolumn{7}{c}{\textbf{Inductive} $e$} \\ 
\midrule
FB-v1        & 0.646 & 0.523 & 0.705 & 0.599 & \textbf{0.814} & \textbf{0.723} \\
FB-v2        & 0.695 & 0.570 & 0.713 & 0.590 & \textbf{0.847} & \textbf{0.761} \\
FB-v3        & 0.679 & 0.553 & 0.742 & 0.644 & \textbf{0.860} & \textbf{0.780} \\
FB-v4        & 0.638 & 0.488 & 0.766 & 0.665 & \textbf{0.873} & \textbf{0.799} \\
WN-v1        & 0.836 & 0.740 & 0.792 & 0.613 & \textbf{0.924} & \textbf{0.858} \\
WN-v2        & 0.853 & 0.790 & 0.764 & 0.572 & \textbf{0.924} & \textbf{0.863} \\
WN-v3        & 0.707 & 0.577 & 0.741 & 0.568 & \textbf{0.937} & \textbf{0.888} \\
WN-v4        & 0.860 & 0.803 & 0.764 & 0.570 & \textbf{0.937} & \textbf{0.886} \\
NL-v1        & 0.636 & 0.358 & 0.657 & 0.453 & \textbf{0.862} & \textbf{0.731} \\
NL-v2        & 0.742 & 0.652 & 0.780 & 0.696 & \textbf{0.893} & \textbf{0.855} \\
NL-v3        & 0.669 & 0.544 & 0.725 & 0.612 & \textbf{0.815} & \textbf{0.731} \\
NL-v4        & 0.606 & 0.489 & 0.794 & 0.691 & \textbf{0.868} & \textbf{0.807} \\
ILPC Small   & 0.905 & 0.843 & 0.919 & 0.872 & \textbf{0.955} & \textbf{0.921} \\
ILPC Large   & 0.875 & 0.799 & 0.894 & 0.829 & \textbf{0.948} & \textbf{0.908} \\
HM 1k        & 0.626 & 0.447 & 0.663 & 0.414 & \textbf{0.687} & \textbf{0.500} \\
HM 3k        & 0.592 & 0.439 & 0.664 & 0.418 & \textbf{0.714} & \textbf{0.549} \\
HM 5k        & 0.605 & 0.452 & 0.672 & 0.428 & \textbf{0.746} & \textbf{0.593} \\
HM Indigo    & 0.681 & 0.559 & 0.852 & 0.765 & \textbf{0.956} & \textbf{0.921} \\
\midrule
\multicolumn{7}{c}{\textbf{Transductive}} \\ 
\midrule
NELL995        & 0.583 & 0.437 & 0.578 & 0.457 & \textbf{0.684} & \textbf{0.555} \\
NELL23k        & 0.669 & 0.548 & 0.756 & 0.657 & \textbf{0.831} & \textbf{0.762} \\
WDsinger       & 0.668 & 0.546 & 0.720 & 0.621 & \textbf{0.823} & \textbf{0.738} \\
ConceptNet100k & 0.181 & 0.083 & 0.650 & 0.469 & \textbf{0.795} & \textbf{0.658} \\
CoDExSmall     & 0.900 & 0.820 & 0.961 & 0.935 & \textbf{0.982} & \textbf{0.970} \\
CoDExLarge     & 0.892 & 0.824 & 0.902 & 0.837 & \textbf{0.973} & \textbf{0.950} \\
YAGO310        & 0.646 & 0.403 & 0.783 & 0.598 & \textbf{0.971} & \textbf{0.943} \\
AristoV4       & 0.254 & 0.201 & 0.389 & 0.265 & \textbf{0.597} & \textbf{0.496} \\
DBpedia100k    & 0.650 & 0.509 & 0.717 & 0.582 & \textbf{0.919} & \textbf{0.861} \\
Hetionet       & 0.634 & 0.524 & 0.809 & 0.707 & \textbf{0.940} & \textbf{0.890} \\
FB15k-237(10)   & 0.688 & 0.550 & 0.795 & 0.711 & \textbf{0.918} & \textbf{0.876} \\
FB15k-237(20)   & 0.695 & 0.558 & 0.834 & 0.758 & \textbf{0.952} & \textbf{0.923} \\
FB15k-237(50)   & 0.717 & 0.591 & 0.876 & 0.812 & \textbf{0.968} & \textbf{0.946}\\
\bottomrule
\end{tabular}
\end{adjustbox}
\end{table}

\begin{table}[p]
\caption{Finetuned relation prediction results. Bold indicates the best score per row.}
\footnotesize
\label{tab:finetuned-relation}
\centering
\begin{adjustbox}{max height=0.48\textheight}
\begin{tabular}{lcccc|cc}
\toprule
\multirow{2}{*}{Dataset} & \multicolumn{2}{c}{$\ultra$} & \multicolumn{2}{c}{$\trix$} & \multicolumn{2}{c}{$\flock$} \\
\cmidrule(lr){2-3}\cmidrule(lr){4-5}\cmidrule(lr){6-7}
& MRR & Hits@1 & MRR & Hits@1 & MRR & Hits@1 \\
\midrule
\multicolumn{7}{c}{\textbf{Inductive} $e,r$} \\ 
\midrule
FB-25        & 0.684 & 0.563 & 0.805 & 0.724 & \textbf{0.909} & \textbf{0.857} \\
FB-50        & 0.696 & 0.575 & 0.780 & 0.699 & \textbf{0.881} & \textbf{0.820} \\
FB-75        & 0.754 & 0.638 & 0.822 & 0.699 & \textbf{0.911} & \textbf{0.854} \\
FB-100       & 0.851 & 0.769 & 0.921 & 0.880 & \textbf{0.965} & \textbf{0.939} \\
WK-25        & 0.897 & 0.834 & 0.905 & 0.860 & \textbf{0.968} & \textbf{0.954} \\
WK-50        & 0.865 & 0.793 & 0.881 & 0.840 & \textbf{0.925} & \textbf{0.876} \\
WK-75        & 0.911 & 0.875 & 0.937 & 0.910 & \textbf{0.965} & \textbf{0.948} \\
WK-100       & 0.924 & 0.879 & 0.916 & 0.885 & \textbf{0.970} & \textbf{0.946} \\
NL-0         & 0.632 & 0.502 & 0.655 & 0.518 & \textbf{0.731} & \textbf{0.602} \\
NL-25        & 0.737 & 0.622 & 0.709 & 0.606 & \textbf{0.757} & \textbf{0.634} \\
NL-50        & 0.808 & 0.704 & 0.774 & 0.683 & \textbf{0.814} & \textbf{0.721} \\
NL-75        & 0.795 & 0.678 & 0.790 & 0.671 & \textbf{0.848} & \textbf{0.774} \\
NL-100       & 0.803 & 0.678 & 0.885 & 0.793 & \textbf{0.937} & \textbf{0.887} \\
MT1 tax      & 0.990 & 0.984 & 0.995 & 0.990 & \textbf{0.999} & \textbf{0.998} \\
MT1 health   & 0.929 & 0.867 & 0.973 & 0.949 & \textbf{0.994} & \textbf{0.988} \\
MT2 org      & 0.981 & 0.963 & 0.987 & 0.978 & \textbf{0.994} & \textbf{0.988} \\
MT2 sci      & 0.977 & 0.961 & 0.990 & 0.984 & \textbf{0.995} & \textbf{0.992} \\
MT3 art      & 0.907 & 0.851 & 0.887 & 0.828 & \textbf{0.950} & \textbf{0.916} \\
MT3 infra    & 0.966 & 0.947 & 0.970 & 0.952 & \textbf{0.996} & \textbf{0.993} \\
MT4 sci      & 0.954 & 0.929 & 0.972 & 0.952 & \textbf{0.983} & \textbf{0.968} \\
MT4 health   & 0.951 & 0.919 & 0.986 & 0.979 & \textbf{0.995} & \textbf{0.991} \\
Metafam      & 0.368 & 0.112 & 0.265 & 0.024 & \textbf{0.997} & \textbf{0.995} \\
FBNELL       & 0.720 & 0.576 & 0.766 & 0.639 & \textbf{0.879} & \textbf{0.801} \\
\midrule 
\multicolumn{7}{c}{\textbf{Inductive} $e$} \\ 
\midrule
FB-v1        & 0.650 & 0.513 & 0.705 & 0.599 & \textbf{0.855} & \textbf{0.766} \\
FB-v2        & 0.675 & 0.547 & 0.713 & 0.590 & \textbf{0.887} & \textbf{0.812} \\
FB-v3        & 0.677 & 0.556 & 0.742 & 0.644 & \textbf{0.879} & \textbf{0.810} \\
FB-v4        & 0.690 & 0.560 & 0.766 & 0.665 & \textbf{0.884} & \textbf{0.807} \\
WN-v1        & 0.844 & 0.754 & 0.776 & 0.591 & \textbf{0.926} & \textbf{0.879} \\
WN-v2        & 0.834 & 0.766 & 0.765 & 0.574 & \textbf{0.927} & \textbf{0.869} \\
WN-v3        & 0.707 & 0.577 & 0.756 & 0.594 & \textbf{0.950} & \textbf{0.911} \\
WN-v4        & 0.861 & 0.795 & 0.804 & 0.651 & \textbf{0.943} & \textbf{0.898} \\
NL-v1        & 0.719 & 0.504 & 0.590 & 0.341 & \textbf{0.883} & \textbf{0.766} \\
NL-v2        & 0.668 & 0.549 & 0.811 & 0.740 & \textbf{0.911} & \textbf{0.870} \\
NL-v3        & 0.646 & 0.484 & 0.757 & 0.643 & \textbf{0.868} & \textbf{0.795} \\
NL-v4        & 0.570 & 0.412 & 0.822 & 0.735 & \textbf{0.906} & \textbf{0.849} \\
ILPC Small   & 0.922 & 0.876 & 0.919 & 0.872 & \textbf{0.953} & \textbf{0.918} \\
ILPC Large   & 0.875 & 0.799 & 0.894 & 0.829 & \textbf{0.953} & \textbf{0.915} \\
HM 1k        & 0.626 & 0.447 & 0.663 & 0.414 & \textbf{0.756} & \textbf{0.561} \\
HM 3k        & 0.592 & 0.439 & 0.664 & 0.418 & \textbf{0.790} & \textbf{0.623} \\
HM 5k        & 0.605 & 0.452 & 0.672 & 0.428 & \textbf{0.744} & \textbf{0.591} \\
HM Indigo    & 0.726 & 0.614 & 0.835 & 0.746 & \textbf{0.946} & \textbf{0.903} \\
\midrule 
\multicolumn{7}{c}{\textbf{Transductive}} \\ 
\midrule
NELL995         & 0.630 & 0.513 & 0.578 & 0.457 & \textbf{0.713} & \textbf{0.584} \\
NELL23k         & 0.688 & 0.571 & 0.755 & 0.658 & \textbf{0.869} & \textbf{0.805} \\
WDsinger        & 0.730 & 0.603 & 0.721 & 0.627 & \textbf{0.885} & \textbf{0.815} \\
ConceptNet100k  & 0.612 & 0.488 & 0.712 & 0.551 & \textbf{0.885} & \textbf{0.813} \\
CoDExSmall      & 0.942 & 0.900 & 0.964 & 0.943 & \textbf{0.981} & \textbf{0.967} \\
CoDExLarge      & 0.907 & 0.850 & 0.908 & 0.845 & \textbf{0.973} & \textbf{0.950} \\
YAGO310         & 0.930 & 0.891 & 0.826 & 0.666 & \textbf{0.970} & \textbf{0.942}\\
AristoV4        & 0.254 & 0.201 & 0.498 & 0.381 & \textbf{0.651} & \textbf{0.547} \\
DBpedia100k     & 0.650 & 0.509 & 0.780 & 0.665 & \textbf{0.923} & \textbf{0.869} \\
Hetionet        & 0.737 & 0.646 & 0.922 & 0.862 & \textbf{0.942} & \textbf{0.897} \\
FB15k-237(10)    & 0.688 & 0.550 & 0.795 & 0.711 & \textbf{0.940} & \textbf{0.905} \\
FB15k-237(20)    & 0.695 & 0.558 & 0.846 & 0.778 & \textbf{0.958} & \textbf{0.931} \\
FB15k-237(50)    & 0.728 & 0.618 & 0.903 & 0.858 & \textbf{0.970} & \textbf{0.948} \\
\midrule
\multicolumn{7}{c}{\textbf{Pretrained}} \\ 
\midrule
FB15k-237 & 0.795 &	0.709 &	0.924 &	0.870 & \textbf{0.976} &\textbf{0.957}\\
WN18RR & 0.914 & 0.871 &	0.783 &	0.634 & \textbf{0.982} &	\textbf{0.968}\\
CoDExMedium & 0.919 &	0.870 &	0.931 &	0.886 & \textbf{0.974} &	\textbf{0.952}\\
\bottomrule
\end{tabular}
\end{adjustbox}
\end{table}

\begin{table}[p]
    \scriptsize
    \centering
    \caption{Dataset statistics for \textbf{inductive}-$e,r$ link prediction datasets. Triples are the
number of edges given at training, validation, or test graphs, respectively, whereas Valid and Test denote triples to be predicted in the validation and test graphs.}
    \label{tab:inductive-e-r-statistics}
    \begin{tabular}{lccc|cccc|ccccc}
    \toprule
 \multirow{2}{*}{ \textbf{Dataset }} & \multicolumn{3}{c}{ \textbf{Training Graph} } & \multicolumn{4}{c}{ \textbf{Validation Graph} } & \multicolumn{4}{c}{ \textbf{Test Graph} }  \\
 \cmidrule{2-12}
 & \textbf{Entities} & \textbf{Rels} & \textbf{Triples} & \textbf{Entities} & \textbf{Rels} & \textbf{Triples} & \textbf{Valid} & \textbf{Entities} & \textbf{Rels} & \textbf{Triples} & \textbf{Test}  \\
 \midrule
 FB-25  & 5190 & 163 & 91571 & 4097 & 216 & 17147 & 5716 & 4097 & 216 & 17147 & 5716  \\
 FB-50  & 5190 & 153 & 85375 & 4445 & 205 & 11636 & 3879 & 4445 & 205 & 11636 & 3879  \\
 FB-75  & 4659 & 134 & 62809 & 2792 & 186 & 9316 & 3106 & 2792 & 186 & 9316 & 3106  \\
 FB-100  & 4659 & 134 & 62809 & 2624 & 77 & 6987 & 2329 & 2624 & 77 & 6987 & 2329  \\
 WK-25  & 12659 & 47 & 41873 & 3228 & 74 & 3391 & 1130 & 3228 & 74 & 3391 & 1131  \\
 WK-50  & 12022 & 72 & 82481 & 9328 & 93 & 9672 & 3224 & 9328 & 93 & 9672 & 3225  \\
 WK-75  & 6853 & 52 & 28741 & 2722 & 65 & 3430 & 1143 & 2722 & 65 & 3430 & 1144  \\
 WK-100  & 9784 & 67 & 49875 & 12136 & 37 & 13487 & 4496 & 12136 & 37 & 13487 & 4496  \\
 NL-0  & 1814 & 134 & 7796 & 2026 & 112 & 2287 & 763 & 2026 & 112 & 2287 & 763  \\
 NL-25  & 4396 & 106 & 17578 & 2146 & 120 & 2230 & 743 & 2146 & 120 & 2230 & 744  \\
 NL-50  & 4396 & 106 & 17578 & 2335 & 119 & 2576 & 859 & 2335 & 119 & 2576 & 859  \\
 NL-75  & 2607 & 96 & 11058 & 1578 & 116 & 1818 & 606 & 1578 & 116 & 1818 & 607  \\
 NL-100  & 1258 & 55 & 7832 & 1709 & 53 & 2378 & 793 & 1709 & 53 & 2378 & 793  \\
 Metafam  & 1316 & 28 & 13821 & 1316 & 28 & 13821 & 590 & 656 & 28 & 7257 & 184   \\
 FBNELL  & 4636 & 100 & 10275 & 4636 & 100 & 10275 & 1055 & 4752 & 183 & 10685 & 597   \\
 \midrule
 Wiki MT1 tax  & 10000 & 10 & 17178 & 10000 & 10 & 17178 & 1908 & 10000 & 9 & 16526 & 1834   \\
 Wiki MT1 health  & 10000 & 7 & 14371 & 10000 & 7 & 14371 & 1596 & 10000 & 7 & 14110 & 1566   \\
 Wiki MT2 org  & 10000 & 10 & 23233 & 10000 & 10 & 23233 & 2581 & 10000 & 11 & 21976 & 2441  \\
 Wiki MT2 sci  & 10000 & 16 & 16471 & 10000 & 16 & 16471 & 1830 & 10000 & 16 & 14852 & 1650  \\
 Wiki MT3 art  & 10000 & 45 & 27262 & 10000 & 45 & 27262 & 3026 & 10000 & 45 & 28023 & 3113  \\
 Wiki MT3 infra  & 10000 & 24 & 21990 & 10000 & 24 & 21990 & 2443 & 10000 & 27 & 21646 & 2405  \\
 Wiki MT4 sci  & 10000 & 42 & 12576 & 10000 & 42 & 12576 & 1397 & 10000 & 42 & 12516 & 1388  \\
 Wiki MT4 health  & 10000 & 21 & 15539 & 10000 & 21 & 15539 & 1725 & 10000 & 20 & 15337 & 1703  \\
\bottomrule
\end{tabular}
\end{table}

\begin{table}[p]
    \small
    \centering
     \caption{Dataset statistics for \textbf{inductive}-$e$ link prediction datasets. Triples are the
number of edges given at training, validation, or test graphs, respectively, whereas Valid and Test denote triples to be predicted in the validation and test graphs.}
    \label{tab:inductive-e-statistics}
    \begin{tabular}{lccc|ccc|ccc}
\toprule \multirow{2}{*}{ \textbf{Dataset }} & \multirow{2}{*}{ \textbf{Rels} } & \multicolumn{2}{c}{ \textbf{Training Graph} } & \multicolumn{3}{c}{ \textbf{Validation Graph} } & \multicolumn{3}{c}{ \textbf{Test Graph} }  \\
\cmidrule{3-10} & & \textbf{Entities} & \textbf{Triples} & \textbf{Entities} & \textbf{Triples} & \textbf{Valid} & \textbf{Entities} & \textbf{Triples} & \textbf{Test} \\
\midrule
 FB-v1  & 180 & 1594 & 4245 & 1594 & 4245 & 489 & 1093 & 1993 & 411 \\
 FB-v2 & 200 & 2608 & 9739 & 2608 & 9739 & 1166 & 1660 & 4145 & 947  \\
 FB-v3  & 215 & 3668 & 17986 & 3668 & 17986 & 2194 & 2501 & 7406 & 1731  \\
 FB-v4  & 219 & 4707 & 27203 & 4707 & 27203 & 3352 & 3051 & 11714 & 2840 \\
 WN-v1  & 9 & 2746 & 5410 & 2746 & 5410 & 630 & 922 & 1618 & 373  \\
 WN-v2  & 10 & 6954 & 15262 & 6954 & 15262 & 1838 & 2757 & 4011 & 852  \\
 WN-v3  & 11 & 12078 & 25901 & 12078 & 25901 & 3097 & 5084 & 6327 & 1143  \\
 WN-v4  & 9 & 3861 & 7940 & 3861 & 7940 & 934 & 7084 & 12334 & 2823  \\
 NL-v1  & 14 & 3103 & 4687 & 3103 & 4687 & 414 & 225 & 833 & 201  \\
 NL-v2  & 88 & 2564 & 8219 & 2564 & 8219 & 922 & 2086 & 4586 & 935  \\
 NL-v3  & 142 & 4647 & 16393 & 4647 & 16393 & 1851 & 3566 & 8048 & 1620  \\
 NL-v4  & 76 & 2092 & 7546 & 2092 & 7546 & 876 & 2795 & 7073 & 1447  \\
 ILPC Small  & 48 & 10230 & 78616 & 6653 & 20960 & 2908 & 6653 & 20960 & 2902  \\
 ILPC Large  & 65 & 46626 & 202446 & 29246 & 77044 & 10179 & 29246 & 77044 & 10184 \\
 HM 1k & 11 & 36237 & 93364 & 36311 & 93364 & 1771 & 9899 & 18638 & 476 \\
 HM 3k  & 11 & 32118 & 71097 & 32250 & 71097 & 1201 & 19218 & 38285 & 1349  \\
 HM 5k & 11 & 28601 & 57601 & 28744 & 57601 & 900 & 23792 & 48425 & 2124  \\
 HM Indigo& 229 & 12721 & 121601 & 12797 & 121601 & 14121 & 14775 & 250195 & 14904  \\
\bottomrule
\end{tabular}
\end{table}

\begin{table}[p]
    \centering
    \caption{Dataset statistics for \textbf{transductive} link prediction datasets. Entity task denotes the entity-prediction task: $h/t$ is predicting both heads and tails, and $t$ is
predicting only tails. }
\footnotesize
    \label{tab:transductive-statistics}
\begin{tabular}{lccccccc}
\toprule \textbf{Dataset} & \textbf{Entities} & \textbf{Rels} & \textbf{Train} & \textbf{Valid} & \textbf{Test} & \textbf{Entity Task} \\
\midrule
FB15k-237  & 14541 & 237 & 272115 & 17535 & 20466 & $h/t$ \\
WN18RR  & 40943 & 11 & 86835 & 3034 & 3134 & $h/t$  \\
 CoDEx Small  & 2034 & 42 & 32888 & 1827 & 1828 & $h/t$ \\
CoDEx Medium  & 17050 & 51 & 185584 & 10310 & 10311 & $h/t$  \\
CoDEx Large & 77951 & 69 & 551193 & 30622 & 30622 & $h/t$ \\
NELL995 & 74536 & 200 & 149678 & 543 & 2818 & $h/t$ \\
YAGO310 & 123182 & 37 & 1079040 & 5000 & 5000 & $h/t$  \\
WDsinger  & 10282 & 135 & 16142 & 2163 & 2203 & $h/t$  \\
NELL23k  & 22925 & 200 & 25445 & 4961 & 4952 & $h/t$ \\
AristoV4       & 44949  & 1605 & 242567  & 20000      & 20000  & $h/t$ \\
DBpedia100k    & 99604  & 470  & 597572  & 50000      & 50000  & $h/t$ \\
ConceptNet100k  & 78334  & 34   & 100000   & 1200       & 1200    & $h/t$ \\
FB15k-237(10)  & 11512 & 237 & 27211 & 15624 & 18150 & $t$  \\
FB15k-237(20) & 13166 & 237 & 54423 & 16963 & 19776 &$t$ \\
FB15k-237(50)  & 14149 & 237 & 136057 & 17449 & 20324 & $t$ \\
Hetionet & 45158 & 24 & 2025177 & 112510 & 112510 & $h/t$  \\
\bottomrule
\end{tabular}
\end{table}

\begin{table}[p]
\centering
\caption{Different graph pretraining mix shown in \Cref{sec:scaling_main}.}
\footnotesize
\label{tab:pretrain-mix-table}
\begin{tabular}{lccccccc}
\toprule
 & 1 & 2 & 3 & 4 & 5 & 6 & 8 \\
\midrule FB15k-237 & $\checkmark$ & $\checkmark$ & $\checkmark$ & $\checkmark$ & $\checkmark$ & $\checkmark$ & $\checkmark$ \\
 WN18RR & & $\checkmark$ & $\checkmark$ & $\checkmark$ & $\checkmark$ & $\checkmark$ & $\checkmark$ \\
 CoDEx Medium & & & $\checkmark$ & $\checkmark$ & $\checkmark$ & $\checkmark$ & $\checkmark$ \\
 NELL995 & & & & $\checkmark$ & $\checkmark$ & $\checkmark$ & $\checkmark$ \\
 YAGO 310 & & & & & $\checkmark$ & $\checkmark$ & $\checkmark$ \\
 ConceptNet100k & & & & & & $\checkmark$ & $\checkmark$ \\
 DBpedia100k & & & & & & & $\checkmark$ \\
 AristoV4 & & & & & & & $\checkmark$ \\
\bottomrule
\end{tabular}
\end{table}

\begin{table}[t]
\small
\centering
\caption{Hyperparameter for $\flock$ in pretraining and finetuning setups.}
\label{tab:hyperparameter}
\begin{tabular}{cccc}
\toprule & \textbf{Hyperparameter} & \textbf{Entity prediction} & \textbf{Relation prediction}\\
\midrule 
\multirow{3}{*}{Random walk} & Walk length $\ell$ & 128 & 128 \\
& \# Pretraining base walk $n_\text{train}$ & 128 & 128 \\
& \# Test-time or finetuning base walk $n$ & 16--512 & 16--512 \\
\midrule 
\multirow{2}{*}{Sequence processor} & \# Layers  & 1 & 1 \\
& Hidden dimension & 64 & 64 \\
\midrule 
\multirow{2}{*}{Consensus protocol} & \# Heads $h$  & 4 & 4 \\
& Head dimension $d_h$ & 16 & 16 \\
\midrule 
\multirow{1}{*}{Update} & \# Update step $I$  & 6 & 6 \\
\midrule 
\multirow{1}{*}{Ensemble} & \# Maximum ensembled passes $P$ & 16 & 16 \\
\midrule 
\multirow{7}{*}{Pretraining} & Optimizer & AdamW  & AdamW \\
& Learning rate & 0.0005 & 0.0005 \\
& Training steps & 400,000 & 40,000 \\
& Adversarial temperature & 1 & 1 \\
& \# Negatives & 512  & 512 \\
& Batch size & 8 & 8 \\
& Weight decay & 0.01  & 0.00 \\
\midrule
\multirow{5}{*}{Finetuning} & Optimizer & AdamW & AdamW \\
& Learning rate & 0.0005  & 0.0005 \\
& Adversarial temperature & 1 & 1 \\
& \# Negatives & 256 & 256 \\
& Batch size & 4--32 & 4--8 \\
\bottomrule
\end{tabular}
\end{table}

\begin{table}[p]
\small
\centering
\caption{Detailed finetuning and inference hyperparameters for $\flock$ in entity prediction. For each dataset, we report the finetuning epochs, batches per epoch, batch size, and the inference settings for both zero-shot and finetuned modes: test-time ensemble size $P$, base walk count $n$.
For Hetionet finetuning we used $(P, n)=(1, 1024)$, instead of $(2, 512)$ as in zero-shot.
}
\label{tab:hyperparameter-finetune-1}
\begin{tabular}{lccccc}
\toprule
\textbf{Dataset} & \textbf{Epoch} & \textbf{\# Batch/Epoch} & \textbf{Batch Size} & \textbf{\# Ensembled Passes $P$} & \textbf{\# Base Walk $n$} \\
\midrule
FB15k-237           & 1 & full & 8 & 16 & 128  \\
WN18RR             & 1 & full & 8 & 16 & 128  \\
CoDEx Small        & 1 & full & 32 & 16 & 16  \\
CoDEx Medium       & 1 & full & 8 & 16 & 128  \\
CoDEx Large        & 1 & 2000 & 4 & 2 & 512 \\
NELL-995           & 1 & full & 8 & 16 & 128 \\
YAGO310            & 1 & 2000 & 4 & 8 & 512 \\
WDsinger           & 1 & full & 8 & 16 & 16  \\
NELL23k            & 3 & full & 8 & 16 & 32  \\
FB15k-237(10)       & 1 & full & 8 & 16 & 32  \\
FB15k-237(20)       & 1 & full & 8 & 16 & 64  \\
FB15k-237(50)       & 1 & full & 8 & 16 & 64  \\
Hetionet           & 1 & 4000 & 8 & 2 & 512 \\
DBpedia100k        & 1 & 1000 & 4 & 2 & 512 \\
AristoV4           & 1 & full & 8 & 4 & 256 \\
ConceptNet100k     & 1 & full & 8 & 16 & 128 \\
FB v1--v4 & 1 & full & 8 & 16 & 16 \\
WN v1--v4 & 1 & full & 8 & 16 & 16 \\
NL v1--v4 & 3 & full & 8 & 16 & 16 \\
ILPC Small & 1 & full & 8 & 16 & 16 \\
ILPC Large & 1 & full & 8 & 16 & 64 \\
FB 25--100  & 3 & full & 8 & 16 & 16 \\
WK 25--100  & 3 & full & 8 & 16 & 16 \\
NL 0--100   & 3 & full & 8 & 16 & 16 \\
Wiki MT1 tax     & 3 & full & 8 & 16 & 16 \\
Wiki MT1 health  & 3 & full & 8 & 16 & 16 \\
Wiki MT2 org     & 3 & full & 16 & 16 & 32 \\
Wiki MT2 sci     & 3 & full & 8 & 16 & 16 \\
Wiki MT3 art     & 3 & full & 16 & 16 & 32 \\
Wiki MT3 infra   & 3 & full & 16 & 16 & 32 \\
Wiki MT4 sci     & 3 & full & 8 & 16 & 16 \\
Wiki MT4 health  & 3 & full & 8 & 16 & 16 \\
Metafam & 3 & full & 8 & 16 & 16 \\
FBNELL & 3 & full & 8 & 16 & 16 \\
HM 1k     & 1 & full & 8 & 16 & 16  \\
HM 3k     & 1 & full & 16 & 16 & 32  \\
HM 5k     & 1 & full & 8 & 16 & 64  \\
HM Indigo & 1 & full & 8 & 16 & 128 \\
\bottomrule
\end{tabular}
\end{table}

\begin{table}[p]
\small
\centering
\caption{Detailed finetuning and inference hyperparameters for $\flock$ in relation prediction. For each dataset, we report the finetuning epochs, batches per epoch, batch size, and the inference settings for both zero-shot and finetuned modes: test-time ensemble size $P$ and base walk count $n$.
}
\label{tab:hyperparameter-finetune-2}
\begin{tabular}{lccccc}
\toprule
\textbf{Dataset} & \textbf{Epoch} & \textbf{\# Batch/Epoch} & \textbf{Batch Size} & \textbf{\# Ensembled Passes $P$} & \textbf{\# Base Walk $n$} \\
\midrule
FB15k-237           & 1 & 1000 & 8 & 16 & 128  \\
WN18RR             & 1 & 1000 & 8 & 16 & 128  \\
CoDEx Small        & 1 & 1000 & 8 & 16 & 16  \\
CoDEx Medium       & 1 & 1000 & 8 & 16 & 128  \\
CoDEx Large        & 1 & 1000 & 4 & 2 & 512 \\
NELL-995           & 1 & 1000 & 8 & 16 & 128 \\
YAGO310            & 1 & 1000 & 8 & 4 & 512 \\
WDsinger           & 1 & 1000 & 8 & 16 & 16  \\
NELL23k            & 1 & 1000 & 8 & 16 & 32  \\
FB15k-237(10)       & 1 & 1000 & 8 & 16 & 32  \\
FB15k-237(20)       & 1 & 1000 & 8 & 16 & 64  \\
FB15k-237(50)       & 1 & 1000 & 8 & 16 & 64  \\
Hetionet           & 1 & 1000 & 4 & 2 & 512 \\
DBpedia100k        & 1 & 1000 & 4 & 2 & 512 \\
AristoV4           & 1 & 1000 & 8 & 4 & 256 \\
ConceptNet100k     & 1 & 1000 & 8 & 16 & 128 \\
FB v1--v4 & 1 & 1000 & 8 & 16 & 16 \\
WN v1--v4 & 1 & 1000 & 8 & 16 & 16 \\
NL v1--v4 & 1 & 1000 & 8 & 16 & 16 \\
ILPC Small & 1 & 1000 & 8 & 16 & 16 \\
ILPC Large & 1 & 1000 & 8 & 16 & 64 \\
FB 25--100  & 1 & 1000 & 8 & 16 & 16 \\
WK 25--100  & 1 & 1000 & 8 & 16 & 16 \\
NL 0--100   & 1 & 1000 & 8 & 16 & 16 \\
Wiki MT1 tax     & 1 & 1000 & 8 & 16 & 16 \\
Wiki MT1 health  & 1 & 1000 & 8 & 16 & 16 \\
Wiki MT2 org     & 1 & 1000 & 8 & 16 & 32 \\
Wiki MT2 sci     & 1 & 1000 & 8 & 16 & 16 \\
Wiki MT3 art     & 1 & 1000 & 8 & 16 & 32 \\
Wiki MT3 infra   & 1 & 1000 & 8 & 16 & 32 \\
Wiki MT4 sci     & 1 & 1000 & 8 & 16 & 16 \\
Wiki MT4 health  & 1 & 1000 & 8 & 16 & 16 \\
Metafam & 1 & 1000 & 8 & 16 & 16 \\
FBNELL & 1 & 1000 & 8 & 16 & 16 \\
HM 1k     & 1 & 1000 & 8 & 16 & 16  \\
HM 3k     & 1 & 1000 & 8 & 16 & 32  \\
HM 5k     & 1 & 1000 & 8 & 16 & 64  \\
HM Indigo & 1 & 1000 & 8 & 16 & 128 \\
\bottomrule
\end{tabular}
\end{table}

\end{document}

%% file: math_commands.tex
\usepackage{amsmath,amsfonts,bm}
\usepackage{mathtools}

\definecolor{likeblue}{RGB}{80, 80, 190}
\definecolor{dislikered}{RGB}{190, 80, 80}
\definecolor{neutralgray}{RGB}{150, 150, 150}
\definecolor{solidgreen}{RGB}{80, 190, 80}
\definecolor{sillyorange}{RGB}{255, 140, 20}
\definecolor{unnecessarypink}{RGB}{255, 110, 190}
\definecolor{unapologeticyellow}{RGB}{255, 215, 5}
\definecolor{gainsboro}{rgb}{0.86, 0.86, 0.86}

\newcommand\probeq{\stackrel{\mathclap{\text{\tiny{\mbox{$d$}}}}}{=}}
\newcommand\musim{\stackrel{\mathclap{\text{\tiny{\mbox{$\mu$}}}}}{\simeq}}

\def\eqref#1{equation~\ref{#1}}
\def\Eqref#1{Equation~\ref{#1}}

\def\1{\bm{1}}

\DeclareMathAlphabet{\mathsfit}{\encodingdefault}{\sfdefault}{m}{sl}
\SetMathAlphabet{\mathsfit}{bold}{\encodingdefault}{\sfdefault}{bx}{n}

\def\gF{{\mathcal{F}}}

\def\gU{{\mathcal{U}}}

\def\sP{{\mathbb{P}}}

\DeclareMathOperator*{\argmin}{arg\,min}

%% file: macros.tex
\usepackage{amsthm}

\newtheorem{globalcounter}{Theorem}[section] 

\newtheorem{proposition}[globalcounter]{Proposition}
\newtheorem{lemma}[globalcounter]{Lemma}

\newtheorem{remark}[globalcounter]{Remark}
\newtheorem{corollary}[globalcounter]{Corollary}

\theoremstyle{plain}

\newtheorem{innercustomprop}{Proposition}
\newenvironment{propositioncopy}[1]{\renewcommand\theinnercustomprop{#1}\innercustomprop}{\endinnercustomprop}

\newcommand\draft[1]{#1}  

\makeatletter
\newcommand\tinysmall{\@setfontsize\tinysmall{8}{10}}
\makeatother

\newcommand{\ultra}{\textsc{Ultra}}
\newcommand{\mgnn}{\textsc{Motif}}
\newcommand{\trix}{\textsc{Trix}}
\newcommand{\flock}{\textsc{Flock}}